\documentclass{article}

\usepackage{algorithm}
\usepackage{algpseudocode}
\usepackage{amsfonts}       
\usepackage{subcaption,siunitx,booktabs}
\usepackage{nicefrac}
\usepackage{times}
 \usepackage[utf8]{inputenc}
\usepackage{diagbox}
\usepackage{mathtools, nccmath}
\usepackage{adjustbox}
\usepackage{wrapfig}

\usepackage{comment}
\usepackage{microtype}
\usepackage{graphicx}
\usepackage{multirow}
\usepackage{multicol}
\usepackage{rotating}
\usepackage{array}
\usepackage{textcomp}

\usepackage{caption}
\usepackage{hyperref}

\usepackage{siunitx}
\usepackage{color}
\usepackage{xcolor}

\usepackage{anyfontsize}
\usepackage{nicematrix}
 
\usepackage{arydshln}
\usepackage{tabularx}
\usepackage{soul}

\clearpage

\newcommand{\tenpt}{\fontsize{10pt}{12pt}\selectfont}

\newcommand{\BlankLine}{\vspace{0.5cm}}

\usepackage[accepted]{icml2024}

\usepackage{amsmath}
\usepackage{amssymb}
\usepackage{mathtools}
\usepackage{amsthm}
\usepackage{float}

\usepackage[utf8]{inputenc} 
\usepackage[T1]{fontenc}    
\usepackage{lipsum}

\usepackage[capitalize,noabbrev]{cleveref}

\theoremstyle{plain}
\newtheorem{theorem}{Theorem}[section]

\theoremstyle{definition}
\newtheorem{definition}[theorem]{Definition}

\theoremstyle{remark}

\usepackage[textsize=tiny]{todonotes}

\definecolor{lightgray}{gray}{0.95}

\icmltitlerunning{RODEO: Robust Outlier Detection
via Exposing Adaptive Outliers}

\begin{document}

\twocolumn[
\icmltitle{RODEO: Robust Outlier Detection via Exposing Adaptive Out-of-Distribution Samples}



\icmlsetsymbol{equal}{*}

\begin{icmlauthorlist}
\icmlauthor{Hossein Mirzaei}{sut}
\icmlauthor{Mohammad Jafari}{sut}
\icmlauthor{Hamid Reza Dehbashi}{sut}
\icmlauthor{Ali Ansari}{sut}
\icmlauthor{Sepehr Ghobadi}{sut}
\icmlauthor{Masoud Hadi}{iut}
\icmlauthor{Arshia Soltani Moakhar}{sut}
\icmlauthor{Mohammad Azizmalayeri}{sut}
\icmlauthor{Mahdieh Soleymani Baghshah}{sut}
\icmlauthor{Mohammad Hossein Rohban}{sut}
\end{icmlauthorlist}

\icmlaffiliation{sut}{Sharif University of Technology, Tehran, Iran}
\icmlaffiliation{iut}{Isfahan University of Technology, Isfahan, Iran}

\icmlcorrespondingauthor{Mohammad Hossein Rohban}{rohban@sharif.edu}
\icmlcorrespondingauthor{Hossein Mirzaei}{hossein.mirzaeisadeghlou@epfl.ch}


\icmlkeywords{Outlier Detection, Adversarial Robustness, Robust Anomaly Detection, Generative Models}

\vskip 0.3in
]



\printAffiliationsAndNotice{}  

\begin{abstract}
In recent years, there have been significant improvements in various forms of image outlier detection. However, outlier detection performance under adversarial settings lags far behind that in standard settings. This is due to the lack of effective exposure to adversarial scenarios during training, especially on unseen outliers, leading to detection models failing to learn robust features. To bridge this gap, we introduce RODEO, a data-centric approach that generates effective outliers for robust outlier detection. More specifically, we show that incorporating outlier exposure (OE) and adversarial training can be an effective strategy for this purpose, as long as the exposed training outliers meet certain characteristics, including diversity, and both conceptual differentiability and analogy to the inlier samples. We leverage a text-to-image model to achieve this goal. We demonstrate both quantitatively and qualitatively that our adaptive OE method effectively generates ``diverse'' and ``near-distribution'' outliers, leveraging information from both text and image domains. Moreover, our experimental results show that utilizing our synthesized outliers significantly enhances the performance of the outlier detector, particularly in adversarial settings.  The implementation of our work is available at: \url{https://github.com/rohban-lab/RODEO}.

\end{abstract}


\section{Introduction}

\begin{figure}[t]
  \begin{center}
    \centerline{\includegraphics[width=\columnwidth]{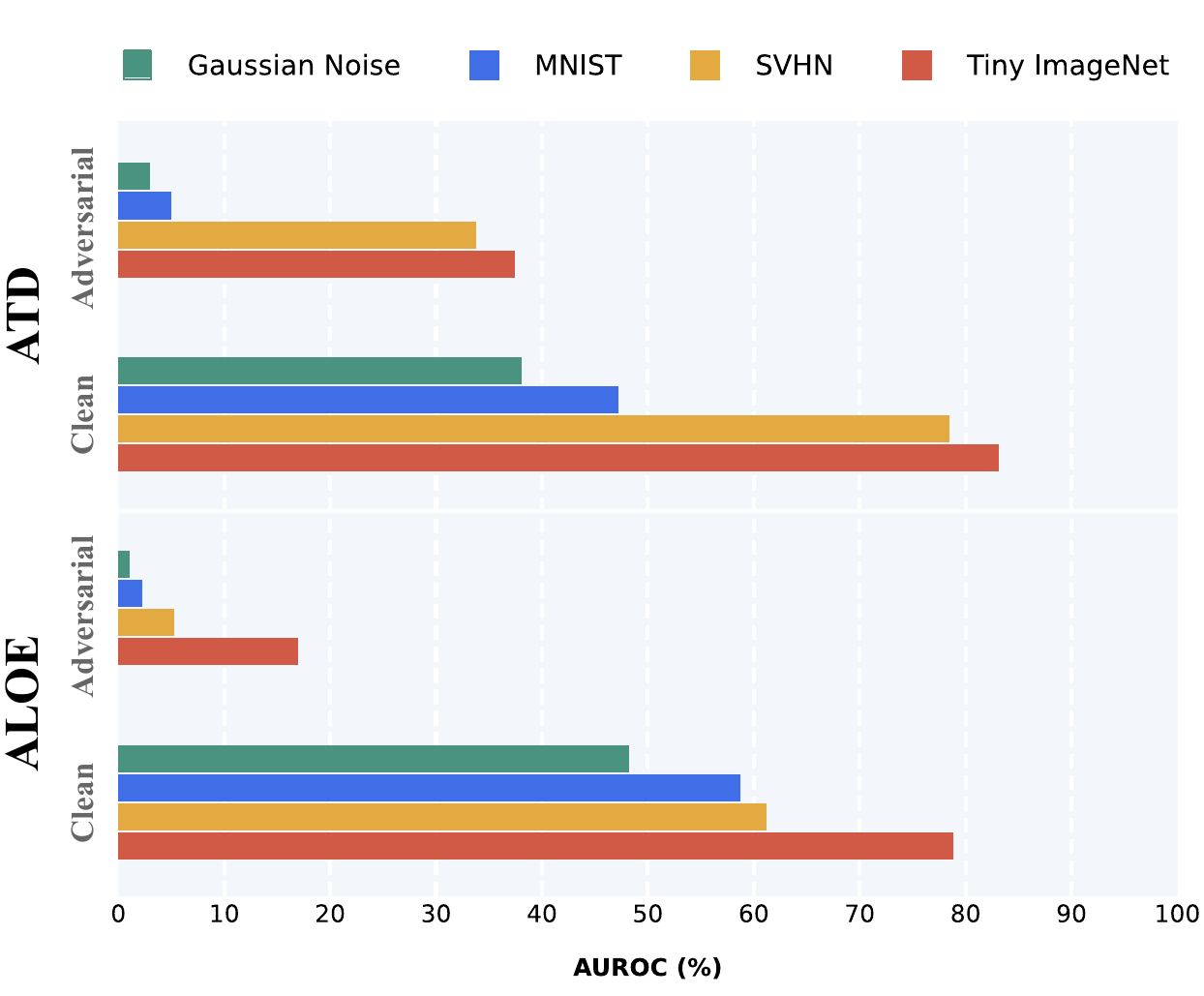}}
        \caption{ ALOE and ATD are robust outlier detection methods that utilize the Tiny ImageNet dataset as OE. In this experiment, while keeping all other aspects of the original methods constant, we replaced Tiny ImageNet with SVHN, MNIST, and Gaussian noise and repeated the experiments for both ALOE and ATD. This replacement led to a notable decline in detection performance for ALOE and ATD on the CIFAR10 vs. CIFAR100 task, particularly under adversarial attack conditions. We attribute this performance drop to the fact that the SVHN, MNIST, and Gaussian Noise distributions are more distant from CIFAR10 (the inlier distribution in this task) compared to Tiny ImageNet.}
    \label{fig:Comparison_Plot_bar}
  \end{center}
\end{figure}

Outlier detection has become a crucial component in the design of reliable open-world machine learning models \cite{drummond2006open,bendale2015towards,perera2021one}. Robustness against adversarial attacks is another important machine learning safety feature \cite{szegedy2013intriguing,goodfellow2014explaining,akhtar2018threat}. Despite the emergence of several promising outlier detection methods in recent years \cite{liznerski2022exposing,cohen2021transformaly,cao2022deep}, they often suffer significant performance drops when subjected to adversarial attacks, which aim to convert inliers into outliers and vice versa by adding imperceptible perturbations to the input data. In light of this, recently, several robust outlier detection methods have been proposed \cite{azizmalayeri2022your,lo2022adversarially,chen2020robust,shao2020open,shao2022open,bethune2023robust,goodge2021robustness,chen2021atom,meinke2022provably,franco2023diffusion}. However, their results are still unsatisfactory, sometimes performing even worse than random detection, and are often focused on specific cases of outlier detection, such as the open-set recognition or tailored to a specific dataset, rather than being broadly applicable. Motivated by this, we aim to provide a robust and unified solution for outlier detection that can perform well in both clean and adversarial settings.\\
\begin{figure*}[thb]
  \begin{center}
    \includegraphics[width=1\linewidth]{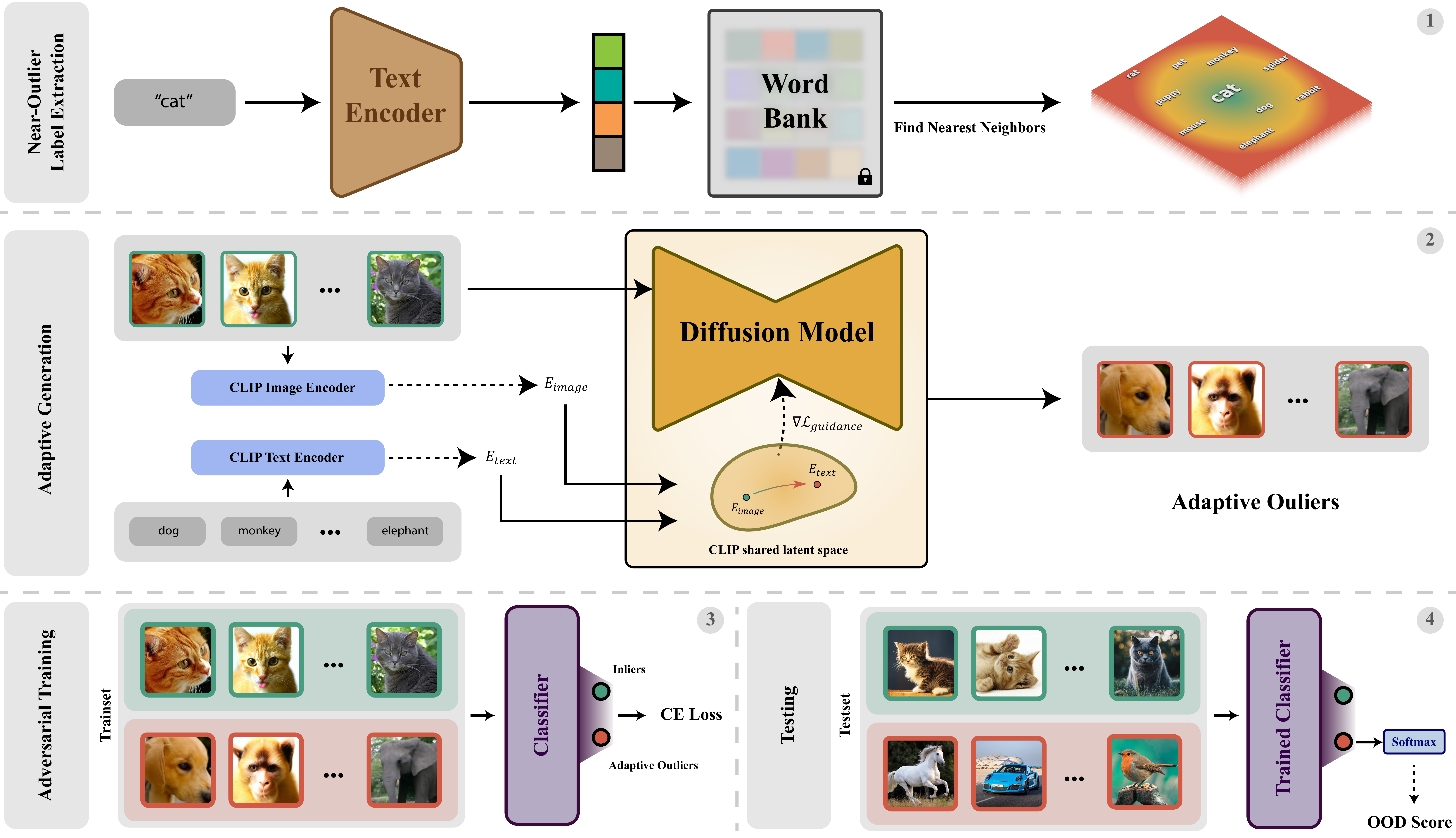}
    \caption{Our proposed adversarially robust outlier detection method is initiated with a\textit{ Near-Outlier Label Extraction}, which finds words analogous to a given input label. These words, combined with inlier training image data, are employed in the \textit{Adaptive Generation} stage to create Near-Outllier data. This stage is followed by \textit{Adversarial Training} using both inliers and generated OE data, utilizing the cross-entropy loss function. During \textit{Testing}, the model processes the input and computes the OOD score as the softmax of the OOD class. (The data filtering steps are not shown in this figure)}
    \label{fig:Method_Plot_Me}
  \end{center}
\end{figure*}
Adversarial training, which is the augmentation of training samples with adversarial perturbations, is among the best practices for making models robust \cite{madry2017towards}. However, this approach is less effective in outlier detection, as outlier patterns are unavailable during training, thus preventing training of the models with the adversarial perturbations associated with these outliers. For this reason, recent robust outlier detection methods use Outlier Exposure (OE) technique \cite{hendrycks2018deep} in combination with adversarial training to tackle this issue \cite{azizmalayeri2022your,chen2020robust,chen2021atom}. In OE, the auxiliary outlier samples are typically obtained from a random and fixed dataset and are leveraged during training. It is clear that these samples should be semantically different from the inlier training set to avoid misleading the detection.

In this study, we experimentally observe (see Fig. \ref{fig:Comparison_Plot_bar} and Sec. \ref{sec3}) that the OE technique's performance is highly sensitive to the distance between the exposed outliers and the inlier training set distribution. Our results suggest that a near-distribution OE set is significantly more beneficial than a distant one. By near-distribution outliers, we refer to image data that possesses semantically and stylistically related characteristics to those of the inlier dataset. 

Our observation aligns with \cite{xing2022artificially}, which suggests that incorporating data near the decision boundary leads to a more adversarially robust model in the classification task. Simultaneously, numerous studies \cite{schmidt2018adversarially,stutz2019disentangling} have demonstrated that adversarial training demands a greater level of sample complexity relative to the standard setting. Thus, several efforts have been made to enrich the data diversity to enhance the adversarial robustness \cite{hendrycks2019usingx,sehwag2021robust,pang2022robustness}.

These observations prompt us to propose the following hypotheses: {\it For adversarial training to be effective in robust outlier detection, the OE samples need to be diverse, near-distribution, and conceptually distinguishable from the inlier samples.} We have conducted numerous extensive ablation studies (Sec. \ref{Ablation_Section}), and provided theoretical insights (Sec. \ref{Theory}) to support these claims.

Driven by the mentioned insights, we introduce RODEO (Robust Outlier Detection via Exposing adaptive Out-of-distribution samples), a novel method that enhances the robustness of outlier detection by leveraging an adaptive OE strategy. Our method assumes access to the text label(s) describing the content of the inlier samples, which is a fair assumption according to the prior works \cite{liznerski2022exposing, adaloglou2023adapting}. Specifically, the first step involves utilizing a simple text encoder to extract labels that are semantically close to the inlier class label(s), based on their proximity within the CLIP \cite{radford2021learning} textual representation space. To ensure the extracted labels are semantically distinguishable from inlier concepts, we apply a threshold filter, precomputed from a validation set. Then, we initiate the denoising process of a pretrained diffusion image generator \cite{dhariwal2021diffusion} conditioned on the inlier images. The backward process of the diffusion model is guided by the gradient of the distance between the extracted outlier labels' textual embeddings and the visual embeddings of the generated images. Through this guidance, the diffusion model is enforced to increase the similarity between the generated images and the extracted labels at each step, and transfer inliers to near-distribution outliers during the process. Finally, another predefined threshold, obtained through validation, filters the generated data belonging to the in-distribution based on the CLIP score. We then demonstrate that adversarial training on a classifier that discriminates the inlier and synthesized OE significantly improves robust outlier detection.  

We evaluate RODEO in both clean and adversarial settings. In the adversarial setting, numerous strong attacks, including PGD-1000 \cite{madry2017towards}, AutoAttack \cite{croce2020reliable}, and Adaptive Auto Attack \cite{liu2022practical}, are employed for robustness evaluation. Our experiments are conducted across various common outlier detection setups, including Novelty Detection (ND), Open-Set Recognition (OSR), and Out-of-Distribution (OOD) detection. It is noteworthy that previous robust outlier detection methods were primarily limited to specific types of outlier setups. In these experiments, RODEO's performance is compared against recent and representative outlier detection methods. The compared methods, including EXOE \cite{liznerski2022exposing} and PLP \cite{adaloglou2023adapting}, utilized a pretrained CLIP as their detector backbone. The results indicate that RODEO establishes significant performance in adversarial settings, surpassing existing methods by up to 50\% in terms of AUROC detection, and achieves competitive results in clean settings. Moreover, through an extensive ablation study, we evaluated our adaptive OE method pipeline in comparison to alternative OE methods \cite{lee2018simple, tao2023non, kirchheim2022outlier, du2022vos, mirzaei2022fake}, including both baseline and recent synthetic outlier methods such as Dream-OOD \cite{du2023dream}, which utilizes Stable Diffusion \cite{rombach2022high} trained on 5 billion data samples as its generative backbone. In the ablation study, we analyze why RODEO outperforms other alternative OE methods.

\begin{figure*} 
  \begin{center}
    \includegraphics[width=0.95\linewidth]{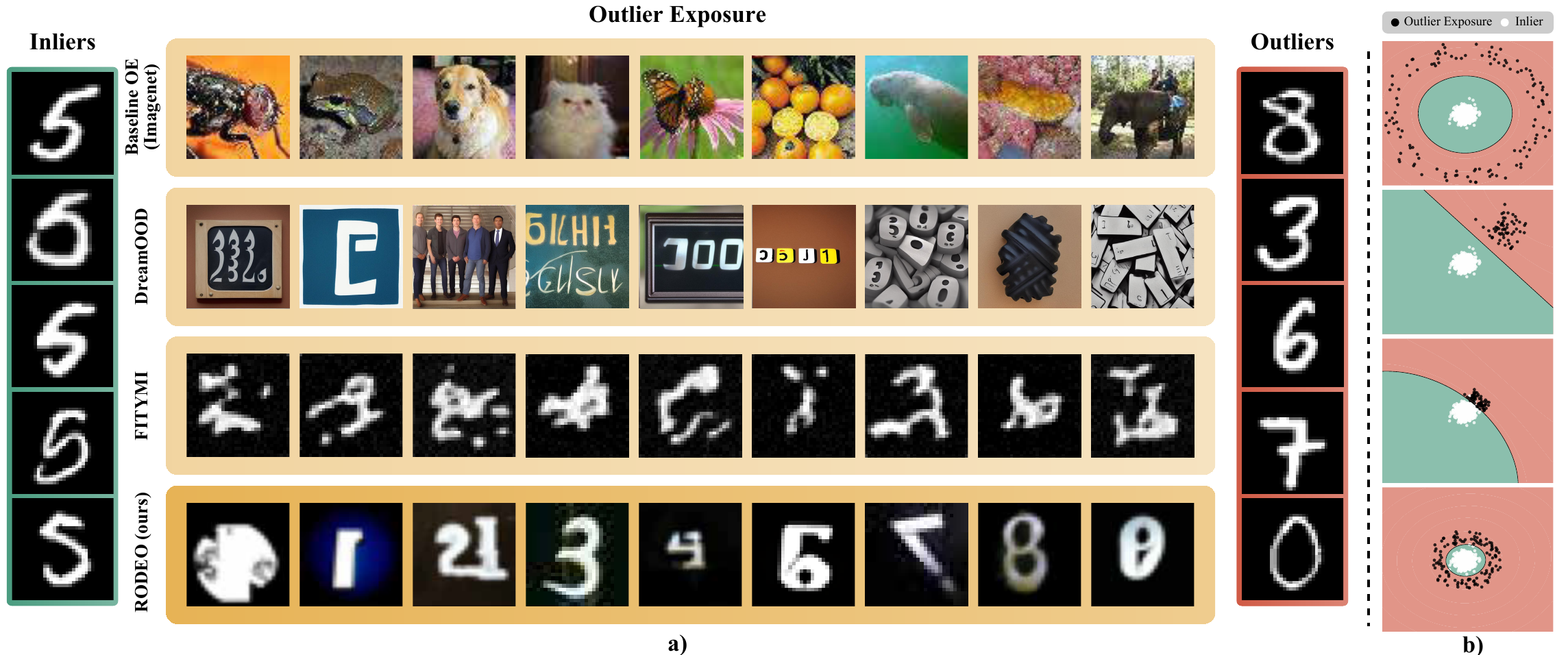}
    \caption{An overview of outlier data from different OE techniques. FITYMI considers image domain information exclusively. Dream-OOD utilizes both text and image domains, but initiating generation from the embedding space makes this method highly biased toward its prior knowledge of the generative backbone, Stable Diffusion. In contrast, RODEO shifts data from inlier to outlier while operating in pixel space. To provide further intuition about the importance of diversity and the distance of the OE from the in-distribution, we compute features for inliers and generated outlier data via a pretrained ViT model \cite{dosovitskiy2020image}, and apply t-SNE \cite{van2008visualizing} to visualize the data in 2D. We then find decision boundaries of the data with SVM \cite{cortes1995support} and present them on the right side of each generated OE example. Our OE samples are both near-distribution and diverse.}
    \label{fig:Method_Plot}
  \end{center}
\end{figure*}

\section{Related Work}
\textbf{Outlier Detection.} \   Several works have been proposed in outlier detection, with the goal of learning the distribution of inlier samples. Some methods such as CSI \cite{tack2020csi} do this with self-supervised approaches. On the other hand, many methods, such as MSAD \cite{reiss2021mean}, Transformaly \cite{cohen2021transformaly}, ViT-MSP \cite{NEURIPS2021_3941c435}, and Patchcore \cite{roth2021total}, aim to leverage knowledge from pre-trained models. EXOE \cite{liznerski2022exposing} and PLP \cite{adaloglou2023adapting} utilized text and image data for the detection task, employing a CLIP model that was pretrained on 400 million data points. Furthermore, some other works have pursued outlier detection in an adversarial setting, including APAE \cite{goodge2021robustness}, PrincipaLS \cite{lo2022adversarially}, OCSDF \cite{bethune2023robust}, and OSAD \cite{shao2022open}. ATOM \cite{chen2021atom}, ALOE \cite{chen2020robust}, and ATD \cite{azizmalayeri2022your} achieved relatively better results compared to others by incorporating OE techniques and adversarial training. However, their performance falls short (as presented in Fig. \ref{fig:Comparison_Plot_bar}) when the inlier set distribution is far from their fixed OE set. 
 For more details about previous works, see Appendix (Sec. \ref{detailed_base}).\\ 
 
\textbf{Outlier Exposure Methods.}
MixUp \cite{zhang2017mixup} adopts a more adaptive OE approach by blending ImageNet samples with inlier samples to create outlier samples closer to the in-distribution. FITYMI \cite{mirzaei2022fake} introduced an OE generation pipeline using a diffusion generator trained on inliers but halted early to create synthetic images that resemble inliers yet display clear differences. The GOE method \cite{kirchheim2022outlier} employs a pretrained GAN to generate synthetic outliers by targeting low-density areas in the inlier distribution. Dream-OOD \cite{du2023dream} uses both image and text domains to learn visual representations of inliers in a latent space and samples new images from its low-likelihood regions. 
Other OE methods, such as VOS and NPOS \cite{du2022vos, tao2023nonparametric}, generate embeddings instead of actual image data.

\section{Theoretical Insights}\label{Theory}
\label{sec3}

In this section, we provide some insightful examples that highlight the need for near-distribution and diverse OE in outlier detection. Our setup is the following: We assume that the inlier data come from $\mathcal{N}(0, \sigma^2 I)$ and the test-time outlier is distributed according to $\mathcal{N}(a, \sigma^2 I)$. Furthermore, let $\mathcal{N}(a^\prime, \sigma^2 I)$ be the OE data distribution. We also assume equal class a priori probabilities. We train a classifier using a balanced mixture of inlier and OE samples. However, at test time, the classifier is tested against a balanced mixture of inlier and outlier samples, rather than the OE data.\\

\textbf{Near-distribution OE is beneficial}
 \begin{theorem}
    Let $\sigma = 1$, and assume that $\| a^\prime \| \geq \| a \|$, reflecting that the OE is far from the distribution. Let $\theta$ be the angle between $a$ and $a^\prime$. Under the setup mentioned in Sec. \ref{sec3}, for fixed $\theta$ and $a$, and small $\epsilon$, the optimal Bayes' adversarial error under $\ell_2$ norm bounded attacks with $\epsilon$ norm increases as the OE moves farther from the distribution, i.e., as $\| a^\prime \|$ increases. More specifically, the adversarial error is:
 \small
\begin{equation}
\begin{split}
  1 - \Phi\left( \frac{\|a^\prime\|}{2} - \epsilon \right) +  1 - \Phi\left( \frac{\|a^\prime\|}{2} - c - \epsilon \right),
\end{split}
\end{equation}
with $c := \| a^\prime \| - \| a \| \cos(\theta)$, and $\Phi(.)$ being the standard normal cumulative distribution function.
 \end{theorem}
 \begin{proof}
  Under the mentioned setup, the optimal robust classifier is $f^\star(x) = \frac{a^{\prime\top}}{\|a^\prime\|} (x - \frac{a^\prime}{2}) = \frac{a^{\prime\top}}{\|a^\prime\|} x - \frac{\|a^\prime\|}{2}$, for an adversary that has a budget of at most $\epsilon$ perturbation in $\ell_2$ norm \cite{schmidt2018adversarially}. Now, applying this classifier on the inlier and outlier classes at test time, we get:
{\small
\begin{equation}
   \frac{a^{\prime\top} x}{\|a^\prime\|} \sim \mathcal{N}(0, I),  
\end{equation}}
for an inlier $x$, and also:
{\small
\begin{equation}
   \frac{a^{\prime\top} x}{\|a^\prime\|} \sim \mathcal{N}\left(\frac{a^\top a^\prime}{\|a^\prime\|}, I\right),  
\end{equation}}
for an outlier $x$. Therefore, using the  classifier $f^\star$ to discriminate the inlier and outlier classes, the adversarial error rate would be:
{\small
\begin{equation} \label{eq3}
    1 - \Phi\left( \frac{\|a^\prime\|}{2} - \epsilon\right) + 1 - \Phi\left(\frac{a^\top a^\prime}{\|a^\prime\|} - \frac{\|a^\prime\|}{2} - \epsilon\right),
\end{equation}}
where $\Phi(.)$ is the CDF for the inlier distribution $\mathcal{N}(0, 1)$. 

Let $\delta = a^\prime - a$, and note that:
{\small
\begin{equation} \label{eq4}
\begin{split}
    \frac{a^\top a^\prime}{\|a^\prime\|} - \frac{\|a^\prime\|}{2} = \frac{(a^\prime - \delta)^\top a^\prime}{\|a^\prime\|} - \frac{\|a^\prime\|}{2} 
    = \frac{\|a^\prime\|}{2} - \frac{\delta^\top a^\prime}{\| a^\prime \|}.     
\end{split}
\end{equation}}

But note that $\frac{\delta^\top a^\prime}{\| a^\prime \|} = \| a^\prime \| - \| a \| \cos(\theta) =: c \geq \| a^\prime \| - \| a \|  \geq 0$, where $\theta$ is the angle between $a^\prime$ and $a$. Hence, by plugging Eq. \ref{eq4} into Eq. \ref{eq3}, the error rate can be written as:  
\small
\begin{equation}
\begin{split}
  1 - \Phi\left( \frac{\|a^\prime\|}{2} - \epsilon \right) +  1 - \Phi\left( \frac{\|a^\prime\|}{2} - c - \epsilon \right).
\end{split}
\end{equation}
But note that the derivative of the above expression with respect to $\| a^\prime \|$ is

\begin{multline*}    
    - \frac{1}{2\sqrt{2\pi}} \exp\left( -\frac{1}{2} \left( \frac{\| a^\prime\|}{2} - \epsilon \right)^2 \right) \\ + \frac{1}{2\sqrt{2 \pi}} \exp\left( -\frac{1}{2} \left( -\frac{\| a^\prime\|}{2} + \| a \| \cos(\theta) - \epsilon \right)^2 \right), 
\end{multline*}
noting that the derivative of $\frac{1}{2} \| a^\prime \| - c$ with respect to $\| a^\prime \|$ is $-\frac{1}{2}$. We note that for $\epsilon \leq \frac{\| a^\prime \|}{2}$, the derivative is positive as long as 
\begin{equation}
    \frac{\| a^\prime \|}{2} - \epsilon \geq \frac{\| a^\prime \|}{2} - \| a \| \cos(\theta) + \epsilon,
\end{equation}
and 
\begin{equation}
    \frac{\| a^\prime \|}{2} - \epsilon \geq -\frac{\| a^\prime \|}{2} + \| a \| \cos(\theta) - \epsilon.
\end{equation}
The first condition is satisfied as long as $\epsilon \leq \| a \| \cos(\theta)$, which is always satisfied for small $\epsilon$ if $\theta \neq \pi/2$. The second condition is also satisfied as $\|a^\prime\| \geq \|a\|$ by the theorem assumptions. Hence, the derivative is positive and the adversarial error rate is increasing by increasing $\| a^\prime \|$.
\end{proof}


 \textbf{Diverse OE is beneficial}

In the last section, we provided an example of why OE data should be near-distribution to be helpful in a simple setup. Now, we give further insights into why the OE should be diverse. To show this, we introduce the notion of the worst-case outlier detection error. 
\begin{definition}
    As the outlier distribution is {\it not known} during training, we seek to optimize for its worst-case performance, i.e.:
\begin{equation}
    R(f) := \sup_{\| a \| = \alpha } \mathbb{E}_{x \sim p}(\ell(f(x), y)),
\end{equation}
where we assume 0/1 loss for simplicity, and $p := 0.5 \mathcal{N}(0, \sigma^2 I) + 0.5 \mathcal{N}(a, \sigma^2 I)$.
\end{definition}
\begin{theorem}
Assuming $\sigma \approx 0$, the optimal worst-case outlier detection error under the setup of Sec. \ref{sec3} is $R(f^\star) = 50\%$. Additionally, if the OE is sampled from a Gaussian mixture, with infinitely many mixture components, whose means are sampled uniformly from the hypersphere $\mathcal{S}^{d-1}(\alpha),$ then $R(f^\star) = 0\%$.    
\end{theorem}
\begin{proof}

First note that $f^\star$, which is the optimal classifier, takes a linear form of $f^\star(x) = \frac{a^{\prime \top}} {\| a^\prime \|}(x - \frac{a^\prime}{2})$.
Also, note that if the outlier distribution mean value, $a$, is far from that of the OE, $a^\prime$, the risk $R(f)$ would grow large and become $50\%$. That is, in finding the supremum, $a$ would be placed far from $a^\prime$, resulting in $R(f^\star) = 50\%$; i.e. one plausible solution is $a = - \alpha \frac{a^\prime}{\| a^\prime \|}$, which leads to erroneous output $f^\star(x) < 0$, for all outlier samples $x$ concentrated around $a$. 

However, we note that for this worst-case scenario, a better OE choice, than a simple Gaussian distribution, would be to first randomly pick a center $a^\prime$ uniformly from the sphere centered around zero, with radius $\alpha$, denoted as $\mathcal{S}^{d-1}(\alpha)$, where $d$ is feature space dimensionality. This results in the marginal OE distribution, $p_{OE}$ as follows:
\begin{equation} \label{mixOE}
    p_{OE}(x) \propto \int_{\mathcal{S}^{d-1}(\alpha)} \mathcal{N}(a^\prime, \sigma^2 I) d a^\prime,
\end{equation}
For this choice of OE, the optimal Bayes' classifier that discriminates the inliers and OE, would be a hypersphere centered around zero with radius $\frac{\alpha}{2}$. It is evident that for this classifier, the worst-case risk $R$, would be $0$. The intuition behind this is that wherever the $a$ is placed in taking the supremum, the classifier would detect it as an outlier. 
\end{proof}

This simple example highlights the need for a {\it diverse} OE distribution in solving the outlier detection in the {\it worst-case}. Inspired by this example, one could approach constructing $p_{OE}$ through conditioning the OE distribution on the inlier samples $x_0$, and a target semantic label $y$ that is distinct from $x_0$ original semantic class, $y_0$; i.e. $\hat{p}(x | x_0, y)$, and assuming $\hat{p}$ as a generative process that minimally transforms $x_0$ into an outlier sample with semantic label $y$. This is similar to OE samples in the previous theorem, where Gaussian kernels with means deviating large enough from the in-distribution (with distance $\alpha$ from the inlier class mean)  constitute the OE distribution. To make this analogy happen, the outlier label $y$ has the role of making $x$ sufficiently distant away from the in-distribution. This way, $p_{OE}$ would become:
\begin{equation}
    p_{OE}(x) = \sum_y \int \hat{p}(x | x_0, y) p(x_0) p(y | y_0) dx_0,
\end{equation}
where $p(x_0)$ is the inlier class distribution and $p(y | y_0)$ is the prior distribution over the target classes that are distinct from the inlier semantic class $y_0$. This is similar to the form of OE distribution in Eq. \ref{mixOE}.

\section{Method}

\textbf{Motivation.}
To develop a robust outlier detection model, the Outlier Exposure (OE) technique appears to be crucial \cite{chen2020robust,chen2021atom,azizmalayeri2022your}; otherwise, the model would lack information about the adversarial patterns in the outlier data. However, the Baseline OE technique, which involves leveraging outliers from a presumed dataset, leads to unsatisfactory results in situations where the auxiliary exposed outliers deviate significantly from the in-distribution. Motivated by these factors, we aim to propose an adaptive OE technique that attempts to generate diverse and near-distribution outliers, which can act as a proxy for the real inference-time outliers. The subsequent sections will provide a detailed description of the primary stages of our methodology. Our method is outlined in Fig.~\ref{fig:Method_Plot_Me}. 

\subsection{Generation Step}

\textbf{Near-outlier Label Extraction.} 
Utilizing a text encoder and given the class labels of the inlier samples, we identify words closely related to them. To achieve this, we utilize Word2Vec \cite{mikolov2013efficient}, a renowned and simple text encoder, to obtain the embeddings of the inlier labels, denoted as \( y_{\text{inlier}} \) (e.g., \textit{``screw''}), and subsequently retrieve their nearest labels (e.g., \textit{``nail''}). By comparing these with a pre-computed threshold (\( \tau_{\text{text}} \)), we refine the extracted labels by excluding those very similar to the inlier labels. This threshold is computed using ImageNet labels as the validation set and the CLIP text encoder embedding to compute word similarities. Utilizing these extracted near labels in subsequent steps leads to the generation of semantically-level outlier samples (those that are semantically different from the inliers). To further enhance the diversity of synthesized outliers, we also consider pixel-level OOD samples (those that differ from the in-distribution at the texture level). For this purpose, we incorporate texts containing negative attributes of the inlier labels (e.g., \textit{``broken screw''}), and the union of these two sets of labels forms near outliers (n-outliers) labels: \( y_{\text{n-outliers}} \), which will guide the image generation process utilizing the CLIP model in the next steps. More details about near-label set extraction and threshold (\( \tau_{\text{text}} \)) computing are available in the Appendix (Sec.~\ref{prompt_gen}).

\textbf{CLIP Guidance.} The CLIP model is designed to learn joint representations between the text and image domains, and it comprises a pre-trained text encoder and an image encoder. The CLIP model operates by embedding both images and texts into a shared latent space. This allows the model to assign a CLIP score that evaluates the relevance of a caption to the actual content of an image. In order to effectively extract knowledge from the CLIP in image generation, we propose the $\mathcal{L}_{\text{guidance}}\left(\boldsymbol{x}_{\text{gen}}, \boldsymbol{y}_{\text{n-outliers}}\right)$ loss, which aims to minimize the cosine similarity between the CLIP space embeddings of the generated image $\boldsymbol{x}_{\text{gen}}$ and the target text (extracted outlier labels) \ $\boldsymbol{y}_{\text{n-outliers}}$, i.e. $\mathcal{D}\left(x, y\right) = \frac {x^\top y}{\|x\|\|y\|}$:
  

\begin{equation}
\resizebox{0.85\linewidth}{!}{%
$ \mathcal{L}_{\text{guidance}}\left(\boldsymbol{x}_{\text{gen}}, \boldsymbol{y}_{\text{n-outliers}}\right) = -\mathcal{D}\left(E_I\left(\boldsymbol{x}_{\text{gen}}\right), E_T\left(\boldsymbol{y}_{\text{n-outliers}}\right)\right).$
}
\end{equation}

Here, $E_I$ and $E_T$ represent the embeddings extracted by the CLIP image and text encoders, respectively. During the conditional diffusion sampling process, the gradients from the $\mathcal{L}_{\text{guidance}}$ will be used to guide the inlier sample towards the near outliers.

\textbf{Conditioning on Image.}\ We condition the denoising process on the inlier images instead of initializing it with random Gaussian noise. Specifically, we employ a pre-trained diffusion generator and start the diffusion process from a random time step, initiated with the inliers with noise (instead of beginning with pure Gaussian noise). Based on previous works \cite{meng2021sdedit,kim2022diffusionclip}, we set $t_{0} \sim U(0.3T, 0.6T)$, where $T$ represents the number of denoising steps in the regular generation setup.

Randomly choosing $t_{0}$ for the denoising process leads to the generation of diverse outliers since, with smaller $t_0$, inlier images undergo minor changes, while relatively larger $t_0$ values lead to more significant changes, thereby increasing the diversity of generated outlier samples. We then progressively remove the noise with CLIP guidance to obtain a denoised result that is both outlier and close to the in-distribution: 
$ \ \  {x}_{t-1} \sim \mathcal{N} ( \mu( {x}_t|y_{\text {NOOD}}) + s \cdot \Sigma ( {x}_t|y_{\text {n-outliers}})\cdot \nabla_{{x}_t} (\mathcal{L}_{\text {guidance }}\left( {x}_{{t}}, y_{\text {n-outliers}}\right)), \ \Sigma( {x}_t|y_{\text {n-outliers}}))$, where the scale coefficient $s$ controls the level of perturbation applied to the model. Please see Appendix (Sec. \ref{gen_step_app}) for more details about the preliminaries of diffusion models \cite{sohl2015deep,ho2020denoising} and the generation step. Additionally, refer to Fig. \ref{fig:Samples_Plot2} for some examples of generated images.

\begin{table*}[t]
    \caption{}
    \begin{subtable}{1\textwidth}
    \centering
    \setlength{\tabcolsep}{4pt}
    \caption{
    In this table, we assess ND methods on various datasets against PGD-1000, AutoAttack (AA), AdaptiveAutoAttack (A3), and a clean setting, using the AUROC(\%) metric. Perturbations: $ \epsilon = \frac{8}{255} $ for low-res and $ \epsilon = \frac{2}{255} $ for high-res datasets. The best scores are \textbf{bolded}.}
    \vspace{0.1in}
    \resizebox{\linewidth}{!}{
    \begin{tabular}{ll*{5}{c}*{6}{c}c} 
    \toprule

     \noalign{\smallskip}
     \multirow{2}{*}{\textbf{\tenpt{Method}}}& \multirow{2}{*}{\textbf{\tenpt{Attack}}} & \multicolumn{5}{c}{\textbf{\tenpt{Low-Res Datasets}}} &
     \multicolumn{6}{c}{\textbf{\tenpt{High-Res Datasets}}} &
     \multirow{2}{*}{\textbf{\tenpt{Mean}}}\\
    \cmidrule(lr){3-7}  
    \cmidrule(lr){8-13}  
    & & \textbf{CIFAR10} & \textbf{CIFAR100} & \textbf{MNIST} & \textbf{FMNIST} & \textbf{SVHN} & \textbf{MVTecAD} & \textbf{Head-CT} & \textbf{BrainMRI} & \textbf{Tumor Detection} & \textbf{Covid19} & \textbf{Imagenet-30} \\
    \specialrule{0.8pt}{\aboverulesep}{\belowrulesep}

    \textsc{CSI} & Clean / PGD & 94.3 / 2.7& 89.6 / 2.5& 93.8 / 0.0& 92.7 / 4.1& \textbf{96.0} / 1.3& 63.6 / 0.0& 60.9 / 0.1& 93.2 / 0.0& 85.3 / 0.0& 65.1 / 0.0& 91.6 / 0.3 & 84.2 / 1.0\\
    \specialrule{0.8pt}{\aboverulesep}{\belowrulesep}
    
    \textsc{MSAD} & Clean / PGD & 97.2 / 0.0& 96.4 / 2.6& 96.0 / 0.0& 94.2 / 0.0& 63.1 / 0.5& 87.2 / 0.4& 59.4 / 0.0&\textbf{99.9} / 1.5& {95.1} / 0.0& {89.2} / 4.0&96.9 / 0.0 & {88.6} / 0.8 \\
    \specialrule{0.8pt}{\aboverulesep}{\belowrulesep}
    
    \textsc{Transformaly} & Clean / PGD & {98.3} / 0.0& {97.3} / 4.1& 94.8 / 9.9& 94.4 / 0.2& 55.4 / 0.3& {87.9} / 0.0& 78.1 / 5.8& {98.3} / 4.5& \textbf{97.4} / 6.4& \textbf{91.0} / 9.1&  {97.8} / 0.0 & \textbf{90.1} / 3.7\\
    \specialrule{0.8pt}{\aboverulesep}{\belowrulesep}
    
    \textsc{EXOE} & Clean / PGD & \textbf{99.6} / 0.3& \textbf{97.8} / 0.0& 96.0 / 0.0& {94.7} / 1.8& 68.2 / 0.0& 76.2 / 0.2& 82.4 / 0.1 & 86.2 / 0.1& 79.3 / 0.0&72.5 / 0.8 &\textbf{98.1} / 0.0 & 86.5 / 0.3 \\
    \specialrule{0.8pt}{\aboverulesep}{\belowrulesep}
    
    \textsc{PatchCore} & Clean / PGD & 68.3 / 0.0& 66.8 / 0.0& 83.2 / 0.0& 77.4 / 0.0& 52.1 / 3.0& \textbf{99.6} / 6.5& \textbf{98.5} / 1.3& 91.4 / 0.0& 92.8 / 9.2& 77.7 / 3.8& 74.2 / 0.0 & 80.2 / 2.3 \\
    \specialrule{0.8pt}{\aboverulesep}{\belowrulesep}
    
    \textsc{PrincipaLS} & Clean / PGD & 57.7 / {23.6}& 52.0 / {15.3}& {97.3} / {76.4}& 91.0 / {60.8}& 63.0 / {30.3}& 63.8 / \textbf{24.0}& 68.9 / {26.8}& 70.2 / {32.9}& 73.5 / {24.4}& 54.2 / {15.1}& 61.4 / 18.7 & 68.4 / {31.7} \\
    \specialrule{0.8pt}{\aboverulesep}{\belowrulesep}
    
    \textsc{OCSDF} & Clean / PGD & 57.1 / 22.9& 48.2 / 14.6& 95.5 / 60.8& 90.6 / 53.2& 58.1 / 23.0& 58.7 / 4.8& 62.4 / 13.0& 63.2 / 18.6& 65.2 / 16.3& 46.1 / 8.4&  62.0 / {24.8} & 64.3 / 23.7\\
    \specialrule{0.8pt}{\aboverulesep}{\belowrulesep}
    
    \textsc{APAE} & Clean / PGD & 55.2 / 0.0& 51.8 / 0.0& 92.5 / 21.3& 86.1 / 9.7& 52.6 / 16.5& 62.1 / 3.9& 68.1 / 6.4& 55.4 / 9.1& 64.6 / 15.0& 50.7 / 9.8&54.5 / 12.8 & 63.0 / 9.5\\
    \specialrule{0.8pt}{\aboverulesep}{\belowrulesep}

   \multirow{2}{*}{\textsc{RODEO (ours)}}   & Clean / PGD & 87.4 / \textbf{70.2} & 79.6 / \textbf{62.1} & \textbf{99.4} / \textbf{94.6} & \textbf{95.6} / \textbf{87.2} & {78.6} / \textbf{33.8} & 61.5 / {14.9} & {87.3} / \textbf{68.6} & 76.3 / \textbf{68.4} & 89.0 / \textbf{67.0} & 79.6 / \textbf{58.3} & 86.1 / \textbf{73.5} & 83.7 / \textbf{63.5} \\
    \noalign{\smallskip}
    \cdashline{2-14}
    \noalign{\smallskip}
      & AA / A3 & 69.3 / 70.5 & 61.0 / 61.3 & 95.2 / 94.0 & 87.6 / 87.0 & 33.2 / 31.8 & 14.2 / 13.4 & 68.4 / 68.1 & 70.5 / 67.7 & 66.9 / 65.6 & 58.8 / 57.6 &  76.8 / 72.4 & 63.8 / 62.6\\
    
    \specialrule{0.8pt}{\aboverulesep}{\belowrulesep}
     
    \end{tabular}}
    \label{Table1:Novelty Detection}
    \end{subtable}
    \hfill
    \begin{subtable}{0.49\textwidth}
    \begin{minipage}[t]{\linewidth}
    \setlength{\tabcolsep}{8pt} 

\caption{Comparison of OSR method performance on MNIST, FMNIST, CIFAR10, and CIFAR100 datasets under PGD-1000, AutoAttack (AA), AdaptiveAutoAttack (A3), and clean conditions, evaluated using AUROC(\%).}

\vspace{0.1in} 
\resizebox{\columnwidth}{!}{
\begin{tabular}{ll*{5}{c}} 
\toprule
\noalign{\smallskip}

    \multicolumn{1}{l}{\multirow{3}{*}{\textbf{\tenpt{Method}}}}& \multirow{3}{*}{\textbf{\tenpt{Attack}}} & \multicolumn{5}{c}{\textbf{\tenpt{Dataset}}} \\
    \cmidrule(lr){3-6}  
     & &\multirow{2}{*}{\textbf{MNIST}}&\multirow{2}{*}{\textbf{FMNIST}}&\multirow{2}{*}{\textbf{CIFAR10}} & \multirow{2}{*}{\textbf{CIFAR100}} 
   \\

\noalign{\smallskip\smallskip\smallskip}

\midrule


\multirow{1}{*}{\textsc{ViT-MSP}} &Clean / PGD&{92.4} / 2.9 & {87.6} / 2.0 & \textbf{96.8} / 1.6 &{92.1} / 0.0
\\ 
\midrule

\multirow{1}{*}{$\text{\textsc{AT}}^*$} &Clean / PGD&80.2 / 36.1&72.5 / 29.8&65.2 / 20.6&61.7 / 17.9
\\ 
\midrule
 
\multirow{1}{*}{\textsc{ATOM}} &Clean / PGD&74.8 / 4.1&64.3 / 4.2&68.3 / 5.0&51.4 / 2.6
\\ 
\midrule

\noalign{\smallskip}
\multirow{1}{*}{\textsc{ALOE}} &Clean / PGD&79.5 / 37.3&72.6 / 28.5&52.4 / 25.6&49.8 / 18.2
\\ 
\midrule
\noalign{\smallskip}

\multirow{1}{*}{\textsc{PLP}} &Clean / PGD& 88.3 / 0.4 & 82.6 / 0.0& {94.1} / 0.0& \textbf{92.7} / 3.1
\\ 
\noalign{\smallskip}
\midrule

\multirow{1}{*}{\textsc{ATD}} &Clean / PGD&68.7 / {56.5}&59.6 / {42.1}&49.0 / {32.4}&50.5 / \textbf{36.1}
\\ 
\midrule
\noalign{\smallskip}
\multirow{2}{*}{\textsc{RODEO}} &Clean / PGD&\textbf{97.2} / \textbf{85.0}&\textbf{87.7} / \textbf{65.3}&79.6 / \textbf{62.7}&64.1 / {35.3}
\\ 
\noalign{\smallskip}
\cdashline{2-6}
\noalign{\smallskip}

& AA / A3 & 86.4 / 84.1& 66.8 / 62.9& 63.5 / 63.0& 36.9 / 35.4
\\

\bottomrule
 
\end{tabular}
}
\scriptsize{* indicates the model was trained without using OE.}
\label{Table2.b:OSR}
    \end{minipage}
    \end{subtable}
    \hfill
    \begin{subtable}{.49\textwidth}
    \begin{minipage}[t]{\linewidth}
    \flushright
    \setlength{\tabcolsep}{21.5pt} 
\caption{Comparison of OOD detection performance using CIFAR10 as in-distribution and CIFAR100 as out-of-distribution, and vice versa, under PGD-1000, AutoAttack (AA), AdaptiveAutoAttack (A3), and clean conditions, measured by AUROC(\%).}
\vspace{0.08in} 
\resizebox{\columnwidth}{!}{
\begin{tabular}{ll*{3}{c}} 
\toprule
\noalign{\smallskip}

    \multicolumn{1}{l}{\multirow{3}{*}{\textbf{\tenpt{Method}}}}& \multirow{3}{*}{\textbf{\tenpt{Attack}}} & \multicolumn{2}{c}{\textbf{\tenpt{In-Dataset}}} \\
    \cmidrule(lr){3-4}  
     & &\multirow{2}{*}{\textbf{CIFAR10}}&\multirow{2}{*}{\textbf{CIFAR100}}
   \\

\noalign{\smallskip\smallskip\smallskip}

\midrule

\multirow{1}{*}{\textsc{ViT-MSP}} &Clean / PGD&\textbf{99.5} / 0.0&\textbf{95.1} / 0.0
\\ 

\midrule

\multirow{1}{*}{$\text{\textsc{AT}}^*$} &Clean / PGD&80.5 / 18.9&70.0 / 12.7
\\ 
\midrule
 
\multirow{1}{*}{\textsc{ATOM}} &Clean / PGD&82.7 / 24.4&91.6 / 3.6
\\ 
\midrule

\noalign{\smallskip}
\multirow{1}{*}{\textsc{ALOE}} &Clean / PGD&97.8 / 5.0&79.3 / 24.9
\\ 
\midrule
\noalign{\smallskip}

\multirow{1}{*}{\textsc{PLP}} &Clean / PGD& {98.4} / 0.1 & {93.7} / 0.0
\\ 
\midrule

\multirow{1}{*}{\textsc{ATD}} &Clean / PGD&94.3 / {69.1}&87.7 / {54.8}
\\ 
\midrule
\noalign{\smallskip}
\multirow{2}{*}{\textsc{RODEO}} &Clean / PGD&93.2 / \textbf{69.5}&88.1 / \textbf{64.7}
\\
\noalign{\smallskip}
\cdashline{2-4}
\noalign{\smallskip}
& AA / A3&69.0 / 68.8&65.3 / 63.2
\\

\bottomrule
 
\end{tabular}
}
\scriptsize{* indicates the model was trained without using OE.}
\raggedright

\label{Table2.a:OOD}
    \end{minipage}
    \end{subtable}
\end{table*}

\textbf{Filtering Generated Images.} There is a concern that generated images may still belong to the inlier distribution, which can potentially lead to misleading information in subsequent steps. To mitigate this issue, we have implemented a method that involves defining a threshold to identify and exclude data that falls within the in-distribution. To determine the threshold, we utilize the ImageNet dataset and the CLIP similarity score to quantify the mismatch. We calculate the CLIP score for the synthesized data and its corresponding inlier label. If the computed CLIP score exceeds the threshold, it indicates that the generated data likely belongs to the in-distribution and should be excluded from the outlier set. Assuming the ImageNet dataset includes $M$ classes and each class contains $N$ data samples, let $\mathcal{X} = \{ x_1^1, x_2^1, \ldots, x_N^M \}$ represent the set of all data samples, and let $\mathcal{Y} = \{ y_1, \ldots, y_M \}$ represent the set of all labels, where $x_k^l$ denotes the $k^{\text{th}}$ data sample with label $y_l$. The threshold is then defined as:

{\small

  \begin{align}
\tau_{\text {image}} =   \frac{  \sum_{i=1}^{N}\sum_{j=1}^{M}\sum_{  r=1,r\neq j}^M \mathcal{D}\left(E_I(x_{i}^j) , E_T(y_{r}) \right)}{MN(M-1)}
\end{align}}

\textbf{Model Selection.}\ During our OE generation process, CLIP encoders receive input $x_t$, which is a noisy image. Since the public CLIP model is trained on noise-free images, this discrepancy leads to the generation of low-quality data, as observed in \cite{nichol2021glide}. Consequently, we opt for the smaller CLIP model proposed by \cite{nichol2021glide}, which has been trained on noisy image datasets. It is noteworthy that this model has been trained on 67 million samples and is still well-suited for our pipeline. It can generate outlier samples that has not been exposed to during training. For more details, see Appendix (Sec. \ref{best_fit}).

\subsection{Training \& Test Step}
\textbf{Adversarial Training.} During training, we have access to an inlier dataset $\mathcal{D}^{in}$ consisting of pairs $(x_i, y_i)$ where $y_i \in \{1, ..., K \}$, and we augment it with generated OE $D^{gen}$ consisting of pairs $(x_i, K+1)$ as auxiliary outliers to obtain $\mathcal{D}^{train} = \mathcal{D}^{in} \cup D^{gen}$. We then adversarially train a classifier $f_{\theta}$ with the standard cross-entropy loss $\ell_{\theta}$:$   \quad \min_{\theta} E_{(x,y) \sim \mathcal{D}_{train}} \ \ \max_{\| x^* - x \|_{\infty} \leq \epsilon} \ell_{\theta}(x^*, y)
$
with the minimization and maximization performed respectively by Adam and PGD-10. For evaluation purposes, we utilize a dataset $D^{test}$ that consists of both inlier and outlier samples.

 \textbf{Adversarial Testing.} During test time, we utilize the $(K+1)$-th logit of $f_{\theta}$ as the anomaly score, which corresponds to the class of auxiliary outliers in the training phase. For evaluating our model, along with other methods, we specifically target both inlier and outlier samples with several end-to-end adversarial attacks. Our objective is to cause the detectors to produce erroneous detection results by decreasing anomaly scores for outlier samples and increasing them for inlier samples. We set the value of $\epsilon$ to $\frac{8}{255}$ for low-resolution datasets and $\frac{2}{255}$ for high-resolution ones. For the PGD attack, we use 10 random restarts for the attack, with random initializations within the range of $(-\epsilon, \epsilon)$, and perform $N = 1000$ steps. Furthermore, we select the attack step size as $\alpha=2.5 \times \frac{\epsilon}{N}$.
In addition to the PGD-1000 attack, we have evaluated the models using AutoAttack (AA) \cite{croce2020reliable} and Adaptive AutoAttack (A3) \cite{liu2022practical}. It is important to note that our reported robust performance for some methods differs from their published results. This discrepancy arises because we used a larger epsilon value and targeted all test samples, unlike some methods that focus solely on outliers or inliers. Furthermore, in the Appendix (see Table \hyperref[Table 1.a:Novelty Detection for Low-Resolution_0]{8}), we have evaluated the models under black-box attack \cite{guo2019simple}. Due to limited space, the preliminaries of adversarial attacks, their details, and how we adopt them for targeting the outlier detection method can be found in Appendix \ref{Adversarial_Attack_detail}. Additionally, more information on the evaluation metrics, datasets, and implementation is provided in Appendix~(Sec.~\ref{imple}). The hyperparameters of our method are provided in Tables \ref{tab:adv_training_hyperparams} and \ref{tab:gen_step_hyperparams}.

\begin{table*}[h]
 
\caption{AUROC (\%) of the detector model after adversarial training with outliers generated by different OE techniques in clean data and under PGD-1000 attack evaluation. The results indicate that our adaptive OE method outperforms other methods in terms of improving robust detection. The experiments were conducted in the ND setting (Clean/PGD-1000).}
\label{Table:Ablation_11}
\setlength{\tabcolsep}{15pt} 
\vspace{0.1in}
\resizebox{\linewidth}{!}{
\begin{tabular}{ll*{15}{c}} 
\specialrule{1.5pt}{\aboverulesep}{\belowrulesep}
\noalign{\smallskip}

\multicolumn{1}{l}{\multirow{2}{*}{\textbf{Exposure Technique}}} & \multicolumn{1}{l}{\multirow{2}{*}{\textbf{Measure}}} & \multicolumn{6}{c}{\large{\textbf{Target Dataset}}} & \multicolumn{1}{c}{\multirow{2}{*}{\textbf{Mean}}} \\
\cmidrule{3-8}
& &\small{\textbf{CIFAR10}}&\small{\textbf{MNIST}}&\small{\textbf{FMNIST}}&\small{\textbf{MVTec-AD}}&\small{\textbf{Head-CT}}& \small{\textbf{Covid19}}&  \\
\specialrule{1.5pt}{\aboverulesep}{\belowrulesep}



\multirow{2}{*}{\textsc{Baseline  OE (Gaussian Noise)}} & Clean / PGD & 64.4 / 15.2 & 60.1 / 11.6 & 62.7 / 15.0 & 41.9 / 0.0 & 59.0 / 0.5 & 40.7 / 0.0 & 54.8 / 7.1 \\
& FDC & 0.024 & 0.005 & 0.015 & 0.008 & 0.004 & 0.002 & 0.010 \\
\cmidrule(lr){1-1}
\cmidrule(lr){2-11}

\multirow{2}{*}{\textsc{Baseline  OE (ImageNet)}} & Clean / PGD & {87.3} / {69.3} & {90.0} / 42.8 & {93.0} / {82.0} & \textbf{64.6} / 0.0 & 61.8 / 1.3 & 62.7 / 23.4 & {76.6} / {36.5} \\
& FDC & 3.205 & 0.152 & {1.320} & 0.848 & {1.739} & {1.702} & {1.971} \\
\cmidrule(lr){1-1}
\cmidrule(lr){2-11}

\multirow{2}{*}{\textsc{Mixup with ImageNet}} & Clean / PGD &59.4 / 30.8 & 59.6 / 1.7 & 74.2 / 47.8 & 58.5 / 0.5 & 54.4 / {20.6} & {69.2} / {50.2} & 62.6 / 25.3 \\
& FDC & {3.342} & 0.031 & 0.584 & {0.938} & 1.001 & 0.969 & 1.872 \\
\cmidrule(lr){1-1}
\cmidrule(lr){2-11}

\multirow{2}{*}{\textsc{FITYMI}} & Clean / PGD &29.5 / 15.5 & 76.0 / 51.1 & 52.2 / 30.6 & 43.5 / 7.2 & 63.7 / 6.9 & 42.7 / 12.4 & 56.3 / 20.6 \\
& FDC & 0.683 & 0.016 & 0.064 & 0.457 & 0.028 & 0.029 & 0.249 \\

\cmidrule(lr){1-1}
\cmidrule(lr){2-11}

\multirow{2}{*}{\textsc{GOE}} & Clean / PGD &67.4 / 38.0 & 80.3 / {58.3} & 63.7 / 47.1 & 62.8 / 18.0 & {71.5} / 1.4 & 39.6 / 13.7 & 64.2 /29.4 \\
& FDC & 0.589 & 1.128 & 0.547 & 0.838 & 0.128 & 0.011 & 0.613 \\
\cmidrule(lr){1-1}
\cmidrule(lr){2-11}

\multirow{2}{*}{\textsc{Dream OOD  }} & Clean / PGD & 58.2 / 24.7 & 80.5 / 51.4 & 66.8 / 45.9 & 55.0 / {12.7} & 69.9 / 1.2 & 44.1 / 0.1 & 62.4 / 22.7 \\
& FDC & 0.619 & {1.210} & 0.515 & 0.760 & 0.280 & 0.055 & 0.573 \\
\cmidrule(lr){1-1}
\cmidrule(lr){2-11}

\multirow{2}{*}{\textbf{\textsc{Adaptive OE (Ours)}}} & Clean / PGD &\textbf{87.4} / \textbf{70.2} & \textbf{99.4} / \textbf{94.6} & \textbf{95.6} / \textbf{87.2} & {61.5} / \textbf{14.9} & \textbf{87.3} / \textbf{68.6} & \textbf{79.6} / \textbf{58.3} & \textbf{85.1} / \textbf{65.6} \\
& FDC & \textbf{3.674} & \textbf{3.902} & \textbf{3.046} & \textbf{1.160} & \textbf{3.476} & \textbf{3.059} & \textbf{3.325} \\
\specialrule{1.5pt}{\aboverulesep}{\belowrulesep}

\end{tabular}}

\end{table*}

\section{Experimental Results}
In this section, we conduct comprehensive experiments to evaluate previous outlier detection methods, including both standard and adversarially trained approaches, as well as our own method, under clean and various adversarial attack scenarios. Our experiments are organized into three categories and are presented in Tables \ref{Table1:Novelty Detection}, \ref{Table2.b:OSR}, and \ref{Table2.a:OOD}. We provide a brief overview of each category below, and further details can be found in the Appendix (sec. \ref{osr_app}, \ref{ood_app}).

\textbf{Novelty Detection.}  For each dataset containing $N$ distinct classes, we perform $N$ separate experiments and average the results. In every individual experiment, samples from a single class are used as the inlier data, while instances from the remaining $N-1$ classes are considered outliers. Table \ref{Table1:Novelty Detection} presents the results of this setup. 

\textbf{Open-Set Recognition.}
For this task, each dataset was randomly split into an inlier set and an outlier set at a 60/40 ratio. This random splitting was repeated 5 times. The model was exclusively trained on the inlier set samples, with average AUROC scores reported. Results are presented in Table~\ref{Table2.b:OSR}.
 
 \textbf{Out-Of-Distribution Detection.}
For the task of OOD detection, we considered CIFAR10 and CIFAR100 datasets as inlier datasets in separate experiments. Following earlier works \cite{chen2020robust}, and \cite{hendrycks2018deep}, we test the model against several outlier datasets which are semantically distinct from the in-distribution datasets, including MNIST, TinyImageNet \cite{imagenet}, Places365 \cite{places}, LSUN \cite{LSUN}, iSUN \cite{iSUN}, Birds \cite{birds}, Flowers 
 \cite{flowers}, COIL-100 \cite{coil} and CIFAR10/CIFAR100  (depending on which is considered the inlier dataset). For any given in-distribution dataset, the results are averaged over the OOD datasets. We have provided the results of this task in Table~\ref{Table2.a:OOD}.

As the results indicate, RODEO demonstrates significant performance in robust outlier detection, outperforming others by a large margin under various strong attacks. Notably, in open-world applications where robustness is crucial, a slight decrease in clean performance is an acceptable trade-off for enhanced robustness. Our results support this stance, where in ND setup, achieving an average of 83.7\% in clean setting and 63.5\% in adversarial scenario across various datasets. RODEO surpasses recent methods in clean detection, such as EXOE, which utilizes pretrained CLIP. While achieving 86.5\% in clean settings, it experiences a substantial drop to 0.3\% in adversarial settings. RODEO also shows superiority in OSR and OOD detection setups. More experiments, including performance in non-adversarial training setups, can be found in Appendix \ref{Table_NON_AT}.

\section{Ablation Study} \label{Ablation_Section}
In this section, we quantitatively evaluate our adaptive OE method pipeline in comparison to alternative OE methods, as presented in Fig. \ref{fig:Method_Plot}. Our experiment is divided into two categories. In the first category, we substituted our adaptive OE method with an alternative one and reported the performance of our outlier detection method. In the second category, we compared OE methods by assessing their exposed auxiliary data, using the Fréchet Inception Distance (FID) \cite{DBLP:journals/corr/HeuselRUNKH17} and Density \& Convergence \cite{naeem2020reliable} metrics. FID measures the distance between two sets of image distributions, with higher values indicating greater distance. The Density \& Convergence metric measures diversity, where higher values indicate more diversity. For a unified comparison, we define 
\begin{equation}
   \text{FDC}=\log(1 + \frac{\text{Density} \times \text{Coverage}}{\text{FID}} \times 10^4) 
\end{equation}using the logarithm to scale values and adding one to ensure inputs to the logarithm are greater than zero. A higher FDC value indicates that the generated outliers have more diversity and are closer to the in-distribution. The results of the experiments are presented in Table \ref{Table:Ablation_11}. For more details about FDC, please refer to Sec. \ref{DENSITY}. Additionally, more comprehensive results can be found in the Appendix \ref{Table:Ablation_21}.

\textbf{Analyzing Ablation Study.} \label{limit_diff}
Ablation study results indicate that alternative OE methods underperform compared to RODEO in enhancing robust outlier detection, closeness to in-distribution, and diversity. RODEO's superiority stems from its pixel-based generation process that shifts inliers to outliers. This process, leveraging entire ID information from both text and image domains, was overlooked by methods like GOE. Meanwhile, Dream-OOD, despite leveraging both text and image information and being trained on a significantly larger dataset of 5 billion data points compared to RODEO's 67 million, underperforms due to its methodology of generating images in the embedding space. This approach is less suited for synthesizing pixel-level outliers and often leads to the generation of samples with different styles, i.e., far outliers, attributed to the bias of its backbone trained on LAION~\cite{schuhmann2022laion, naik2023social}. We provide a visual representation in the Appendix to showcase our method's capability to generate outliers specifically for unseen image domains, compared to alternatives.

\section{Conclusion}

In conclusion, our work introduces RODEO, a novel and effective approach for enhancing the robustness of outlier detection methods against adversarial attacks. By proposing a novel data-centric approach, we strategically craft informative OOD samples, allowing RODEO to achieve superior detection performance under both clean and adversarial evaluation conditions. We verify RODEO through comprehensive ablation experiments on its various components. Moreover, our extensive experiments across real-world datasets, as well as under various strong attacks, confirm our method's effectiveness, setting a new benchmark for future research in reliable outlier detection.

\section{Limitations}
\addcontentsline{toc}{section}{Limitations}

This study aims to enhance the adversarial detection performance of outlier detection tasks. Despite significant advancements in adversarial detection, our clean performance still lags behind existing state-of-the-art methods. The trade-off between clean and adversarial test performance is well-documented in the literature \cite{zhang2019theoretically,madry2017towards,schmidt2018adversarially}. Our work is also subject to these trade-offs. However, we have also provided results for scenarios where standard training is performed instead of adversarial training, which leads to increased clean performance.
\clearpage

\section*{Impact Statement}
\addcontentsline{toc}{section}{Impact Statement}
This paper aims to enhance the field of Machine Learning by proposing an adversarially robust method for outlier detection. There are many potential societal consequences of our work, none of which we feel must be specifically highlighted here.

\section*{Acknowledgements}
\addcontentsline{toc}{section}{Acknowledgements}
We thank Mohammad Sabokrou, Mohammadreza Salehi, Mojtaba Nafez, and the anonymous reviewers for their helpful discussions and feedback on this work.

\nocite{langley00}

\bibliography{RODEO}

\begin{thebibliography}{112}
\providecommand{\natexlab}[1]{#1}
\providecommand{\url}[1]{\texttt{#1}}
\expandafter\ifx\csname urlstyle\endcsname\relax
  \providecommand{\doi}[1]{doi: #1}\else
  \providecommand{\doi}{doi: \begingroup \urlstyle{rm}\Url}\fi

\bibitem[Adaloglou et~al.(2023)Adaloglou, Michels, Kaiser, and Kollmann]{adaloglou2023adapting}
Adaloglou, N., Michels, F., Kaiser, T., and Kollmann, M.
\newblock Adapting contrastive language-image pretrained (clip) models for out-of-distribution detection, 2023.

\bibitem[Akhtar \& Mian(2018)Akhtar and Mian]{akhtar2018threat}
Akhtar, N. and Mian, A.
\newblock Threat of adversarial attacks on deep learning in computer vision: A survey.
\newblock \emph{Ieee Access}, 6:\penalty0 14410--14430, 2018.

\bibitem[Avrahami et~al.(2022)Avrahami, Lischinski, and Fried]{avrahami2022blended}
Avrahami, O., Lischinski, D., and Fried, O.
\newblock Blended diffusion for text-driven editing of natural images.
\newblock In \emph{Proceedings of the IEEE/CVF Conference on Computer Vision and Pattern Recognition}, pp.\  18208--18218, 2022.

\bibitem[Azizmalayeri et~al.(2022)Azizmalayeri, Soltani~Moakhar, Zarei, Zohrabi, Manzuri, and Rohban]{azizmalayeri2022your}
Azizmalayeri, M., Soltani~Moakhar, A., Zarei, A., Zohrabi, R., Manzuri, M., and Rohban, M.~H.
\newblock Your out-of-distribution detection method is not robust!
\newblock \emph{Advances in Neural Information Processing Systems}, 35:\penalty0 4887--4901, 2022.

\bibitem[Bendale \& Boult(2015)Bendale and Boult]{bendale2015towards}
Bendale, A. and Boult, T.
\newblock Towards open world recognition.
\newblock In \emph{Proceedings of the IEEE Conference on Computer Vision and Pattern Recognition (CVPR)}, pp.\  1893--1902, 2015.

\bibitem[Bendale \& Boult(2016)Bendale and Boult]{openmax}
Bendale, A. and Boult, T.~E.
\newblock Towards open set deep networks.
\newblock In \emph{2016 IEEE Conference on Computer Vision and Pattern Recognition (CVPR)}, pp.\  1563--1572, Los Alamitos, CA, USA, jun 2016. IEEE Computer Society.
\newblock \doi{10.1109/CVPR.2016.173}.
\newblock URL \url{https://doi.ieeecomputersociety.org/10.1109/CVPR.2016.173}.

\bibitem[Bergman et~al.(2020)Bergman, Cohen, and Hoshen]{bergman2020deep}
Bergman, L., Cohen, N., and Hoshen, Y.
\newblock Deep nearest neighbor anomaly detection.
\newblock \emph{arXiv preprint arXiv:2002.10445}, 2020.

\bibitem[Bergmann et~al.(2019)Bergmann, Fauser, Sattlegger, and Steger]{bergmann2019mvtec}
Bergmann, P., Fauser, M., Sattlegger, D., and Steger, C.
\newblock Mvtec ad--a comprehensive real-world dataset for unsupervised anomaly detection.
\newblock In \emph{Proceedings of the IEEE/CVF conference on computer vision and pattern recognition}, pp.\  9592--9600, 2019.

\bibitem[B{\'e}thune et~al.(2023)B{\'e}thune, Novello, Boissin, Coiffier, Serrurier, Vincenot, and Troya-Galvis]{bethune2023robust}
B{\'e}thune, L., Novello, P., Boissin, T., Coiffier, G., Serrurier, M., Vincenot, Q., and Troya-Galvis, A.
\newblock Robust one-class classification with signed distance function using 1-lipschitz neural networks.
\newblock \emph{arXiv preprint arXiv:2303.01978}, 2023.

\bibitem[Bhuvaji et~al.(2020)Bhuvaji, Kadam, Bhumkar, Dedge, and Kanchan]{brainmri}
Bhuvaji, S., Kadam, A., Bhumkar, P., Dedge, S., and Kanchan, S.
\newblock Brain tumor classification (mri), 2020.
\newblock URL \url{https://www.kaggle.com/dsv/1183165}.

\bibitem[Cao \& Zhang(2022)Cao and Zhang]{cao2022deep}
Cao, S. and Zhang, Z.
\newblock Deep hybrid models for out-of-distribution detection.
\newblock In \emph{Proceedings of the IEEE/CVF Conference on Computer Vision and Pattern Recognition}, pp.\  4733--4743, 2022.

\bibitem[Chen et~al.(2020)Chen, Li, Wu, Liang, and Jha]{chen2020robust}
Chen, J., Li, Y., Wu, X., Liang, Y., and Jha, S.
\newblock Robust out-of-distribution detection for neural networks.
\newblock \emph{arXiv preprint arXiv:2003.09711}, 2020.

\bibitem[Chen et~al.(2021)Chen, Li, Wu, Liang, and Jha]{chen2021atom}
Chen, J., Li, Y., Wu, X., Liang, Y., and Jha, S.
\newblock Atom: Robustifying out-of-distribution detection using outlier mining.
\newblock In \emph{Machine Learning and Knowledge Discovery in Databases. Research Track: European Conference, ECML PKDD 2021, Bilbao, Spain, September 13--17, 2021, Proceedings, Part III 21}, pp.\  430--445. Springer, 2021.

\bibitem[Cohen et~al.(2020)Cohen, Morrison, and Dao]{cohen2020covid}
Cohen, J.~P., Morrison, P., and Dao, L.
\newblock Covid-19 image data collection.
\newblock \emph{arXiv}, 2020.
\newblock URL \url{https://github.com/ieee8023/covid-chestxray-dataset}.

\bibitem[Cohen \& Avidan(2021)Cohen and Avidan]{cohen2021transformaly}
Cohen, M.~J. and Avidan, S.
\newblock Transformaly--two (feature spaces) are better than one.
\newblock \emph{arXiv preprint arXiv:2112.04185}, 2021.

\bibitem[Cortes \& Vapnik(1995)Cortes and Vapnik]{cortes1995support}
Cortes, C. and Vapnik, V.
\newblock Support-vector networks.
\newblock \emph{Machine learning}, 20\penalty0 (3):\penalty0 273--297, 1995.

\bibitem[Croce \& Hein(2020)Croce and Hein]{croce2020reliable}
Croce, F. and Hein, M.
\newblock Reliable evaluation of adversarial robustness with an ensemble of diverse parameter-free attacks.
\newblock In \emph{International conference on machine learning}, pp.\  2206--2216. PMLR, 2020.

\bibitem[Croitoru et~al.(2023)Croitoru, Hondru, Ionescu, and Shah]{croitoru2023diffusion}
Croitoru, F.-A., Hondru, V., Ionescu, R.~T., and Shah, M.
\newblock Diffusion models in vision: A survey.
\newblock \emph{IEEE Transactions on Pattern Analysis and Machine Intelligence}, 2023.

\bibitem[Deng et~al.(2009)Deng, Dong, Socher, Li, Li, and Li]{imagenet}
Deng, J., Dong, W., Socher, R., Li, L.-J., Li, K., and Li, F.-F.
\newblock Imagenet: a large-scale hierarchical image database.
\newblock pp.\  248--255, 06 2009.
\newblock \doi{10.1109/CVPR.2009.5206848}.

\bibitem[Devlin et~al.(2018)Devlin, Chang, Lee, and Toutanova]{devlin2018bert}
Devlin, J., Chang, M.-W., Lee, K., and Toutanova, K.
\newblock Bert: Pre-training of deep bidirectional transformers for language understanding.
\newblock \emph{arXiv preprint arXiv:1810.04805}, 2018.

\bibitem[Dhariwal \& Nichol(2021)Dhariwal and Nichol]{dhariwal2021diffusion}
Dhariwal, P. and Nichol, A.
\newblock Diffusion models beat gans on image synthesis.
\newblock \emph{Advances in Neural Information Processing Systems}, 34:\penalty0 8780--8794, 2021.

\bibitem[Dosovitskiy et~al.(2020)Dosovitskiy, Beyer, Kolesnikov, Weissenborn, Zhai, Unterthiner, Dehghani, Minderer, Heigold, Gelly, et~al.]{dosovitskiy2020image}
Dosovitskiy, A., Beyer, L., Kolesnikov, A., Weissenborn, D., Zhai, X., Unterthiner, T., Dehghani, M., Minderer, M., Heigold, G., Gelly, S., et~al.
\newblock An image is worth 16x16 words: Transformers for image recognition at scale.
\newblock \emph{arXiv preprint arXiv:2010.11929}, 2020.

\bibitem[Drummond \& Shearer(2006)Drummond and Shearer]{drummond2006open}
Drummond, N. and Shearer, R.
\newblock The open world assumption.
\newblock In \emph{eSI Workshop: The Closed World of Databases meets the Open World of the Semantic Web}, volume~15, pp.\ ~1, 2006.

\bibitem[Du et~al.(2022)Du, Wang, Cai, and Li]{du2022vos}
Du, X., Wang, Z., Cai, M., and Li, Y.
\newblock Vos: Learning what you don't know by virtual outlier synthesis.
\newblock \emph{arXiv preprint arXiv:2202.01197}, 2022.

\bibitem[Du et~al.(2023)Du, Sun, Zhu, and Li]{du2023dream}
Du, X., Sun, Y., Zhu, X., and Li, Y.
\newblock Dream the impossible: Outlier imagination with diffusion models.
\newblock \emph{arXiv preprint arXiv:2309.13415}, 2023.

\bibitem[Ebrahimi et~al.(2024{\natexlab{a}})Ebrahimi, Azari, Iravani, Alizadeh, Taghavi, and Sameti]{ebrahimi2024sharifa}
Ebrahimi, S.~F., Azari, K.~A., Iravani, A., Alizadeh, H., Taghavi, Z.~S., and Sameti, H.
\newblock Sharif-str at semeval-2024 task 1: Transformer as a regression model for fine-grained scoring of textual semantic relations.
\newblock \emph{arXiv preprint arXiv:2407.12426}, 2024{\natexlab{a}}.

\bibitem[Ebrahimi et~al.(2024{\natexlab{b}})Ebrahimi, Azari, Iravani, Qazvini, Sadeghi, Taghavi, and Sameti]{ebrahimi2024sharif}
Ebrahimi, S.~F., Azari, K.~A., Iravani, A., Qazvini, A., Sadeghi, P., Taghavi, Z.~S., and Sameti, H.
\newblock Sharif-mgtd at semeval-2024 task 8: A transformer-based approach to detect machine generated text.
\newblock \emph{arXiv preprint arXiv:2407.11774}, 2024{\natexlab{b}}.

\bibitem[Esmaeilpour et~al.(2022)Esmaeilpour, Liu, Robertson, and Shu]{esmaeilpour2022zero}
Esmaeilpour, S., Liu, B., Robertson, E., and Shu, L.
\newblock Zero-shot out-of-distribution detection based on the pre-trained model clip.
\newblock In \emph{Proceedings of the AAAI conference on artificial intelligence}, volume~36, pp.\  6568--6576, 2022.

\bibitem[Fort et~al.(2021{\natexlab{a}})Fort, Ren, and Lakshminarayanan]{NEURIPS2021_3941c435}
Fort, S., Ren, J., and Lakshminarayanan, B.
\newblock Exploring the limits of out-of-distribution detection.
\newblock In Ranzato, M., Beygelzimer, A., Dauphin, Y., Liang, P., and Vaughan, J.~W. (eds.), \emph{Advances in Neural Information Processing Systems}, volume~34, pp.\  7068--7081. Curran Associates, Inc., 2021{\natexlab{a}}.
\newblock URL \url{https://proceedings.neurips.cc/paper_files/paper/2021/file/3941c4358616274ac2436eacf67fae05-Paper.pdf}.

\bibitem[Fort et~al.(2021{\natexlab{b}})Fort, Ren, and Lakshminarayanan]{fort2021exploring}
Fort, S., Ren, J., and Lakshminarayanan, B.
\newblock Exploring the limits of out-of-distribution detection.
\newblock \emph{Advances in Neural Information Processing Systems}, 34, 2021{\natexlab{b}}.

\bibitem[Franco et~al.(2023)Franco, Korth, Lorenz, Roscher, and Guennemann]{franco2023diffusion}
Franco, N., Korth, D., Lorenz, J.~M., Roscher, K., and Guennemann, S.
\newblock Diffusion denoised smoothing for certified and adversarial robust out-of-distribution detection.
\newblock \emph{arXiv preprint arXiv:2303.14961}, 2023.

\bibitem[Goodfellow et~al.(2014)Goodfellow, Shlens, and Szegedy]{goodfellow2014explaining}
Goodfellow, I.~J., Shlens, J., and Szegedy, C.
\newblock Explaining and harnessing adversarial examples.
\newblock \emph{arXiv preprint arXiv:1412.6572}, 2014.

\bibitem[Goodge et~al.(2021)Goodge, Hooi, Ng, and Ng]{goodge2021robustness}
Goodge, A., Hooi, B., Ng, S.~K., and Ng, W.~S.
\newblock Robustness of autoencoders for anomaly detection under adversarial impact.
\newblock In \emph{Proceedings of the Twenty-Ninth International Conference on International Joint Conferences on Artificial Intelligence}, pp.\  1244--1250, 2021.

\bibitem[Guo et~al.(2019)Guo, Gardner, You, Wilson, and Weinberger]{guo2019simple}
Guo, C., Gardner, J., You, Y., Wilson, A.~G., and Weinberger, K.
\newblock Simple black-box adversarial attacks.
\newblock In \emph{International Conference on Machine Learning}, pp.\  2484--2493. PMLR, 2019.

\bibitem[He et~al.(2016)He, Zhang, Ren, and Sun]{resnet}
He, K., Zhang, X., Ren, S., and Sun, J.
\newblock Deep residual learning for image recognition.
\newblock In \emph{2016 IEEE Conference on Computer Vision and Pattern Recognition (CVPR)}, pp.\  770--778, 2016.
\newblock \doi{10.1109/CVPR.2016.90}.

\bibitem[Hendrycks \& Gimpel(2017)Hendrycks and Gimpel]{hendrycks2017a}
Hendrycks, D. and Gimpel, K.
\newblock A baseline for detecting misclassified and out-of-distribution examples in neural networks.
\newblock In \emph{International Conference on Learning Representations}, 2017.
\newblock URL \url{https://openreview.net/forum?id=Hkg4TI9xl}.

\bibitem[Hendrycks et~al.(2018)Hendrycks, Mazeika, and Dietterich]{hendrycks2018deep}
Hendrycks, D., Mazeika, M., and Dietterich, T.
\newblock Deep anomaly detection with outlier exposure.
\newblock \emph{arXiv preprint arXiv:1812.04606}, 2018.

\bibitem[Hendrycks et~al.(2019)Hendrycks, Lee, and Mazeika]{hendrycks2019usingx}
Hendrycks, D., Lee, K., and Mazeika, M.
\newblock Using pre-training can improve model robustness and uncertainty.
\newblock In \emph{International conference on machine learning}, pp.\  2712--2721. PMLR, 2019.

\bibitem[Heusel et~al.(2017)Heusel, Ramsauer, Unterthiner, Nessler, Klambauer, and Hochreiter]{DBLP:journals/corr/HeuselRUNKH17}
Heusel, M., Ramsauer, H., Unterthiner, T., Nessler, B., Klambauer, G., and Hochreiter, S.
\newblock Gans trained by a two time-scale update rule converge to a nash equilibrium.
\newblock \emph{CoRR}, abs/1706.08500, 2017.
\newblock URL \url{http://arxiv.org/abs/1706.08500}.

\bibitem[Ho et~al.(2020)Ho, Jain, and Abbeel]{ho2020denoising}
Ho, J., Jain, A., and Abbeel, P.
\newblock Denoising diffusion probabilistic models.
\newblock \emph{Advances in Neural Information Processing Systems}, 33:\penalty0 6840--6851, 2020.

\bibitem[Jafari et~al.(2024)Jafari, Zhang, Zhang, and Liu]{jafari2024power}
Jafari, M., Zhang, Y., Zhang, Y., and Liu, S.
\newblock The power of few: Accelerating and enhancing data reweighting with coreset selection.
\newblock In \emph{ICASSP 2024-2024 IEEE International Conference on Acoustics, Speech and Signal Processing (ICASSP)}, pp.\  7100--7104. IEEE, 2024.

\bibitem[Kim et~al.(2022)Kim, Kwon, and Ye]{kim2022diffusionclip}
Kim, G., Kwon, T., and Ye, J.~C.
\newblock Diffusionclip: Text-guided diffusion models for robust image manipulation.
\newblock In \emph{Proceedings of the IEEE/CVF Conference on Computer Vision and Pattern Recognition}, pp.\  2426--2435, 2022.

\bibitem[Kingma \& Ba(2017)Kingma and Ba]{kingma2017adam}
Kingma, D.~P. and Ba, J.
\newblock Adam: A method for stochastic optimization, 2017.

\bibitem[Kirchheim \& Ortmeier(2022)Kirchheim and Ortmeier]{kirchheim2022outlier}
Kirchheim, K. and Ortmeier, F.
\newblock On outlier exposure with generative models.
\newblock In \emph{NeurIPS ML Safety Workshop}, 2022.

\bibitem[Kitamura(2018)]{felipe-campos-kitamura_2018}
Kitamura, F.~C.
\newblock Head ct - hemorrhage, 2018.
\newblock URL \url{https://www.kaggle.com/dsv/152137}.

\bibitem[Kong \& Ramanan(2021)Kong and Ramanan]{kong2021opengan}
Kong, S. and Ramanan, D.
\newblock Opengan: Open-set recognition via open data generation.
\newblock In \emph{Proceedings of the IEEE/CVF International Conference on Computer Vision}, pp.\  813--822, 2021.

\bibitem[Krizhevsky et~al.(2009)Krizhevsky, Hinton, et~al.]{krizhevsky2009learning}
Krizhevsky, A., Hinton, G., et~al.
\newblock Learning multiple layers of features from tiny images.
\newblock 2009.

\bibitem[LeCun \& Cortes(2010)LeCun and Cortes]{lecun-mnisthandwrittendigit-2010}
LeCun, Y. and Cortes, C.
\newblock {MNIST} handwritten digit database.
\newblock 2010.
\newblock URL \url{http://yann.lecun.com/exdb/mnist/}.

\bibitem[Lee et~al.(2018{\natexlab{a}})Lee, Lee, Lee, and Shin]{MD}
Lee, K., Lee, K., Lee, H., and Shin, J.
\newblock A simple unified framework for detecting out-of-distribution samples and adversarial attacks.
\newblock In Bengio, S., Wallach, H., Larochelle, H., Grauman, K., Cesa-Bianchi, N., and Garnett, R. (eds.), \emph{Advances in Neural Information Processing Systems}, volume~31. Curran Associates, Inc., 2018{\natexlab{a}}.
\newblock URL \url{https://proceedings.neurips.cc/paper_files/paper/2018/file/abdeb6f575ac5c6676b747bca8d09cc2-Paper.pdf}.

\bibitem[Lee et~al.(2018{\natexlab{b}})Lee, Lee, Lee, and Shin]{lee2018simple}
Lee, K., Lee, K., Lee, H., and Shin, J.
\newblock A simple unified framework for detecting out-of-distribution samples and adversarial attacks.
\newblock \emph{Advances in neural information processing systems}, 31, 2018{\natexlab{b}}.

\bibitem[Liang et~al.(2018)Liang, Li, and Srikant]{liang2018enhancing}
Liang, S., Li, Y., and Srikant, R.
\newblock Enhancing the reliability of out-of-distribution image detection in neural networks.
\newblock In \emph{International Conference on Learning Representations}, 2018.
\newblock URL \url{https://openreview.net/forum?id=H1VGkIxRZ}.

\bibitem[Liu et~al.(2022)Liu, Cheng, Gao, Liu, Zhang, and Song]{liu2022practical}
Liu, Y., Cheng, Y., Gao, L., Liu, X., Zhang, Q., and Song, J.
\newblock Practical evaluation of adversarial robustness via adaptive auto attack, 2022.

\bibitem[Liznerski et~al.(2022)Liznerski, Ruff, Vandermeulen, Franks, M{\"u}ller, and Kloft]{liznerski2022exposing}
Liznerski, P., Ruff, L., Vandermeulen, R.~A., Franks, B.~J., M{\"u}ller, K.-R., and Kloft, M.
\newblock Exposing outlier exposure: What can be learned from few, one, and zero outlier images.
\newblock \emph{arXiv preprint arXiv:2205.11474}, 2022.

\bibitem[Lo et~al.(2022)Lo, Oza, and Patel]{lo2022adversarially}
Lo, S.-Y., Oza, P., and Patel, V.~M.
\newblock Adversarially robust one-class novelty detection.
\newblock \emph{IEEE Transactions on Pattern Analysis and Machine Intelligence}, 2022.

\bibitem[Madry et~al.(2017)Madry, Makelov, Schmidt, Tsipras, and Vladu]{madry2017towards}
Madry, A., Makelov, A., Schmidt, L., Tsipras, D., and Vladu, A.
\newblock Towards deep learning models resistant to adversarial attacks.
\newblock \emph{arXiv preprint arXiv:1706.06083}, 2017.

\bibitem[Meinke et~al.(2022)Meinke, Bitterwolf, and Hein]{meinke2022provably}
Meinke, A., Bitterwolf, J., and Hein, M.
\newblock Provably adversarially robust detection of out-of-distribution data (almost) for free.
\newblock \emph{Advances in Neural Information Processing Systems}, 35:\penalty0 30167--30180, 2022.

\bibitem[Meng et~al.(2021)Meng, He, Song, Song, Wu, Zhu, and Ermon]{meng2021sdedit}
Meng, C., He, Y., Song, Y., Song, J., Wu, J., Zhu, J.-Y., and Ermon, S.
\newblock Sdedit: Guided image synthesis and editing with stochastic differential equations.
\newblock In \emph{International Conference on Learning Representations}, 2021.

\bibitem[Mikolov et~al.(2013)Mikolov, Chen, Corrado, and Dean]{mikolov2013efficient}
Mikolov, T., Chen, K., Corrado, G., and Dean, J.
\newblock Efficient estimation of word representations in vector space.
\newblock \emph{arXiv preprint arXiv:1301.3781}, 2013.

\bibitem[Mirzaei \& Mathis(2024)Mirzaei and Mathis]{mirzaei2024adversarially}
Mirzaei, H. and Mathis, M.~W.
\newblock Adversarially robust out-of-distribution detection using lyapunov-stabilized embeddings.
\newblock \emph{arXiv preprint arXiv:2410.10744}, 2024.

\bibitem[Mirzaei et~al.()Mirzaei, Ansari, Nia, Nafez, Madadi, Rezaee, Taghavi, Maleki, Shamsaie, Hajialilue, et~al.]{mirzaeiscanning}
Mirzaei, H., Ansari, A., Nia, B.~D., Nafez, M., Madadi, M., Rezaee, S., Taghavi, Z.~S., Maleki, A., Shamsaie, K., Hajialilue, M., et~al.
\newblock Scanning trojaned models using out-of-distribution samples.
\newblock In \emph{The Thirty-eighth Annual Conference on Neural Information Processing Systems}.

\bibitem[Mirzaei et~al.(2022)Mirzaei, Salehi, Shahabi, Gavves, Snoek, Sabokrou, and Rohban]{mirzaei2022fake}
Mirzaei, H., Salehi, M., Shahabi, S., Gavves, E., Snoek, C.~G., Sabokrou, M., and Rohban, M.~H.
\newblock Fake it till you make it: Near-distribution novelty detection by score-based generative models.
\newblock \emph{arXiv preprint arXiv:2205.14297}, 2022.

\bibitem[Mirzaei et~al.(2024{\natexlab{a}})Mirzaei, Jafari, Dehbashi, Taghavi, Sabokrou, and Rohban]{mirzaei2024killing}
Mirzaei, H., Jafari, M., Dehbashi, H.~R., Taghavi, Z.~S., Sabokrou, M., and Rohban, M.~H.
\newblock Killing it with zero-shot: Adversarially robust novelty detection.
\newblock In \emph{ICASSP 2024-2024 IEEE International Conference on Acoustics, Speech and Signal Processing (ICASSP)}, pp.\  7415--7419. IEEE, 2024{\natexlab{a}}.

\bibitem[Mirzaei et~al.(2024{\natexlab{b}})Mirzaei, Nafez, Jafari, Soltani, Azizmalayeri, Habibi, Sabokrou, and Rohban]{mirzaei2024universal}
Mirzaei, H., Nafez, M., Jafari, M., Soltani, M.~B., Azizmalayeri, M., Habibi, J., Sabokrou, M., and Rohban, M.~H.
\newblock Universal novelty detection through adaptive contrastive learning.
\newblock In \emph{Proceedings of the IEEE/CVF Conference on Computer Vision and Pattern Recognition}, pp.\  22914--22923, 2024{\natexlab{b}}.

\bibitem[Moakhar et~al.(2023)Moakhar, Azizmalayeri, Mirzaei, Manzuri, and Rohban]{moakhar2023seeking}
Moakhar, A.~S., Azizmalayeri, M., Mirzaei, H., Manzuri, M.~T., and Rohban, M.~H.
\newblock Seeking next layer neurons' attention for error-backpropagation-like training in a multi-agent network framework.
\newblock \emph{arXiv preprint arXiv:2310.09952}, 2023.

\bibitem[Naeem et~al.(2020)Naeem, Oh, Uh, Choi, and Yoo]{naeem2020reliable}
Naeem, M.~F., Oh, S.~J., Uh, Y., Choi, Y., and Yoo, J.
\newblock Reliable fidelity and diversity metrics for generative models, 2020.

\bibitem[Naik \& Nushi(2023)Naik and Nushi]{naik2023social}
Naik, R. and Nushi, B.
\newblock Social biases through the text-to-image generation lens.
\newblock \emph{arXiv preprint arXiv:2304.06034}, 2023.

\bibitem[Nayar \& Murase(1996)Nayar and Murase]{coil}
Nayar and Murase, H.
\newblock Columbia object image library: Coil-100.
\newblock Technical Report CUCS-006-96, Department of Computer Science, Columbia University, February 1996.

\bibitem[Nichol et~al.(2021)Nichol, Dhariwal, Ramesh, Shyam, Mishkin, McGrew, Sutskever, and Chen]{nichol2021glide}
Nichol, A., Dhariwal, P., Ramesh, A., Shyam, P., Mishkin, P., McGrew, B., Sutskever, I., and Chen, M.
\newblock Glide: Towards photorealistic image generation and editing with text-guided diffusion models.
\newblock \emph{arXiv preprint arXiv:2112.10741}, 2021.

\bibitem[Nickparvar(2021)]{msoud-nickparvar_2021}
Nickparvar, M.
\newblock Brain tumor mri dataset, 2021.
\newblock URL \url{https://www.kaggle.com/dsv/2645886}.

\bibitem[Nilsback \& Zisserman(2008)Nilsback and Zisserman]{flowers}
Nilsback, M.-E. and Zisserman, A.
\newblock Automated flower classification over a large number of classes.
\newblock In \emph{2008 Sixth Indian Conference on Computer Vision, Graphics \& Image Processing}, pp.\  722--729, 2008.
\newblock \doi{10.1109/ICVGIP.2008.47}.

\bibitem[Pang et~al.(2022)Pang, Lin, Yang, Zhu, and Yan]{pang2022robustness}
Pang, T., Lin, M., Yang, X., Zhu, J., and Yan, S.
\newblock Robustness and accuracy could be reconcilable by (proper) definition.
\newblock In \emph{International Conference on Machine Learning}, pp.\  17258--17277. PMLR, 2022.

\bibitem[Perera et~al.(2021)Perera, Oza, and Patel]{perera2021one}
Perera, P., Oza, P., and Patel, V.~M.
\newblock One-class classification: A survey.
\newblock \emph{arXiv preprint arXiv:2101.03064}, 2021.

\bibitem[Radford et~al.(2021)Radford, Kim, Hallacy, Ramesh, Goh, Agarwal, Sastry, Askell, Mishkin, Clark, et~al.]{radford2021learning}
Radford, A., Kim, J.~W., Hallacy, C., Ramesh, A., Goh, G., Agarwal, S., Sastry, G., Askell, A., Mishkin, P., Clark, J., et~al.
\newblock Learning transferable visual models from natural language supervision.
\newblock In \emph{International conference on machine learning}, pp.\  8748--8763. PMLR, 2021.

\bibitem[Rahimi et~al.(2024{\natexlab{a}})Rahimi, Amirzadeh, Sohrabi, Taghavi, and Sameti]{rahimi-etal-2024-hallusafe}
Rahimi, Z., Amirzadeh, H., Sohrabi, A., Taghavi, Z., and Sameti, H.
\newblock {H}allu{S}afe at {S}em{E}val-2024 task 6: An {NLI}-based approach to make {LLM}s safer by better detecting hallucinations and overgeneration mistakes.
\newblock In Ojha, A.~K., Do{\u{g}}ru{\"o}z, A.~S., Tayyar~Madabushi, H., Da~San~Martino, G., Rosenthal, S., and Ros{\'a}, A. (eds.), \emph{Proceedings of the 18th International Workshop on Semantic Evaluation (SemEval-2024)}, pp.\  139--147, Mexico City, Mexico, June 2024{\natexlab{a}}. Association for Computational Linguistics.
\newblock \doi{10.18653/v1/2024.semeval-1.22}.
\newblock URL \url{https://aclanthology.org/2024.semeval-1.22/}.

\bibitem[Rahimi et~al.(2024{\natexlab{b}})Rahimi, Shirzady, Taghavi, and Sameti]{rahimi-etal-2024-nimz}
Rahimi, Z., Shirzady, M.~M., Taghavi, Z., and Sameti, H.
\newblock {NIMZ} at {S}em{E}val-2024 task 9: Evaluating methods in solving brainteasers defying commonsense.
\newblock In Ojha, A.~K., Do{\u{g}}ru{\"o}z, A.~S., Tayyar~Madabushi, H., Da~San~Martino, G., Rosenthal, S., and Ros{\'a}, A. (eds.), \emph{Proceedings of the 18th International Workshop on Semantic Evaluation (SemEval-2024)}, pp.\  148--154, Mexico City, Mexico, June 2024{\natexlab{b}}. Association for Computational Linguistics.
\newblock \doi{10.18653/v1/2024.semeval-1.23}.
\newblock URL \url{https://aclanthology.org/2024.semeval-1.23/}.

\bibitem[Reiss \& Hoshen(2021)Reiss and Hoshen]{reiss2021mean}
Reiss, T. and Hoshen, Y.
\newblock Mean-shifted contrastive loss for anomaly detection.
\newblock \emph{arXiv preprint arXiv:2106.03844}, 2021.

\bibitem[Reiss et~al.(2021)Reiss, Cohen, Bergman, and Hoshen]{reiss2021panda}
Reiss, T., Cohen, N., Bergman, L., and Hoshen, Y.
\newblock Panda: Adapting pretrained features for anomaly detection and segmentation.
\newblock In \emph{Proceedings of the IEEE/CVF Conference on Computer Vision and Pattern Recognition}, pp.\  2806--2814, 2021.

\bibitem[Ren et~al.(2021)Ren, Fort, Liu, Roy, Padhy, and Lakshminarayanan]{ren2021simple}
Ren, J., Fort, S., Liu, J., Roy, A.~G., Padhy, S., and Lakshminarayanan, B.
\newblock A simple fix to mahalanobis distance for improving near-ood detection.
\newblock \emph{arXiv preprint arXiv:2106.09022}, 2021.

\bibitem[Rombach et~al.(2022)Rombach, Blattmann, Lorenz, Esser, and Ommer]{rombach2022high}
Rombach, R., Blattmann, A., Lorenz, D., Esser, P., and Ommer, B.
\newblock High-resolution image synthesis with latent diffusion models.
\newblock In \emph{Proceedings of the IEEE/CVF Conference on Computer Vision and Pattern Recognition}, pp.\  10684--10695, 2022.

\bibitem[Roth et~al.(2021)Roth, Pemula, Zepeda, Schölkopf, Brox, and Gehler]{roth2021total}
Roth, K., Pemula, L., Zepeda, J., Schölkopf, B., Brox, T., and Gehler, P.
\newblock Towards total recall in industrial anomaly detection, 2021.

\bibitem[Ruiz et~al.(2023)Ruiz, Li, Jampani, Pritch, Rubinstein, and Aberman]{ruiz2023dreambooth}
Ruiz, N., Li, Y., Jampani, V., Pritch, Y., Rubinstein, M., and Aberman, K.
\newblock Dreambooth: Fine tuning text-to-image diffusion models for subject-driven generation.
\newblock In \emph{Proceedings of the IEEE/CVF Conference on Computer Vision and Pattern Recognition}, pp.\  22500--22510, 2023.

\bibitem[Salehi et~al.(2021{\natexlab{a}})Salehi, Mirzaei, Hendrycks, Li, Rohban, and Sabokrou]{salehi2021unified}
Salehi, M., Mirzaei, H., Hendrycks, D., Li, Y., Rohban, M.~H., and Sabokrou, M.
\newblock A unified survey on anomaly, novelty, open-set, and out-of-distribution detection: Solutions and future challenges.
\newblock \emph{arXiv preprint arXiv:2110.14051}, 2021{\natexlab{a}}.

\bibitem[Salehi et~al.(2021{\natexlab{b}})Salehi, Sadjadi, Baselizadeh, Rohban, and Rabiee]{Salehi_2021_CVPR}
Salehi, M., Sadjadi, N., Baselizadeh, S., Rohban, M.~H., and Rabiee, H.~R.
\newblock Multiresolution knowledge distillation for anomaly detection.
\newblock In \emph{Proceedings of the IEEE/CVF Conference on Computer Vision and Pattern Recognition (CVPR)}, pp.\  14902--14912, June 2021{\natexlab{b}}.

\bibitem[Schmidt et~al.(2018)Schmidt, Santurkar, Tsipras, Talwar, and Madry]{schmidt2018adversarially}
Schmidt, L., Santurkar, S., Tsipras, D., Talwar, K., and Madry, A.
\newblock Adversarially robust generalization requires more data.
\newblock \emph{Advances in neural information processing systems}, 31, 2018.

\bibitem[Schuhmann et~al.(2022)Schuhmann, Beaumont, Vencu, Gordon, Wightman, Cherti, Coombes, Katta, Mullis, Wortsman, et~al.]{schuhmann2022laion}
Schuhmann, C., Beaumont, R., Vencu, R., Gordon, C., Wightman, R., Cherti, M., Coombes, T., Katta, A., Mullis, C., Wortsman, M., et~al.
\newblock Laion-5b: An open large-scale dataset for training next generation image-text models.
\newblock \emph{arXiv preprint arXiv:2210.08402}, 2022.

\bibitem[Sehwag et~al.(2021)Sehwag, Mahloujifar, Handina, Dai, Xiang, Chiang, and Mittal]{sehwag2021robust}
Sehwag, V., Mahloujifar, S., Handina, T., Dai, S., Xiang, C., Chiang, M., and Mittal, P.
\newblock Robust learning meets generative models: Can proxy distributions improve adversarial robustness?
\newblock \emph{arXiv preprint arXiv:2104.09425}, 2021.

\bibitem[Shao et~al.(2020)Shao, Perera, Yuen, and Patel]{shao2020open}
Shao, R., Perera, P., Yuen, P.~C., and Patel, V.~M.
\newblock Open-set adversarial defense.
\newblock In \emph{Computer Vision--ECCV 2020: 16th European Conference, Glasgow, UK, August 23--28, 2020, Proceedings, Part XVII 16}, pp.\  682--698. Springer, 2020.

\bibitem[Shao et~al.(2022)Shao, Perera, Yuen, and Patel]{shao2022open}
Shao, R., Perera, P., Yuen, P.~C., and Patel, V.~M.
\newblock Open-set adversarial defense with clean-adversarial mutual learning.
\newblock \emph{International Journal of Computer Vision}, 130\penalty0 (4):\penalty0 1070--1087, 2022.

\bibitem[Sohl-Dickstein et~al.(2015)Sohl-Dickstein, Weiss, Maheswaranathan, and Ganguli]{sohl2015deep}
Sohl-Dickstein, J., Weiss, E., Maheswaranathan, N., and Ganguli, S.
\newblock Deep unsupervised learning using nonequilibrium thermodynamics.
\newblock In \emph{International Conference on Machine Learning}, pp.\  2256--2265. PMLR, 2015.

\bibitem[Stutz et~al.(2019)Stutz, Hein, and Schiele]{stutz2019disentangling}
Stutz, D., Hein, M., and Schiele, B.
\newblock Disentangling adversarial robustness and generalization.
\newblock In \emph{Proceedings of the IEEE/CVF Conference on Computer Vision and Pattern Recognition}, pp.\  6976--6987, 2019.

\bibitem[Szegedy et~al.(2013)Szegedy, Zaremba, Sutskever, Bruna, Erhan, Goodfellow, and Fergus]{szegedy2013intriguing}
Szegedy, C., Zaremba, W., Sutskever, I., Bruna, J., Erhan, D., Goodfellow, I., and Fergus, R.
\newblock Intriguing properties of neural networks.
\newblock \emph{arXiv preprint arXiv:1312.6199}, 2013.

\bibitem[Tack et~al.(2020)Tack, Mo, Jeong, and Shin]{tack2020csi}
Tack, J., Mo, S., Jeong, J., and Shin, J.
\newblock Csi: Novelty detection via contrastive learning on distributionally shifted instances.
\newblock \emph{Advances in neural information processing systems}, 33:\penalty0 11839--11852, 2020.

\bibitem[Taghavi \& Mirzaei(2024)Taghavi and Mirzaei]{taghavi2024backdooring}
Taghavi, Z. and Mirzaei, H.
\newblock Backdooring outlier detection methods: A novel attack approach.
\newblock \emph{arXiv preprint arXiv:2412.05010}, 2024.

\bibitem[Taghavi et~al.(2023{\natexlab{a}})Taghavi, Naeini, Sadraei~Javaheri, Gooran, Asgari, Rabiee, and Sameti]{taghavi-etal-2023-ebhaam}
Taghavi, Z., Naeini, P.~H., Sadraei~Javaheri, M.~A., Gooran, S., Asgari, E., Rabiee, H.~R., and Sameti, H.
\newblock Ebhaam at {S}em{E}val-2023 task 1: A {CLIP}-based approach for comparing cross-modality and unimodality in visual word sense disambiguation.
\newblock In Ojha, A.~K., Do{\u{g}}ru{\"o}z, A.~S., Da~San~Martino, G., Tayyar~Madabushi, H., Kumar, R., and Sartori, E. (eds.), \emph{Proceedings of the 17th International Workshop on Semantic Evaluation (SemEval-2023)}, pp.\  1960--1964, Toronto, Canada, July 2023{\natexlab{a}}. Association for Computational Linguistics.
\newblock \doi{10.18653/v1/2023.semeval-1.269}.
\newblock URL \url{https://aclanthology.org/2023.semeval-1.269/}.

\bibitem[Taghavi et~al.(2023{\natexlab{b}})Taghavi, Gooran, Dalili, Amirzadeh, Nematbakhsh, and Sameti]{taghavi2023imaginations}
Taghavi, Z.~S., Gooran, S., Dalili, S.~A., Amirzadeh, H., Nematbakhsh, M.~J., and Sameti, H.
\newblock Imaginations of wall-e: Reconstructing experiences with an imagination-inspired module for advanced ai systems.
\newblock \emph{arXiv preprint arXiv:2308.10354}, 2023{\natexlab{b}}.

\bibitem[Taghavi et~al.(2023{\natexlab{c}})Taghavi, Satvaty, and Sameti]{taghavi2023change}
Taghavi, Z.~S., Satvaty, A., and Sameti, H.
\newblock A change of heart: Improving speech emotion recognition through speech-to-text modality conversion.
\newblock \emph{arXiv preprint arXiv:2307.11584}, 2023{\natexlab{c}}.

\bibitem[Tao et~al.(2023{\natexlab{a}})Tao, Du, Zhu, and Li]{tao2023non}
Tao, L., Du, X., Zhu, X., and Li, Y.
\newblock Non-parametric outlier synthesis.
\newblock \emph{arXiv preprint arXiv:2303.02966}, 2023{\natexlab{a}}.

\bibitem[Tao et~al.(2023{\natexlab{b}})Tao, Du, Zhu, and Li]{tao2023nonparametric}
Tao, L., Du, X., Zhu, X., and Li, Y.
\newblock Non-parametric outlier synthesis, 2023{\natexlab{b}}.

\bibitem[Van~der Maaten \& Hinton(2008)Van~der Maaten and Hinton]{van2008visualizing}
Van~der Maaten, L. and Hinton, G.
\newblock Visualizing data using t-sne.
\newblock \emph{Journal of machine learning research}, 9\penalty0 (11), 2008.

\bibitem[Wang et~al.(2022)Wang, Jia, Li, Yu, Xiong, Dong, and Liao]{wang2022bridging}
Wang, D., Jia, Z., Li, S., Yu, Y., Xiong, Y., Dong, W., and Liao, X.
\newblock Bridging pre-trained models and downstream tasks for source code understanding.
\newblock In \emph{Proceedings of the 44th International Conference on Software Engineering}, pp.\  287--298, 2022.

\bibitem[Wei et~al.(2021)Wei, Xie, and Ma]{wei2021pretrained}
Wei, C., Xie, S.~M., and Ma, T.
\newblock Why do pretrained language models help in downstream tasks? an analysis of head and prompt tuning.
\newblock \emph{Advances in Neural Information Processing Systems}, 34:\penalty0 16158--16170, 2021.

\bibitem[Welinder et~al.(2010)Welinder, Branson, Mita, Wah, Schroff, Belongie, and Perona]{birds}
Welinder, P., Branson, S., Mita, T., Wah, C., Schroff, F., Belongie, S., and Perona, P.
\newblock Caltech-ucsd birds 200.
\newblock 09 2010.

\bibitem[Xiao et~al.(2017)Xiao, Rasul, and Vollgraf]{xiao2017fashionmnist}
Xiao, H., Rasul, K., and Vollgraf, R.
\newblock Fashion-mnist: a novel image dataset for benchmarking machine learning algorithms, 2017.

\bibitem[Xing et~al.(2022)Xing, Song, and Cheng]{xing2022artificially}
Xing, Y., Song, Q., and Cheng, G.
\newblock Why do artificially generated data help adversarial robustness.
\newblock \emph{Advances in Neural Information Processing Systems}, 35:\penalty0 954--966, 2022.

\bibitem[Xu et~al.(2021)Xu, Ren, Zhang, Feng, and Xiong]{xu2021unsupervised}
Xu, K., Ren, T., Zhang, S., Feng, Y., and Xiong, C.
\newblock Unsupervised out-of-domain detection via pre-trained transformers.
\newblock \emph{arXiv preprint arXiv:2106.00948}, 2021.

\bibitem[Xu et~al.(2015)Xu, Ehinger, Zhang, Finkelstein, Kulkarni, and Xiao]{iSUN}
Xu, P., Ehinger, K.~A., Zhang, Y., Finkelstein, A., Kulkarni, S.~R., and Xiao, J.
\newblock Turkergaze: Crowdsourcing saliency with webcam based eye tracking.
\newblock \emph{CoRR}, abs/1504.06755, 2015.
\newblock URL \url{http://arxiv.org/abs/1504.06755}.

\bibitem[Yang et~al.(2021)Yang, Zhou, Li, and Liu]{yang2021generalized}
Yang, J., Zhou, K., Li, Y., and Liu, Z.
\newblock Generalized out-of-distribution detection: A survey.
\newblock \emph{arXiv preprint arXiv:2110.11334}, 2021.

\bibitem[Yu et~al.(2015)Yu, Zhang, Song, Seff, and Xiao]{LSUN}
Yu, F., Zhang, Y., Song, S., Seff, A., and Xiao, J.
\newblock {LSUN:} construction of a large-scale image dataset using deep learning with humans in the loop.
\newblock \emph{CoRR}, abs/1506.03365, 2015.
\newblock URL \url{http://arxiv.org/abs/1506.03365}.

\bibitem[Zagoruyko \& Komodakis(2017)Zagoruyko and Komodakis]{zagoruyko2017wide}
Zagoruyko, S. and Komodakis, N.
\newblock Wide residual networks, 2017.

\bibitem[Zhang et~al.(2017)Zhang, Cisse, Dauphin, and Lopez-Paz]{zhang2017mixup}
Zhang, H., Cisse, M., Dauphin, Y.~N., and Lopez-Paz, D.
\newblock mixup: Beyond empirical risk minimization.
\newblock \emph{arXiv preprint arXiv:1710.09412}, 2017.

\bibitem[Zhang et~al.(2019)Zhang, Yu, Jiao, Xing, El~Ghaoui, and Jordan]{zhang2019theoretically}
Zhang, H., Yu, Y., Jiao, J., Xing, E., El~Ghaoui, L., and Jordan, M.
\newblock Theoretically principled trade-off between robustness and accuracy.
\newblock In \emph{International conference on machine learning}, pp.\  7472--7482. PMLR, 2019.

\bibitem[Zhou et~al.(2018)Zhou, Lapedriza, Khosla, Oliva, and Torralba]{places}
Zhou, B., Lapedriza, A., Khosla, A., Oliva, A., and Torralba, A.
\newblock Places: A 10 million image database for scene recognition.
\newblock \emph{IEEE Transactions on Pattern Analysis and Machine Intelligence}, 40\penalty0 (6):\penalty0 1452--1464, 2018.
\newblock \doi{10.1109/TPAMI.2017.2723009}.

\end{thebibliography}
\bibliographystyle{icml2024}

\newpage
\appendix
\onecolumn
\newpage

\section{Evaluation Metrics \& Datasets \& Implementation Details} \label{imple}

\textbf{Evaluation Metrics} AUROC is used as a well-known classification criterion. The AUROC value is in the range [0, 1], and the closer it is to 1, the better the classifier performance.

\textbf{Datasets} For the low-resolution datasets, we included CIFAR10 \cite{krizhevsky2009learning}, CIFAR100 \cite{krizhevsky2009learning}, MNIST~\cite{lecun-mnisthandwrittendigit-2010}, and FashionMNIST~\cite{xiao2017fashionmnist}. Furthermore, we performed experiments on medical and industrial high-resolution datasets, namely Head-CT \cite{felipe-campos-kitamura_2018}, MVTec-ad \cite{bergmann2019mvtec}, Brain-MRI \cite{brainmri}, Covid19 \cite{cohen2020covid}, and Tumor Detection \cite{msoud-nickparvar_2021}. The results are available in Table \hyperref[Table 1.a:Novelty Detection for Low-Resolutionn]{1}.

 \textbf{Implementation Details} We use ResNet-18\cite{resnet} as the architecture of our neural network for the high-resolution datasets and for the low-resolution datasets, we used Wide ResNet\cite{zagoruyko2017wide}. Furthermore, RODEO is trained 100 epochs with Adam\cite{kingma2017adam} optimizer with a learning rate of 0.001 for each experiment.

\section{Algorithm}
This algorithm presents the complete approach, including all components that are integral to it.

\algrenewcommand\algorithmicrequire{\textbf{Input:}}
\algrenewcommand\algorithmicensure{\textbf{Output:}}

\begin{algorithm}
\caption{RODEO: Adversarial Training with Adaptive Exposure Dataset}
\label{alg:RODEO}

\begin{algorithmic}[1] 
\Require $\mathcal{D}_{\text{in}}$, $\mathcal{D}_{\text{val}}$, $\tau_{\text{image-val}}$, $\tau_{\text{label-val}}$, $enc_{text}$, $\mu_{Diffusion}$, $\Sigma_{Diffusion}$, $E_{I}^{CLIP}$, $E_{T}^{CLIP}$, $f_\theta$, $K$, $T_0$, $T$ \Comment{$T_0 \in [0, 0.6T]$}
\Ensure $\hat{f}_{\theta}$

\State \textbf{Near-Distribution outlier Prompt Search}
\State $\tau_{\text{label-val}} = Avg(Dist(E_{T}^{CLIP}(y_i), E_{T}^{CLIP}(y_j))) \hfill \forall (y_i,y_j) \in \mathcal{Y}(\mathcal{D}_{\text{val}})$
\For{$(i, label) \in Y$}
    \State $Prompts[i] \leftarrow enc_{text}.KNN(label)$
    \State $Prompts[i] \leftarrow Prompts[i].Remove(enc_{text}.MinDist(Prompt, Y \backslash label) < \tau_{\text{label-val}})$
    \State $Prompts[i] \leftarrow Prompts[i] \cup Append(NegativeAdjectives[label], label)$
\EndFor

\State \textbf{Adaptive Exposure Generation}
\State $\tau_{\text{image-val}} = Avg(Dist(E_{I}^{CLIP}(x_i), E_{T}^{CLIP}(y))) \hfill \forall (x_i,y_i) \in \mathcal{D}_{\text{val}} \forall y \neq y_i$
\For{$(x_i, y_i) \in \mathcal{D}_{\text{in}}$}
    \State $c \sim \mathcal{U}(Prompts[y_i])$
    \State $t_{init} \sim \mathcal{U}([T_0, \ldots, T])$
    \State $\hat{x}_{t_{init}} = x_i$
    \For{$t = t_{init}, \ldots, 0$}
        \State $\hat{\mu}(\hat{x}_t | c) = \mu_{Diffusion}(\hat{x}_t|c) + s \cdot \Sigma_{Diffusion}(\hat{x}_t|c) \cdot \nabla_{\hat{x}_t} (E_{I}^{CLIP}(\hat{x}_t) \cdot E_{T}^{CLIP}(c))$
        \State $\hat{x}_{t-1} \sim \mathcal{N}(\hat{\mu}(\hat{x}_t | c), \Sigma_{Diffusion}(\hat{x}_t|c))$
    \EndFor
    \If{$Dist(E_{I}^{CLIP}(\hat{x}_0), E_{T}^{CLIP}(y_i)) < \tau_{\text{image-val}}$}
        \State $\mathcal{D}_{\text{exposure}} \leftarrow \mathcal{D}_{\text{exposure}} \cup \{(\hat{x}_0, K+1)\}$
    \EndIf
\EndFor

\State $\mathcal{D}_{\text{train}} \leftarrow \mathcal{D}_{\text{in}} \cup \mathcal{D}_{\text{exposure}}$
\State $\hat{f}_{\theta} \leftarrow \text{Adversarial-Training}(f_{\theta}, \mathcal{D}_{\text{train}})$
\end{algorithmic}
\end{algorithm}

\begin{figure}[h]
  \begin{center}
    \includegraphics[width=1\linewidth]{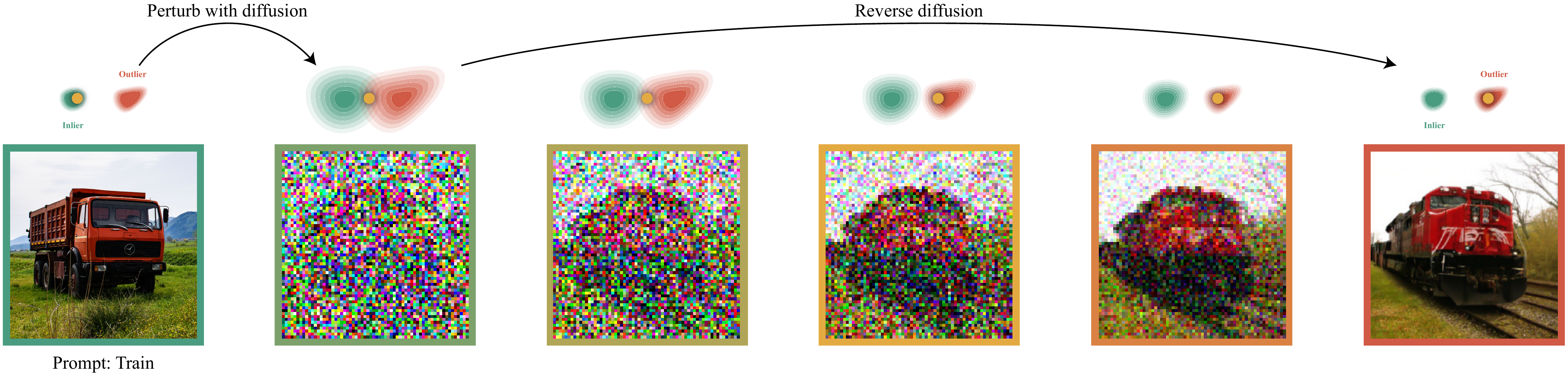}
    \caption{The figure illustrates a text-guided diffusion process. A yellow dot, representing an inlier data point within the green inlier distribution, is progressively transformed towards the red outlier distribution, driven by CLIP guidance. This showcases the model's ability to guide the transformation from inlier to outlier data via textual instructions.}
    \label{fig:Diff_Plot}
  \end{center}
\end{figure}

\section{ND, OSR and outlier Detection} \ 
  
As we have reported the results of our method on the most common settings for outlier detection, in this section, we provide a brief explanation for each setting to provide further clarification.
  In OSR, a model is trained on $K$ classes from an $N$-class training dataset. During testing, the model encounters $N$ distinct classes, where $N-K$ of these classes were not present during the training phase. ND is a type of open-set recognition that is considered an extreme case, specifically when $k$ is equal to 1.
  Some works refer to ND as one-class classification. outlier detection shares similarities with OSR; however, the key distinction is that the open-set and closed-set classes originate from two separate datasets. \cite{yang2021generalized,mirzaei2022fake}
  
\section{Detailed Baselines} \label{detailed_base}

Some works introduced the OE technique for outlier detection tasks, which utilizes auxiliary random images known to be anomalous \cite{hendrycks2018deep}. Many top-performing outlier detection methods incorporate OE to enhance their performance in both classic and adversarial outlier detection evaluation tasks \cite{kong2021opengan,liznerski2022exposing,mirzaei2022fake}. The most direct approach to utilizing outliers involves incorporating them into the training set, with labels uniformly selected from the label space of typical samples. In an effort to improve the adversarial robustness of outlier detection, some methods have attempted to make OE more adaptive. For example, ATD \cite{azizmalayeri2022your} employs a generator to craft fake features instead of images. Another approach, ALOE \cite{chen2021atom}, mines low anomaly score data from an auxiliary  outlier dataset for training, thereby enhancing the robustness of outlier detection.

 Furthermore, some other works have pursued outlier detection in an adversarial setting which includes APAE \cite{goodge2021robustness}, PrincipaLS \cite{lo2022adversarially} and OCSDF \cite{bethune2023robust} and OSAD \cite{shao2022open} ATOM\cite{chen2021atom} ALOE \cite{chen2020robust} and ATD \cite{azizmalayeri2022your}, between these robust outlier detection methods, ATOM, ALOE and ATD achieved relatively better results by incorporating OE and adversarial training, however, their performance falls short in case that inlier set distribution is far from their fixed OE set. For instance, APAE \cite{goodge2021robustness} suggested enhancing adversarial robustness through the utilization of approximate projection and feature weighting. PrincipaLS \cite{lo2022adversarially} proposed Principal Latent Space as a defense strategy to perform adversarially robust ND. OCSDF \cite{bethune2023robust} aimed to achieve robustness in One-Class Classification (OCC) by learning a signed distance function to the boundary of the support of the inlier distribution, which can be interpreted as the inlierity score. Through making the distance function $\ell_1$ Lipschitz, one could guarantee robustness against $\ell_2$ bounded perturbations. OSAD \cite{shao2022open} augmented the model architecture with dual-attentive denoising layers, and integrated the adversarial training loss with an auto-encoder loss. The auto-encoder loss was designed to reconstruct the original image from its adversarial counterpart. 

 In the context of adversarial outlier scenarios, certain studies focused on utilizing the insights gained from the pre-trained models based on Vision Transformers (ViT) \cite{NEURIPS2021_3941c435,dosovitskiy2020image}. Some other works incorporated OE to enhance their performance in both clean and adversarial outlier detection evaluation tasks \cite{chen2021atom, kong2021opengan,liznerski2022exposing,mirzaei2022fake}. The most direct approach to utilizing outliers involved incorporating them into the training set, with labels uniformly selected from the label space of typical samples. In an effort to improve the adversarial robustness of  the detection models, some methods have attempted to make OE more adaptive. For example, ATD \cite{azizmalayeri2022your} employed a generator to craft fake features instead of images, and applied adversarial training on OE and inlier real samples to make the discriminator robust. Another approach, ATOM \cite{chen2021atom}, mined low anomaly score data from an auxiliary outlier dataset for training, thereby enhancing the robustness of outlier detection through adversarial training on the mined samples \cite{mirzaei2022fake,mirzaei2024universal,salehi2021unified,mirzaei2024adversarially,mirzaeiscanning,moakhar2023seeking,mirzaei2024killing,jafari2024power,taghavi2023change,rahimi-etal-2024-hallusafe,taghavi2023imaginations,taghavi-etal-2023-ebhaam,taghavi2024backdooring,ebrahimi2024sharifa,ebrahimi2024sharif,rahimi-etal-2024-nimz}.

  \section{Details About Evaluation and Generation}
  
\subsection{OSR Experiments Details} \label{osr_app}
In order to evaluate earlier works in OSR setting, we first select desired number of classes, say $K$ and rename the labels of samples to be in the range $0$ to $K-1$. Then following the guideline of the method, we evaluate it in both clean and adversarial settings and repeat each experiment 5 times and report the average.

\subsection{OOD Experiments Details} \label{ood_app}
Table \hyperref[Table2:a:OOD]{2} yielded results that are now presented in Table \hyperref[OOD_details]{7} for a more comprehensive overview. We designated multiple datasets as out-of-distribution during the testing phase and reported the outcomes in Table \hyperref[OOD_details]{7}. Adversarial and clean out-of-distribution scenarios have also been examined by other approaches. Prominent methods in the clean setting encompass the ViT architecture and OpenGAN. Regarding image classification, AT and HAT have been recognized as highly effective defenses. AOE, ALOE, and OSAD are regarded as state-of-the-art methods for out-of-distribution detection, and ATD in robust outlier detection. These outlier methods (excluding OpenGAN and ATD) have undergone evaluation with various detection techniques, including MSP \cite{hendrycks2017a}\cite{liang2018enhancing}, MD \cite{MD}, Relative MD \cite{ren2021simple}, and OpenMax \cite{openmax}. The results reported for each outlier method correspond to the best-performing detection method. Notably, our approach has surpassed the state-of-the-art in robust out-of-distribution setting (ATD) for nearly all datasets.
\begin{equation}
\mu_k=\frac{1}{N} \sum_{i: y_i=k} z_i, \quad \Sigma=\frac{1}{N} \sum_{k=1}^K \sum_{i: y_i=k}\left(z_i-\mu_k\right)\left(z_i-\mu_k\right)^T, \quad k=1,2, \ldots, K
\end{equation}
In addition, to use RMD, one has to fit a $\mathcal{N}\left(\mu_0, \Sigma_0\right)$ to the whole in-distribution. Next, the distances and anomaly score for the input $x^{\prime}$ with pre-logits $z^{\prime}$ are computed as:
\begin{equation}
\begin{gathered}
M D_k\left(z^{\prime}\right)=\left(z^{\prime}-\mu_k\right)^T \Sigma^{-1}\left(z^{\prime}-\mu_k\right), \quad R M D_k\left(z^{\prime}\right)=M D_k\left(z^{\prime}\right)-M D_0\left(z^{\prime}\right), \\
\operatorname{score}_{M D}\left(x^{\prime}\right)=-\min _k\left\{M D_k\left(z^{\prime}\right)\right\}, \quad \text { score }_{R M D}\left(x^{\prime}\right)=-\min _k\left\{R M D_k\left(z^{\prime}\right)\right\} .
\end{gathered}
\end{equation}

In this section, we will provide more details about our evaluation methodology and Generation Step.
\subsection {Generation Step} \label{gen_step_app}
 
Denoising Diffusion Probabilistic Models (DDPMs) \cite{sohl2015deep,ho2020denoising} are trained to reverse a parameterized Markovian process that transforms an image to pure noise gradually over time. Beginning with isotropic Gaussian noise samples, they iteratively denoise the image and finally convert it into an image from the training distribution. In particular a network employed and trained as follows: $p_\theta(x_{t-1} | x_t) = \mathcal{N}(  \mu_\theta(x_t, t), \Sigma_\theta(x_t, t))$. This network takes the noisy image $x_t$ and the embedding at time step $t$ as input and learns to predict the mean $\mu_\theta(x_t, t)$ and the covariance $\Sigma_\theta(x_t, t)$. Recent studies have shown that DDPMs can be utilized for tasks such as generating high-quality images, as well as for editing and inpainting \cite{dhariwal2021diffusion, avrahami2022blended, croitoru2023diffusion}.
 
In our proposed generation method, we perturb the in-distribution(ID) images with Gaussian noise and utilize a diffusion model with guidance from the extracted candidate near-outlier labels to shift the ID data to outlier data. This is possible because it has been shown that the reverse process can be solved not only from $t_0=1$ but also from any intermediate time (0, 1). We randomly choose an initial step for each data between 0 and 0.6, which is a common choice based on previous related works. \cite{kim2022diffusionclip,meng2021sdedit}

If we have k classes in the inlier dataset, with each class containing N samples, we generate N outlier samples to extend the dataset to k+1 classes. However, if N is a small number (e.g. N<100), we may generate up to 3000 outlier samples to prevent overfitting.

\subsection {Adversarial Attack on outlier Detectors} \label{Adversarial_Attack_detail}

\textbf{Adversarial attacks.} \ For the input $x$ with an associated ground-truth label $y$, an adversarial example $x^*$ is generated by adding a small noise to $x$, maximizing the predictor model loss $\ell(x^*; y)$. Projected Gradient Descent (PGD) \cite{madry2017towards} method is regarded as a standard and effective attack technique that functions by iteratively maximizing the loss function, through updating the perturbed input by a step size $\alpha$ in the direction of the gradient sign of $\ell(x^*; y)$ with respect to $x$:$\quad   x_0^*=x, \quad x_{t+1}^*=x_t^*+\alpha . \operatorname{sign}\left(\nabla_x \ell\left(x_t^*, y\right)\right),$ where the noise is  projected onto the $\ell_{\infty}$-ball  with a radius of $\epsilon$ during each step.  To adapt adversarial attacks for the outlier detection task, we target the final output, which is the outlier score of each test sample. This study utilized attacks with objectives that include both inlier and outlier samples.

\textbf{Attack to outlier detectors.} 
\ Outlier detection can be formulated as:
$$g_\lambda(\mathrm{x})= \begin{cases}\mathrm{ID} & \text { if } O(\mathrm{x}) \leq \lambda \\ \mathrm{OOD} & \text { if } O(\mathrm{x})>\lambda\end{cases},$$
where $O(x)$ is the detection score, and $\lambda$ is the threshold. In PGD attacks to outlier detection methods, instead of maximizing the loss value, we try to increase $O(x)$ if $x$ belongs to in-distribution samples and decrease it otherwise. The formulation of the attack would be:
$x_0^*=x, x^{t+1}=x^t+\alpha \cdot \operatorname{sgn}\left(\nabla_x\left(y \cdot O_\theta\left(x^t\right)\right)\right)$ Where $y=1$ for in-distribution samples and $y=-1$ for outlier samples. The same setting holds for other attacks in our study including AA, A3, and Blackbox.

  We performed various strong attacks including PGD-1000 with 10 random restarts, Auto Attack (AA) \cite{croce2020reliable}, and Adaptive Auto Attack (A3) \cite{liu2022practical}. The latter is a recently introduced attack that has demonstrated considerable strength. It adapts the attack according to the test dataset and the model to better use the adversarial budget, i.e. the number of iteration steps. It also uses a wiser method for the initialization of the attack to save the adversarial budget for perturbing more samples from the test dataset. The detailed experiments on these attacks are reported in Tables \ref{Detailed_one_class}, \ref{OOD_details}, and \ref{Table 1:Novelty Detection_2}. It is also noteworthy that for Auto Attack, it was not possible to adapt the DLR\cite{croce2020reliable} loss-based attacks due to their presumption that the output of the model has at least 3 elements, which does not hold in outlier detection tasks.
  \subsection{Computational Cost}
Experiments were conducted on RTX 3090 GPUs. Generating approximately 10,000 low-resolution and 1,000 high-resolution outlier data required 1 hour. For the one-class anomaly detection, training each class of low-resolution datasets took about 100 minutes (see Figure \ref{fig:computation_plot} for detailed analysis). The outlier detection task required around 16 hours of training, and each experiment in the OSR setting took approximately 9 hours.
 
\begin{figure}[h]
  \begin{center}
    \includegraphics[width=1\linewidth]{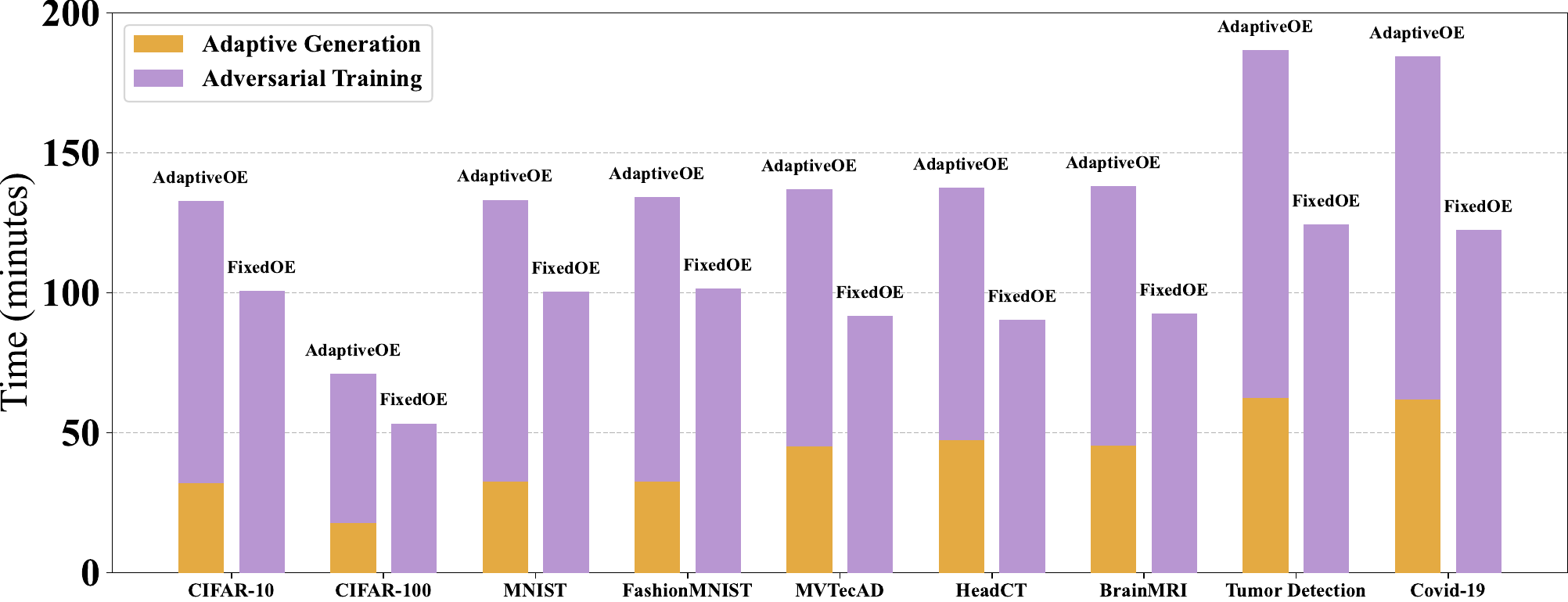}
    \caption{Comparative analysis of computational time for data generation and adversarial training across various datasets in one-class anomaly detection setting. The time is measured in minutes and is split into two components: data generation (golden segment) and the subsequent adversarial training phase (purple segment). The datasets range from standard image benchmarks like CIFAR-10 and MNIST to specialized medical and anomaly detection datasets such as MVTecAD, BrainMRI, and Covid-19.}
    \label{fig:computation_plot}
  \end{center}
\end{figure}

\section{Using Image Labels as Descriptors}

Novelty detection, also known as one-class classification, involves identifying instances that do not fit the pattern of a given class in a training set. Traditionally, methods for this task have been proposed without using the labels of the training data. For example, they did not take into account the fact that the inlier set includes the semantic "\texttt{dog}". In the case of outlier detection (which is a multi-class setting), methods commonly extract features and define supervised algorithms using the labels of the inlier set. However, they do not fully utilize the semantic information contained in these labels. Specifically, they only consider class labels as indexes for their defined task, such as classification.

Recently, there has been a growing interest in leveraging pre-trained multimodal models to enhance outlier detection performance, both in one-class and multi-class scenarios. Unlike prior works, these approaches utilize the semantic information embedded within the labels of the inlier set. This is akin to treating labels as image descriptors rather than just as indices. For example, \cite{liznerski2022exposing} used CLIP in the novelty detection setting and utilized both pairs of inlier images and their labels (e.g., a photo of {x}) to extract maximum information from the inlier set. Similarly, \cite{esmaeilpour2022zero} applied CLIP for zero-shot outlier detection and used both the image and semantic content of their respective labels to achieve the same goal. Motivated by these works, our study utilizes image labels as descriptors in all reported settings (ND, OSR, OOD). In fact, we utilized a simple language model to predict candidate unseen labels for outlier classes located near the boundary, leveraging these image labels. \paragraph{Discussion} Although some recent works have used labels as descriptors, there may be concerns that this approach could provide unfair guidance since it is not commonly used in traditional literature. However, it is important to note that the outlier detection problem is a line of research with many practical applications in industries such as medicine autonomous driving cars and industry. In such cases, knowing the training data labels and semantics, such as \textit{"healthy CT scan images"}, is possible and we do not need more details about inlier data classes except for their names.

 Moreover, previous adversarially robust outlier detector models have reported almost no improvement over random results in real-world datasets, especially in the case of ND settings. Therefore, our use of the inlier class label as an alternative solution is reasonable. Our approach outperforms previous models by up to 50\% in the robust ND scenario and this superiority continues in multi-class modes where data labels are available and we only use the class names to improve the model. Given the applicability of the task addressed in this article and the progress of multi-domain models, our approach has potential for practical use

  \section{Leveraging Pre-trained Models for outlier Detection}
It has been demonstrated that leveraging pre-trained models can significantly improve the performance and practical applicability of downstream tasks \cite{wei2021pretrained,wang2022bridging}, including outlier detection, which has been extensively studied. 

Various works \cite{bergman2020deep,reiss2021panda,reiss2021mean,cohen2021transformaly,roth2021total,Salehi_2021_CVPR} have utilized pre-trained models' features or transfer learning techniques to improve detection results and efficiency, particularly in outlier detection under harder constraints. For example, \cite{esmaeilpour2022zero} used a pre-trained CLIP model trained on 400 million data for Zero-Shot outlier Detection, \cite{fort2021exploring} proposed using a pre-trained ViT \cite{dosovitskiy2020image} model trained on 22 million data for near-distribution outlier detection, and \cite{xu2021unsupervised} utilized a pre-trained BERT \cite{devlin2018bert} model trained on billions of data for outlier detection in the text domain. In our work, we addressed the highly challenging task of developing an adversarially robust outlier detector model, which is unexplored for real-world datasets such as medical datasets. To accomplish this, we utilized the CLIP and diffusion model as our generator backbone, which was trained on 67 million data.

\section{Why Our Diffusion Model Is the Best Fit for Near-outlier Generation} \label{best_fit}

\paragraph{Working in Pixel Space}
In Section \ref{two_type_ood}, we discussed how outlier data can be divided into two categories (i.e. pixel- and semantic-level). Our need for diverse outlier data motivates our preference for generative models that can create both pixel-level and semantic-level outlier data. Our generative model is a suitable choice as it uses a diffusion model applied at the pixel-level to generate images from texts. This allows the model to generate outlier samples that differ in their local appearance, which is particularly important for pixel-level outlier detection. 
Compared to other SOTA text-to-image models that mostly work at the embedding level, our generative model's ability to generate images at the pixel-level makes it a better choice for our purposes.
\paragraph{Comparing with DreamBooth}
As our pipeline's generator model involves image editing, we explored the literature on image manipulating and tested a common methods used for image editing. Numerous algorithms have been proposed for generating new images conditioned on input images  among these, we have chosen DreamBooth as one of the SOTA algorithms for specifying image details in text-to-image models.  we evaluated the DreamBooth \cite{ruiz2023dreambooth} algorithm for changing image details in various datasets. Our experiment showed that, despite DreamBooth's good performance for natural images and human faces, the algorithm had poor results for datasets with different distributions, such as MNIST and FashionMNIST. One possible explanation for the poor performance of these algorithms is their bias towards the distribution of the training datasets, such as LAION, which typically consists of natural images and portraits. Consequently, these algorithms may not yield satisfactory results for datasets with different distributions.
\begin{figure}[thb]
  \begin{center}
    \includegraphics[width=0.5\linewidth]{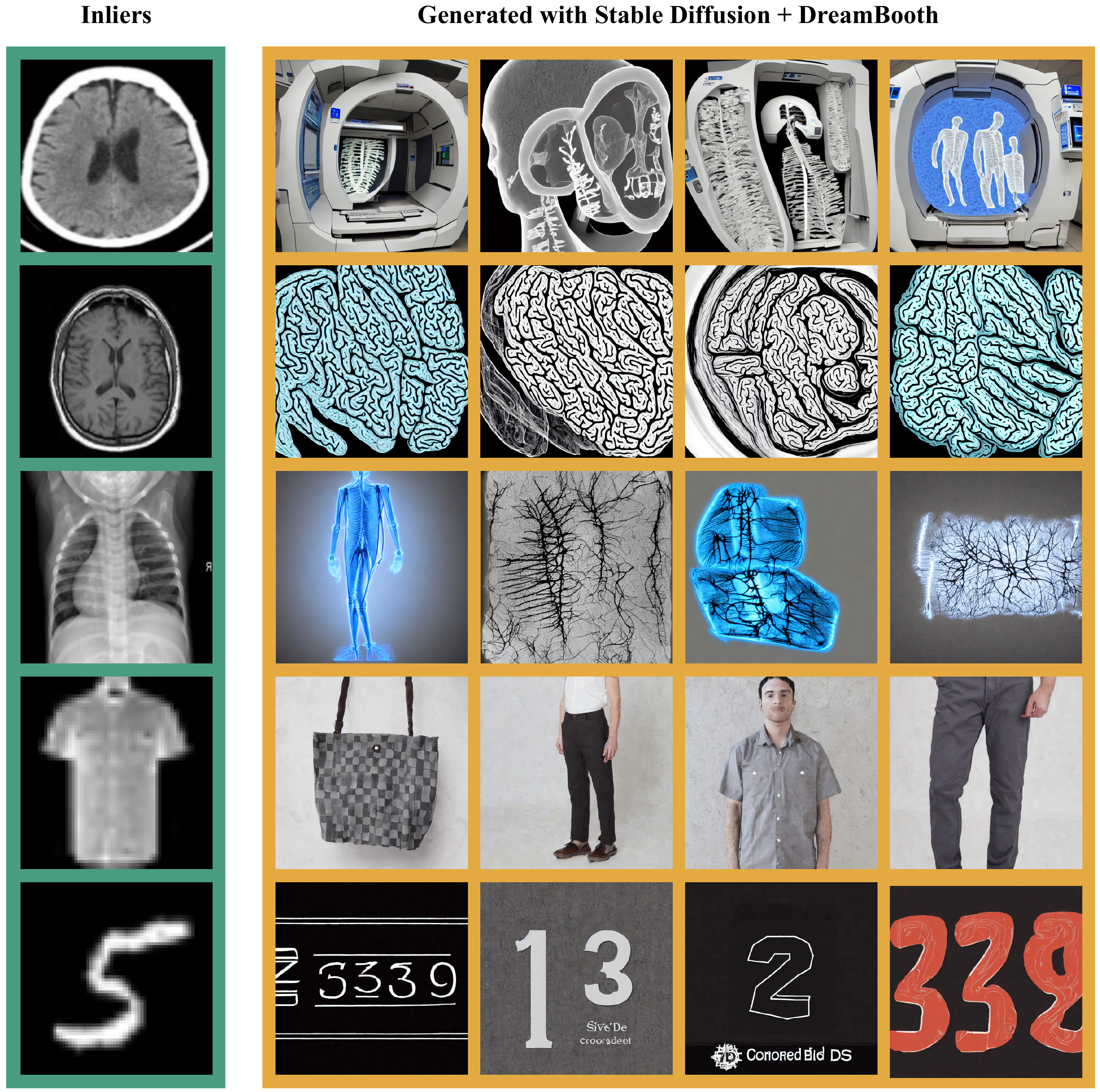}
    \caption{  generated images using the DreamBooth algorithm and StableDiffusion model, which shows a very large shift between ID- and generated OOD data. This demonstrates the superiority of our pipeline as a near OOD generator.}
    \label{fig:Stable_Diffusion_Plot}
  \end{center}
\end{figure}

\section{Detailed Analysis and Insights of Ablation Study}

\paragraph{FID, Density and Coverage}\label{DENSITY}
Fréchet Inception Distance (FID) metric \cite{DBLP:journals/corr/HeuselRUNKH17} measures the distance between feature vectors of real and generated images and calculates the distance them, has been shown to match with human judgments.
{\color{black} The diversity of generative models can be evaluated using two metrics: Density and Coverage \cite{naeem2020reliable}. By utilizing a manifold estimation procedure code, the distance between two sets of images can be measured. To calculate these metrics, features from a pre-trained model are utilized, specifically those before the final classification layer. The mathematical expression for these metrics is as follows}:
\begin{equation} \label{eq_}
    \centering
    \begin{split}
Density\left(\boldsymbol{X}_s, \boldsymbol{X}_t, F, k\right)=\frac{1}{k M} \sum_{j=1}^M \sum_{i=1}^N \mathbb{I}\left(\boldsymbol{f}_{t, j} \in B\left(\boldsymbol{f}_{s, i}, \mathrm{NN}_k\left(F\left(\boldsymbol{X}_s\right), \boldsymbol{f}_{s, i}, k\right)\right)\right) \text {, }
    \end{split}
\end{equation}
\begin{equation} \label{eq_}
    \centering
    \begin{split}
Coverage\left(\boldsymbol{X}_s, \boldsymbol{X}_t, F, k\right)=\frac{1}{N} \sum_{i=1}^N \mathbb{I}\left(\exists j \text { s.t. } \boldsymbol{f}_{t, j} \in B\left(\boldsymbol{f}_{s, i}, \mathrm{NN}_k\left(F\left(\boldsymbol{X}_s\right), \boldsymbol{f}_{s, i}, k\right)\right)\right).
    \end{split}
\end{equation}
 
where $\boldsymbol{F}$ is a feature extractor, $\boldsymbol{f}$ is a collection of features from ${F}$, 

 $\boldsymbol{X}_s=\left\{\boldsymbol{x}_{s, 1}, \ldots, \boldsymbol{x}_{s, N}\right\}$ denotes real images, $\boldsymbol{X}_t=\left\{\boldsymbol{x}_{t, 1}, \ldots, \boldsymbol{x}_{t, M}\right\}$ denotes  generated images, $B(\boldsymbol{f}, r)$ is the $\mathrm{n}$-dimensional sphere in which $\boldsymbol{f}=F(\boldsymbol{x}) \in \mathbb{R}^n$ is the center and $r$ is the radius, $\mathrm{NN}_kf({F}, \boldsymbol{f}, k)$ is the distance from $\boldsymbol{f}$ to the $k$-th nearest embedding in ${F}$, and $\mathbb{I}(\cdot)$ is a indicator function. We use the standard InceptionV3 features, which are also used to compute the FID. The measures are computed using the official code \cite{naeem2020reliable}.\\
 The definition of the FDC metric introduced in the paper is as below:
\begin{equation}
FDC=\log(1 + \frac{Density \times Coverage}{FID} \times 10^4)
\end{equation}

\section{The Significance of Conditioning on Both Images and Text from the inlier Distribution}

In order to have an accurate outlier detector, it's important to generate diverse and realistic samples that are close to the distribution of the inlier data. In our study, we tackle this challenge by leveraging the information contained in the inlier data. Specifically, we extract the labels of classes that are close to the inlier set and use them as guidance for generation. Additionally, we initialize the reverse process generation of a diffusion model with inlier images, so the generation of outlier data in our pipeline is conditioned on both the images and the text of the inlier set. This enables us to generate adaptive outlier samples.

In the Ablation Study (sec. \ref{Ablation_Section}), we demonstrate the importance of using both image and text information for generating outlier data. We compare our approach with two other methods that condition on only one type of information and ignore the other. The first technique generates fake images based on the inlier set, while the other generates outlier data using only the extracted text from inlier labels. The results show that both techniques are less effective than our adaptive exposure technique, which conditions the generation process on both text and image. This confirms that using both sources of information is mandatory and highly beneficial. 

\subsection{Samples Generated Solely Based on Text Conditioning} 
  In this section, we compare inlier images with images generated by our pipline using only text  in Fig. \ref{fig:Samples_Plot1}  (without conditioning on the images). Our results, illustrated by the plotted samples, demonstrate that there is a significant difference in distribution between these generated images and inlier images. This difference is likely the reason for the ineffectiveness of the outlier samples generated with this technique.
\section{Label Generation}\label{prompt_gen}
  \subsection{Pixel-Level and Semantic-Level outlier Detection}\label{two_type_ood}
OOD samples can be categorized into two types: pixel-level and semantic-level. In pixel-level outlier detection, ID and outlier samples differ in their local appearance, while remaining semantically identical. For instance, a broken glass could be considered an outlier sample compared to an intact glass due to its different local appearance. In contrast, semantic-level outlier samples differ at the semantic level, meaning that they have different meanings or concepts than the ID samples. For example, a cat is an outlier sample when we consider dog semantics as ID because they represent different concepts.

\subsection{Our Method of Generating labels} A reliable and generalized approach for outlier detection must have the capability to detect both semantic-level and pixel-wise outliers, as discussed in the previous section. To this end, our proposed method constructs n-outlier labels by combining two sets of words: near-distribution labels and negative adjectives derived from a inlier label name. We hypothesize that the former set can detect semantic-level outliers, while the latter set is effective in detecting pixel-wise outliers. Additionally, we include an extra label, marked as 'others', in the labels list to increase the diversity of exposures.

To generate negative adjectives, we employ a set of constant texts that are listed below and used across all experimental settings (\textbf{X} is the inlier label name):
    \begin{itemize}
        \item \textit{A photo of \textbf{X} with a crack }
        \item \textit{A photo of a broken \textbf{X}}
        \item \textit{A photo of \textbf{X} with a defect }
        \item \textit{A photo of \textbf{X} with damage}
        \item \textit{A photo of \textbf{X} with a scratch }
        \item \textit{A photo of \textbf{X} with a hole }
        \item \textit{A photo of \textbf{X} torn}
        \item \textit{A photo of \textbf{X} cut}
        \item \textit{A photo of \textbf{X} with contamination}
        \item \textit{A photo of \textbf{X} with a fracture}
        \item \textit{A photo of a damaged \textbf{X}}
        \item \textit{A photo of a fractured \textbf{X}}
        \item \textit{A photo of \textbf{X} with destruction}
        \item \textit{A photo of \textbf{X} with a mark}
    \end{itemize}

For  n-outlier labels, we utilize Word2Vec to search for semantically meaningful word embeddings after inlierizing the words through a process of lemmatization. First, we obtain the embedding of the inlier class label and then search among the corpus to identify the 1000 nearest neighbors of the inlier class label. \\
In the subsequent phase, we employ the combination of Imagenet labels and CLIP to effectively identify and eliminate labels that demonstrate semantic equivalence to the inlier label. Initially, we leverage CLIP to derive meaningful representations of the Imagenet labels. Then, we calculate the norm of the pairwise differences among these obtained representations. By computing the average of these values, a threshold is established, serving as a determinant of the degree of semantic similarity between candidate labels and the inlier label. The threshold is defined as:
\begin{equation}
\tau_{text} =   \frac{ \sum_{  i=1,i\neq j}^M \sum_{j=1}^{M} |E_{T}(y_i) - E_{T}(y_j)|}{M(M-1)} \end{equation} 
In which, $M$ is the number of Imagenet labels, and $y_i$s are the Imagenet labels.\\
Consequently, we filter out labels whose CLIP output exhibits a discrepancy from the inlier class(es) that falls below the threshold. 

We then sample n-outlier labels from the obtained words based on the similarity factor of the neighbors to the inlier class label. The selection probability of the n-outlier labels is proportional to their similarity to the inlier class label. Finally, we compile a list of n-outlier labels to serve as near outlier labels.

\begin{table*}[tbh]
\caption{The detailed AUROC scores of the class-specific experiments for (One-Class) Novelty Detection setting in CIFAR10, CIFAR100, MNIST, Fashion-MNIST datasets.}
\label{Detailed_one_class}
\begin{subtable}{1\textwidth}
\BlankLine
\subcaption{CIFAR10}
\label{}
\setlength{\tabcolsep}{12pt}
\resizebox{\linewidth}{!}{\begin{tabular}{llr*{13}{c}} 
\hline\noalign{\smallskip}

  
   \multicolumn{1}{c}{Method} & \multicolumn{1}{c}{Attack} & &\multicolumn{9}{c}{Class } &\multicolumn{1}{c}{ Average }  
     \\      \cmidrule(lr){1-1} \cmidrule(lr){2-2} 
    \cmidrule(lr){3-12}  
    \cmidrule(lr){13-13}   
   
  & &\multirow{2}{*}{0}&\multirow{2}{*}{1}&\multirow{2}{*}{2}&\multirow{2}{*}{3}&\multirow{2}{*}{4} &\multirow{2}{*}{5}&\multirow{2}{*}{6} & \multirow{2}{*}{7} & \multirow{2}{*}{8} &\multirow{2}{*}{9}  \\
  & & & & & &  & &  &   &   & &  \\
  
\noalign{\smallskip}
 \hline

\noalign{\smallskip}
\multirow{4}{*}{Ours} &Clean& 91.7 & 97.3 & 77.4 & 74.0 & 82.6 & 81.2 & 91.5 & 92.8 & 94.0 & 91.7 & 87.4
\\ & BlackBox & 89.9 & 95.8 & 75.5 & 72.1 & 81.6 & 79.1 & 89.1 & 91.2 & 92.4 & 89.6 & 85.6
 
\\ & PGD-1000& 76.6 & 81.4 & 59.1 & 55.2 & 65.0 & 65.9 & 73.6 & 69.6 & 78.9 & 77.3 & 70.2

 \\ & AutoAttack& 75.7 & 80.0 & 58.7 & 53.6 & 64.2 & 65.1 & 72.1 & 68.8 & 78.9 & 76.0 & 69.3
 
  \\

\noalign{\smallskip}

\hline
\end{tabular}}
\end{subtable}

 \bigskip

\begin{subtable}{1\textwidth}
\BlankLine
\subcaption{ CIFAR100 }
\label{}
\resizebox{\linewidth}{!}{\begin{tabular}{llr*{23}{c}} 
\hline\noalign{\smallskip}

  
   \multicolumn{1}{c}{Method} & \multicolumn{1}{c}{Attack} & &\multicolumn{19}{c}{Class } &\multicolumn{1}{c}{ Average }  
     \\      \cmidrule(lr){1-1} \cmidrule(lr){2-2} 
    \cmidrule(lr){3-22}  
    \cmidrule(lr){23-23}   

  & &\multirow{2}{*}{0}&\multirow{2}{*}{1}&\multirow{2}{*}{2}&\multirow{2}{*}{3}&\multirow{2}{*}{4} &\multirow{2}{*}{5}&\multirow{2}{*}{6} & \multirow{2}{*}{7} & \multirow{2}{*}{8} &\multirow{2}{*}{9} &\multirow{2}{*}{10}&\multirow{2}{*}{11}&\multirow{2}{*}{12}&\multirow{2}{*}{13}&\multirow{2}{*}{14} &\multirow{2}{*}{15}&\multirow{2}{*}{16} & \multirow{2}{*}{17} & \multirow{2}{*}{18} &\multirow{2}{*}{19}  \\

  & & & & & & & & & & & & & & & & & & & & &   \\
  
\noalign{\smallskip}
 \hline

\noalign{\smallskip}
\multirow{4}{*}{Ours} &Clean & 79.0 & 78.6 & 95.5 & 78.1 & 89.2 & 69.4 & 76.2 & 82.6 & 77.9 & 87.4 & 92.7 & 73.1 & 77.6 & 65.0 & 83.9 & 65.5 & 72.1 & 93.6 & 83.5 & 72.0 & 79.6
\\ & BlackBox & 76.9 & 77.0 & 93.2 & 75.4 & 87.6 & 67.3 & 74.0 & 79.6 & 75.7 & 85.1 & 89.9 & 71.3 & 74.7 & 62.9 & 81.0 & 63.4 & 69.7 & 91.8 & 81.8 & 70.1 & 77.4
 
\\ & PGD-1000 & 59.7 & 60.2 & 83.0 & 61.3 & 72.9 & 53.0 & 60.4 & 62.6 & 59.0 & 76.5 & 78.1 & 51.2 & 61.0 & 47.4 & 60.3 & 44.5 & 51.9 & 78.7 & 61.2 & 58.1 & 62.1
 \\ & AutoAttack & 56.8 & 59.4 & 80.4 & 60.9 & 73.9 & 51.2 & 61.3 & 61.7 & 58.3 & 75.9 & 75.7 & 51.7 & 59.9 & 45.7 & 59.2 & 44.4 & 50.6 & 76.1 & 60.5 & 57.1 & 61.0
 
  \\

\noalign{\smallskip}

\hline
\end{tabular}}
\end{subtable}

 \bigskip

\begin{subtable}{1\textwidth}
\BlankLine
\subcaption{ MNIST }
\label{}
\setlength{\tabcolsep}{10pt}
\resizebox{\linewidth}{!}{\begin{tabular}{llr*{13}{c}} 
\hline\noalign{\smallskip}

  
   \multicolumn{1}{c}{Method} & \multicolumn{1}{c}{Attack} & &\multicolumn{9}{c}{Class } &\multicolumn{1}{c}{ Average }  
     \\      \cmidrule(lr){1-1} \cmidrule(lr){2-2} 
    \cmidrule(lr){3-12}  
    \cmidrule(lr){13-13}   
   
  & &\multirow{2}{*}{0}&\multirow{2}{*}{1}&\multirow{2}{*}{2}&\multirow{2}{*}{3}&\multirow{2}{*}{4} &\multirow{2}{*}{5}&\multirow{2}{*}{6} & \multirow{2}{*}{7} & \multirow{2}{*}{8} &\multirow{2}{*}{9}  \\
  & & & & & &  & &  &   &   & &  \\
  
\noalign{\smallskip}
 \hline

\noalign{\smallskip}
\multirow{4}{*}{Ours} &Clean & 99.8 & 99.4 & 99.3 & 99.2 & 99.6 & 99.4 & 99.8 & 98.9 & 99.4 & 98.8 & 99.4
\\ & BlackBox & 98.7 & 99.0 & 98.2 & 98.8 & 98.3 & 98.9 & 99.4 & 97.8 & 98.5 & 98.2 & 98.6
 
\\ & PGD-1000 & 96.3 & 96.1 & 95.5 & 92.0 & 97.4 & 95.1 & 96.4 & 92.5 & 94.0 & 91.2 & 94.6
 \\ & AutoAttack& 96.9 & 96.1 & 96.3 & 92.0 & 96.7 & 96.3 & 98.1 & 92.5 & 95.2 & 92.0 & 95.2
 
  \\

\noalign{\smallskip}

\hline
\end{tabular}}
\end{subtable}

 \bigskip

\begin{subtable}{1\textwidth}
\BlankLine
\subcaption{ Fashion-MNIST }
\label{}
\setlength{\tabcolsep}{10pt}
\resizebox{\linewidth}{!}{\begin{tabular}{llr*{13}{c}} 
\hline\noalign{\smallskip}

  
   \multicolumn{1}{c}{Method} & \multicolumn{1}{c}{Attack} & &\multicolumn{9}{c}{Class } &\multicolumn{1}{c}{ Average }  
     \\      \cmidrule(lr){1-1} \cmidrule(lr){2-2} 
    \cmidrule(lr){3-12}  
    \cmidrule(lr){13-13}   
   
  & &\multirow{2}{*}{0}&\multirow{2}{*}{1}&\multirow{2}{*}{2}&\multirow{2}{*}{3}&\multirow{2}{*}{4} &\multirow{2}{*}{5}&\multirow{2}{*}{6} & \multirow{2}{*}{7} & \multirow{2}{*}{8} &\multirow{2}{*}{9}  \\
  & & & & & &  & &  &   &   & &  \\
  
\noalign{\smallskip}
 \hline

\noalign{\smallskip}
\multirow{4}{*}{Ours} &Clean & 95.8 & 99.7 & 93.9 & 93.4 & 92.9 & 98.3 & 86.5 & 98.6 & 98.5 & 98.8 & 95.6
\\ & BlackBox & 94.4 & 98.6 & 92.7 & 92.6 & 91.2 & 96.9 & 85.1 & 97.1 & 97.5 & 97.1 & 94.3
 
\\ & PGD-1000 & 89.7 & 98.4 & 82.9 & 79.9 & 76.1 & 94.8 & 71.4 & 94.0 & 90.7 & 93.9 & 87.2
 \\ & AutoAttack & 89.7 & 98.1 & 83.0 & 80.9 & 76.5 & 95.1 & 72.6 & 94.2 & 92.4 & 94.1 & 87.6
 
  \\

\noalign{\smallskip}

\hline
\end{tabular}}
\end{subtable}

\end{table*}

\section{Outlier Detection Performance Under Various Strong Attacks}
Tables \ref{Detailed_one_class}, \ref{OOD_details}, and \ref{Table 1:Novelty Detection_2} demonstrate the robust detection performance of RODEO when subjected to various strong attacks.

\begin{table*}[h]
\centering
\caption{OOD detailed results}
\label{OOD_details}
\begin{subtable}{1\textwidth}

\BlankLine
\centering
\subcaption{CIFAR10}
\label{OOD_CIFAR10}
\resizebox{\linewidth}{!}{\begin{tabular}{llr*{9}{c}} 
\hline\noalign{\smallskip}

   \multicolumn{1}{c}{Out-Dataset} & \multicolumn{1}{c}{Attack} &\multicolumn{8}{c}{Method}\\
    \cmidrule(lr){3-10}

  & &OpenGAN&ViT (RMD)& ATOM &AT (OpenMax)&OSAD (OpenMax)& ALOE (MSP)& ATD & RODEO    \\
\noalign{\smallskip}
 \hline
\noalign{\smallskip}
 
\multirow{2}{*}{MNIST}&\multirow{1}{*}{\centering Clean}&{99.4}&{98.7}&98.4&80.4&86.2&74.6&98.8&96.9 \\ &\multirow{1}{*}{\centering PGD-1000}&29.4 & 2.6 & 0.0 & 38.7 & 54.4 & 21.8 & \textbf{89.3} & {83.1} \\
    \cmidrule(lr){1-1} \cmidrule(lr){2-2} 
    \cmidrule(lr){3-10}

\multirow{2}{*}{TiImgNet}&\multirow{1}{*}{\centering Clean}&\textbf{95.3}&{95.2}&97.2&81.0&81.9& 82.1&88.0&85.1 \\ &\multirow{1}{*}{\centering PGD-1000}
& 14.3 & 1.4 & 3.4 & 15.6 & 18.4 & 20.7 & {46.1} & \textbf{46.3} \\
    \cmidrule(lr){1-1} \cmidrule(lr){2-2} 
    \cmidrule(lr){3-10}

\multirow{2}{*}{Places}&\multirow{1}{*}{\centering Clean}&95.0&\textbf{98.3}&98.7&82.5& 83.3& 85.1&92.5&{96.2} \\ &\multirow{1}{*}{\centering PGD-1000}&
16.4 & 2.2 & 5.6 & 18.0 & 20.3 & 21.9 & {59.8} & \textbf{70.2} \\
    \cmidrule(lr){1-1} \cmidrule(lr){2-2} 
    \cmidrule(lr){3-10}

\multirow{2}{*}{LSUN}&\multirow{1}{*}{\centering Clean}&96.5 &98.4&99.1&85.0 &   86.4 &  {98.7} &96.0& \textbf{99.0}   \\ &\multirow{1}{*}{\centering PGD-1000}&
23.1 & 1.1 & 1.0 & 18.7 & 19.8 & 50.7 & {68.1} & \textbf{85.1} \\
    \cmidrule(lr){1-1} \cmidrule(lr){2-2} 
    \cmidrule(lr){3-10}  

\multirow{2}{*}{iSUN}&\multirow{1}{*}{\centering Clean}&96.3&\textbf{98.6}&99.5&83.9&84.0&{98.3}&94.8&97.7 \\ &\multirow{1}{*}{\centering PGD-1000}
& 22.1 & 1.2 & 2.5 & 18.6 & 19.4 & 49.5 & {65.9} & \textbf{78.7} \\
    \cmidrule(lr){1-1} \cmidrule(lr){2-2} 
    \cmidrule(lr){3-10}  

\multirow{2}{*}{Birds}&\multirow{1}{*}{\centering Clean}&\textbf{98.3}&76.0&95.8&75.1&76.5&79.9&93.6&{97.8}\\ &\multirow{1}{*}{\centering PGD-1000}
& 33.6 & 0.0 & 5.2 & 13.8 & 18.2 & 20.9 & {68.1} & \textbf{76.0} \\
    \cmidrule(lr){1-1} \cmidrule(lr){2-2} 
    \cmidrule(lr){3-10}  

\multirow{2}{*}{Flower}&\multirow{1}{*}{\centering Clean}&98.3&{99.6}&99.8&85.5&88.6&79.0&\textbf{99.7}&99.5 \\ &\multirow{1}{*}{\centering PGD-1000}
& 29.2 & 1.7 & 19.0 & 20.0 & 25.7 & 18.7 & \textbf{92.8} & {88.7} \\
    \cmidrule(lr){1-1} \cmidrule(lr){2-2} 
    \cmidrule(lr){3-10}

\multirow{2}{*}{COIL}&\multirow{1}{*}{\centering Clean}&\textbf{98.1}&{95.9}&97.3&70.3&75.0&76.8&90.8&91.1  \\ &\multirow{1}{*}{\centering PGD-1000}
& 37.6 & 3.0 & 8.6 & 15.7 & 17.8 & 18.4 & {57.2} & \textbf{59.5} \\
    \cmidrule(lr){1-1} \cmidrule(lr){2-2} 
    \cmidrule(lr){3-10}  

\multirow{2}{*}{CIFAR100}&\multirow{1}{*}{\centering Clean}&{95.0}&\textbf{97.3}& 94.2&79.6&79.9&78.8&82.0&75.6\\&\multirow{1}{*}{\centering PGD-1000}&
9.2 & 0.8 & 1.6 & 15.1 & 17.2 & 16.1 & {37.1} & \textbf{37.8} \\
    \cmidrule(lr){1-1} \cmidrule(lr){2-2} 
    \cmidrule(lr){3-10}  
 \noalign{\smallskip}
 
\multirow{2}{*}{Avg.}&\multirow{1}{*}{\centering Clean}&\textbf{97.1}&{95.1}& 97.8 &80.5&82.7&84.3&94.3&93.2 \\ &\multirow{1}{*}{\centering PGD-1000}&
25.7 & 1.6 & 5.1 & 19.9 & 24.2 & 27.8 & {68.4} & \textbf{69.5} \\
 \hline

\end{tabular}}
\end{subtable}

\begin{subtable}{1\textwidth}

\BlankLine
\centering
\subcaption{CIFAR100}
\label{OOD_CIFAR100}
\resizebox{\linewidth}{!}{\begin{tabular}{llr*{9}{c}} 
\hline\noalign{\smallskip}

   \multicolumn{1}{c}{Out-Dataset} & \multicolumn{1}{c}{ Attack}  &\multicolumn{7}{c}{Method}\\
    \cmidrule(lr){3-10}

  & &OpenGAN&ViT (RMD)&ATOM&AT (RMD)&OSAD (MD)& ALOE(MD)& ATD & RODEO    \\
\noalign{\smallskip}
 \hline
\noalign{\smallskip}

\multirow{2}{*}{MNIST}&\multirow{1}{*}{\centering Clean}&{99.0}&83.8&90.4&41.1&95.9&96.6&97.3&\textbf{99.7}&   \\& \multirow{1}{*}{\centering PGD-1000}&12.9 & 0.0 & 0.0 & 12.5 & 80.3 & 71.4 & {84.6} & \textbf{96.0} \\
    \cmidrule(lr){1-1} \cmidrule(lr){2-2} 
    \cmidrule(lr){3-10}

\multirow{2}{*}{TiImgNet}&\multirow{1}{*}{\centering Clean}&{88.3}&\textbf{90.1}&85.1&72.3&48.3&58.1&73.7&72.9&  \\& \multirow{1}{*}{\centering PGD-1000}&2.2 & 1.4 & 0.1 & 10.3 & 8.2 & 4.6 & {24.3} & \textbf{37.3} \\
    \cmidrule(lr){1-1} \cmidrule(lr){2-2} 
    \cmidrule(lr){3-10}

\multirow{2}{*}{Places}&\multirow{1}{*}{\centering Clean}&\textbf{94.5}&92.3&94.8&73.1&55.7&75.0&83.3&{93.0}&  \\& \multirow{1}{*}{\centering PGD-1000}
&3.2 & 2.0 & 3.0 & 11.0 & 10.4 & 12.4 & {40.0} & \textbf{66.6} \\
    \cmidrule(lr){1-1} \cmidrule(lr){2-2} 
    \cmidrule(lr){3-10}

\multirow{2}{*}{LSUN}&\multirow{1}{*}{\centering Clean}&{97.1}&91.6&96.6&76.0&55.6&83.1&89.2&\textbf{98.1}&  \\& \multirow{1}{*}{\centering PGD-1000}
&5.6 & 0.0 & 1.5 & 11.2 & 8.7 & 19.0 & {47.7} & \textbf{83.1} \\
    \cmidrule(lr){1-1} \cmidrule(lr){2-2} 
    \cmidrule(lr){3-10}

\multirow{2}{*}{iSUN}&\multirow{1}{*}{\centering Clean}&\textbf{96.4}&91.4&96.4&72.5&54.8&80.1&86.5&{95.1}&  \\& \multirow{1}{*}{\centering PGD-1000}&5.8 & 0.0 & 1.4 & 10.2 & 8.9 & 20.4 & {45.6} & \textbf{75.6} \\
    \cmidrule(lr){1-1} \cmidrule(lr){2-2} 
    \cmidrule(lr){3-10}

\multirow{2}{*}{Birds}&\multirow{1}{*}{\centering Clean}&96.6&\textbf{97.8}&95.1&73.1&54.5&78.4&93.4&{96.8}&  \\& \multirow{1}{*}{\centering PGD-1000}&5.7 & 8.8 & 12.5 & 11.7 & 9.3 & 22.0 & {64.5} & \textbf{74.2} \\
    \cmidrule(lr){1-1} \cmidrule(lr){2-2} 
    \cmidrule(lr){3-10}

\multirow{2}{*}{Flower}&\multirow{1}{*}{\centering Clean}&96.8&96.6&98.9&77.6&69.6&85.1&\textbf{97.2}&\textbf{97.2}&  \\& \multirow{1}{*}{\centering PGD-1000}&7.6 & 3.8 & 15.5 & 14.0 & 21.2 & 30.1 & \textbf{78.4} & {77.2} \\
    \cmidrule(lr){1-1} \cmidrule(lr){2-2} 
    \cmidrule(lr){3-10}  

\multirow{2}{*}{COIL}&\multirow{1}{*}{\centering Clean}&\textbf{97.7}&{88.1}&79.5&74.4&57.5&77.9&80.6&78.6&  \\& \multirow{1}{*}{\centering PGD-1000}&14.0 & 1.8 & 0.0 & 14.6 & 12.3 & 17.5 & \textbf{43.6} & {43.1} \\
    \cmidrule(lr){1-1} \cmidrule(lr){2-2} 
    \cmidrule(lr){3-10}  

\multirow{2}{*}{CIFAR10}&\multirow{1}{*}{\centering Clean}&{92.9}&\textbf{94.8}&87.5&67.5&50.3&43.6&57.5&61.5  \\ & \multirow{1}{*}{\centering PGD-1000}&7.4 & 4.1 & 2.0 & 9.0 & 8.6 & 1.3 & {12.1} & \textbf{29.0} \\
\cmidrule(lr){1-1} \cmidrule(lr){2-2} 
\cmidrule(lr){3-10}  

\noalign{\smallskip}
\multirow{2}{*}{Avg.}&\multirow{1}{*}{\centering Clean}&\textbf{95.8}&{91.5}&91.6&70.0& 61.5 &79.3&87.7&88.1&   \\& \multirow{1}{*}{\centering PGD-1000}
&7.1 & 2.0 & 3.7 & 11.9 & 19.9 & 24.7 & {53.6} & \textbf{64.7} \\
 \hline
\end{tabular}}
\end{subtable}
 
\end{table*}

\begin{table*}[h ]
\caption{AUROC scores for (One-Class) Novelty Detection under three different adversarial attacks with $ \epsilon = \frac{8}{255} $ 
}
\label{Table 1:Novelty Detection_2}
\BlankLine
\centering
\label{Table 1.a:Novelty Detection for Low-Resolution_0}
\resizebox{\linewidth}{!}{

\begin{tabular}{ll*{5}{c}*{6}{c}c} 

     \noalign{\smallskip}
     \multirow{3}{*}{\textbf{\tenpt{Method}}}& \multirow{3}{*}{\rotatebox[origin=c]{90}{\textbf{\tenpt{Attack}}}} & \multicolumn{5}{c}{\textbf{\tenpt{Low-Res Datasets}}} &
     \multicolumn{6}{c}{\textbf{\tenpt{High-Res Datasets}}} \\
     \noalign{\smallskip}
    \cmidrule(lr){3-7}  
    \cmidrule(lr){8-13}  
    & & \multirow{2}{*}{CIFAR10} & \multirow{2}{*}{CIFAR100} & \multirow{2}{*}{MNIST} & \multirow{2}{*}{FMNIST} & \multirow{2}{*}{SVHN} & \multirow{2}{*}{MVTecAD} & \multirow{2}{*}{Head-CT} & \multirow{2}{*}{BrainMRI} & \multirow{2}{*}{Tumor Detection} & \multirow{2}{*}{Covid19} & \multirow{2}{*}{Imagenet-30} \\ \\
    \specialrule{0.8pt}{\aboverulesep}{\belowrulesep}

    \multirow{4}{*}{DeepSVDD} & 
    
    Clean & 64.8 & 67.0 &94.8 &{94.5} &60.3 &67.0 &62.5 &74.5 &70.8 &61.9 & 62.8 \\

    &BlackBox & 54.6 & 55.3 & 65.7 & 66.8 & 42.7 & 36.0 & 44.1 & 52.7 & 42.0 & 32.4 & 50.1 \\
    
    &PGD-1000 & 22.4 &14.1 &10.8 &48.7 &7.2 &6.3 &0.0 &3.9 &1.6 &0.0 & 22.0 \\
    
    &AutoAttack & 9.7 & 5.8 & 9.6 & 38.2 & 2.4 & 0.0 & 0.0 & 2.1 & 0.0 & 0.0 & 7.3 \\
    \specialrule{0.8pt}{\aboverulesep}{\belowrulesep}
    
    \multirow{4}{*}{CSI} &
    
    Clean & 94.3& 89.6& 93.8& 92.7& 96.0& 63.6& 60.9& 93.2& 85.3& 65.1& 91.6  \\

    &BlackBox & 43.1 & 34.7 & 72.3 & 64.2 & 32.0 & 37.7 & 50.3 & 61.0 & 60.9 & 25.7 & 36.8 \\
    
    &PGD-1000 & 2.7& 2.5& 0.0& 4.1& 1.3& 0.0& 0.1& 0.0& 0.0& 0.0& 0.3 \\
    
    &AutoAttack & 0.0 & 0.0 & 0.0 & 3.1 & 0.7 & 0.0 & 0.0 & 0.0 & 0.0 & 0.0 & 0.0 \\
    \specialrule{0.8pt}{\aboverulesep}{\belowrulesep}
    
    \multirow{4}{*}{MSAD} &
    
    Clean & 97.2& 96.4& 96.0& 94.2& 63.1& 87.2& 59.4& 99.9& 95.1& 89.2&  96.9 \\

    &BlackBox & 38.4 & 51.8 & 58.1 & 73.8 & 40.9 & 41.3 & 42.6 & 64.2 & {67.7} & 53.6 & 34.9 \\
    
    &PGD-1000 & 0.0& 2.6& 0.0& 0.0& 0.5& 0.4& 0.0& 1.5& 0.0& 4.0& 0.0 \\

    &AutoAttack & 0.0 & 1.7 & 0.0 & 0.0 & 0.0 & 0.0 & 0.0 & 0.0 & 0.0 & 1.9 & 0.0 \\
    
    \specialrule{0.8pt}{\aboverulesep}{\belowrulesep}
    
    \multirow{4}{*}{Transformaly} &
    
    Clean & 98.3& 97.3& 94.8& 94.4& 55.4& 87.9& 78.1& 98.3& 97.4& 91.0& 97.8  \\

    &BlackBox & 62.9 & 64.0 & 73.5 & {79.6} & 26.4 & 56.0 & 65.0 & {71.6} & 78.6 & {70.7} & 63.5 \\
    
    &PGD-1000 & 0.0& 4.1& 9.9& 0.2& 7.3& 0.0& 5.8& 4.5& 6.4& 9.1& 0.0 \\
    
    &AutoAttack & 0.0 & 2.6 & 6.7 & 0.0 & 1.9 & 0.0 & 3.2 & 1.6 & 5.1 & 4.4 & 0.0\\
    
    \specialrule{0.8pt}{\aboverulesep}{\belowrulesep}
    
    \multirow{4}{*}{PatchCore} &
    
    Clean & 68.3& 66.8& 83.2& 77.4& 52.1& 99.6& 98.5& 91.4& 92.8& 77.7& 98.1 \\

    &BlackBox & 18.1 & 23.6 & 46.9 & 58.2 & 12.5 & {58.3} & {80.7} & 72.5 & 67.2 & 56.3 & 24.4 \\
    
    &PGD-1000 & 0.0& 0.0& 0.0& 0.0& 3.0& 6.5& 1.3& 0.0& 9.2& 3.8& 0.0 \\

    &AutoAttack & 0.0 & 0.0 & 0.0 & 0.0 & 1.1 & 4.8 & 0.0 & 0.0 & 6.1 & 0.5 & 0.0 \\
    
    \specialrule{0.8pt}{\aboverulesep}{\belowrulesep}
    
    \multirow{4}{*}{PrincipaLS} &
    
    Clean & 57.7& 52.0& 97.3& 91.0& 63.0& 63.8& 68.9& 70.2& 73.5& 54.2& 74.2 \\

    &BlackBox & 33.3 & 39.4 & 91.6 & 71.1 & 47.7 & 45.2 & 54.3 & 56.9 & 56.4 & 43.8  & 31.9\\
    
    &PGD-1000 & 23.6& 15.3& 76.4& 60.8& 30.3& 24.0& 26.8& 32.9& 24.4& 15.1& 18.7 \\

    &AutoAttack & 20.2 & {14.7} & {72.5} & {58.2} & 29.5 & {12.6} & {16.2} & {17.8} & {14.7} & {9.1} & 18.0 \\
    
    \specialrule{0.8pt}{\aboverulesep}{\belowrulesep}
    \multirow{4}{*}{OCSDF} &
    
    Clean & 57.1& 48.2& 95.5& 90.6& 58.1& 58.7& 62.4& 63.2& 65.2& 46.1& 61.4 \\

    &BlackBox & 48.4 & 36.9 & 85.7 & 77.0 & 46.8 & 33.4 & 40.2 & 48.0 & 35.0 & 28.5 & 52.7 \\
    
    &PGD-1000 & 22.9& 14.6& 60.8& 53.2& 23.0& 4.8& 13.0& 18.6& 16.3& 8.4& 18.7 \\

    &AutoAttack & 15.3 & 12.0 & 58.3 & 49.2 & 19.8 & 0.3 & 8.5 & 12.5 & 10.1 & 6.5 & 14.1\\

    \specialrule{0.8pt}{\aboverulesep}{\belowrulesep}
    \multirow{4}{*}{APAE} &
    
    Clean & 55.2& 51.8& 92.5& 86.1& 52.6& 62.1& 68.1& 55.4& 64.6& 50.7& 62.0 \\

    &BlackBox & 37.6 & 16.3 & 73.0 & 24.3 & 41.6 & 35.9 & 45.2 & 27.1 & 43.1 & 26.1 & 33.9 \\
    
    &PGD-1000 & 0.0& 0.0& 21.3& 9.7& 16.5& 3.9& 6.4& 9.1& 15.0& 9.8& 24.8 \\

    &AutoAttack & 0.0 & 0.0 & 19.8 & 7.0 & 16.2 & 1.8 & 3.8 & 8.3 & 8.3 & 8.7 & 0.0 \\
    
    \specialrule{0.8pt}{\aboverulesep}{\belowrulesep}  

    \multirow{4}{*}{EXOE} &
    
    Clean & 99.6& 97.8& 96.0& 94.7& 68.2& 76.2& 82.4 & 86.2 & 79.3 & 72.5 & 98.1 \\

    &BlackBox & 68.3 & 71.5 & 79.4 & 70.4 & 31.1 & 52.7 & 44.6 & 59.0 & 51.4 & 45.5 & 37.2 \\
    
    &PGD-1000 &  0.3& 0.0& 0.0& 1.8& 0.0& 0.2& 0.1 & 0.1 & 0.0 & 0.8 & 0.0 \\

    &AutoAttack & 0.0 & 0.0 & 0.0 & 1.1 & 0.0 & 0.1 & 0.1 & 0.0 & 0.0 & 0.2 & 0.0 \\
    
    \specialrule{0.8pt}{\aboverulesep}{\belowrulesep}  
    
    \multirow{5}{*}{RODEO (ours)} &
    
    Clean & 87.4& 79.6& 99.4& 95.6& 78.6& 61.5& 87.3& 76.3& 89.0& 79.6& 86.1 \\

    &BlackBox & 85.6 & 77.4 & 98.6 & {94.3} & 77.2 & {60.0} & {85.6} & {75.8} & {87.2} & {75.0} & 83.9 \\
    

    &PGD-1000 & 70.2 &62.1 &94.6 &87.2 &33.8 &14.9 &68.6 &	68.4 &67.0 &58.3 & 73.5  \\
    
    &AutoAttack & 69.3 &61.0 &95.2 &87.6 &33.2 &14.2 &68.4 &70.5 &66.9 &58.8 & 76.8 \\
    &A3 & 70.5 &61.3& 	94.0& 	87.0& 	31.8& 	13.4& 	68.1& 	67.7& 	65.6& 	57.6&  72.4 \\
    
    \specialrule{0.8pt}{\aboverulesep}{\belowrulesep}
    
    \end{tabular}}

\end{table*}

\begin{table*}[tbh]
 
\centering
\caption{Detailed AUROC(\%) comparison of different exposure techniques and our introduced method of Adaptive Exposure over different datasets.}
 \label{Table:X}

\BlankLine
\centering
\label{Table:Ablation_21}
\resizebox{\linewidth}{!}{
\begin{tabular}{llr*{15}{c}} 

\toprule
\diagbox{\small Exposure Technique}{\small Target Dataset}&Attack&CIFAR10&CIFAR100&MNIST &Fashion-MNIST&MVTec-ad & Head-CT& Brain-MRI & Tumor Detection & Covid19 \\
\specialrule{1.5pt}{\aboverulesep}{\belowrulesep}



\multirow{2}{*}{Gaussian Noise} & Clean & 64.4 & 54.6 & 60.1 & 62.7 & 41.9 & 59.0 & 45.3 & 51.7 & 40.7\\
\noalign{\smallskip}
& PGD & 15.2 & 11.9 & 11.6 & 15.0 & 0.0 & 0.5 & 0.0 & 0.9 & 0.0 \\
\noalign{\smallskip}
\cmidrule(lr){1-1}
\cmidrule(lr){2-2}
\cmidrule(lr){3-12}

\multirow{2}{*}{ImageNet (Fixed OE Dataset)} & Clean & {87.3}  & {79.6}  & {90.0}  & {93.0}  & {64.6}  & 61.8  & {69.3} & 71.8 & 62.7\\
\noalign{\smallskip}
& PGD & 69.3 & 64.5 & 42.8 & 82.0 & 0.0 & 1.3 & 0.0 & 22.1 & 23.4 \\
\noalign{\smallskip}
\cmidrule(lr){1-1}
\cmidrule(lr){2-2}
\cmidrule(lr){3-12}

\multirow{2}{*}{Mixup with ImageNet} & Clean & 59.4&{56.1}&59.6&74.2&58.5&54.4&57.3&{76.4}&{69.2}\\
\noalign{\smallskip}
& PGD & 30.8 & 27.1 & 1.7 & 47.8 & 0.5 & 20.6 & 10.8 & 53.1 & 50.2 \\
\noalign{\smallskip}
\cmidrule(lr){1-1}
\cmidrule(lr){2-2}
\cmidrule(lr){3-12}

\multirow{2}{*}{Fake Image Generation} & Clean & 29.5&23.0&76.0&52.2&43.5&63.7&65.2&65.2&42.7\\
\noalign{\smallskip}
& PGD & 15.5 & 14.3 & 51.1 & 30.6 & 7.2 & 6.9 & 28.2 & 32.1 & 12.4 \\
\noalign{\smallskip}
\cmidrule(lr){1-1}
\cmidrule(lr){2-2}
\cmidrule(lr){3-12}

\multirow{2}{*}{Stable Diffusion Prompt} & Clean & 62.4&54.8&84.3&63.7&54.9&{71.5}&66.7&45.8&37.1\\
\noalign{\smallskip}
& PGD & 35.0 & 34.4 & 62.1 & 47.1 & 12.2 & 2.2 & 7.0 & 5.3 & 0.0 \\
\noalign{\smallskip}
\cmidrule(lr){1-1}
\cmidrule(lr){2-2}
\cmidrule(lr){3-12}

\multirow{2}{*}{Dream outlier Prompt} & Clean & 58.2& 50.3 & 80.5 & 66.8 & 55.0 & 69.9 & 68.6 & 42.7 & 44.1\\
\noalign{\smallskip}
& PGD & 24.7 & 20.7 & 51.4 & 45.9 & 12.7 & 1.2 & 5.0 & 10.9 & 0.1 \\
\noalign{\smallskip}
\cmidrule(lr){1-1}
\cmidrule(lr){2-2}
\cmidrule(lr){3-12}

\multirow{2}{*}{\textbf{Adaptive Exposure}} & Clean & {87.4}&{79.6}&{99.4}&{95.6}&{61.5}&{87.3}&{76.3}&{89.0}&{79.6}\\
\noalign{\smallskip}
& PGD &{70.2}&{61.3}&{94.6}&{87.2}&{14.9}&{68.6}&{68.4}&{67.0}&{58.3}\\
 \noalign{\smallskip}
\specialrule{1.5pt}{\aboverulesep}{\belowrulesep}

\end{tabular}}
\end{table*}

\begin{table*}[tbh]
\centering

\caption{FID, Density, and Coverage Metrics for Adaptive Exposure technique}
\label{Table:Ablation_2}
\BlankLine
\resizebox{\linewidth}{!}{
\begin{tabular}{ll*{15}{c}}
\toprule
\noalign{\medskip}
\diagbox{\small Exposure Technique}{\small Target Dataset} & \textbf{Metric} & CIFAR10 & CIFAR100 & MNIST & Fashion-MNIST & MVTec-ad & Head-CT & Brain-MRI & Tumor Detection & Covid19 \\
\noalign{\medskip}
\midrule

\noalign{\medskip}

\multirow{2}{*}{\textbf{Adaptive Exposure}} & FID & 145 & 156 & 133 & 134 & 263 & 204 & 165 & 186 & 201 \\
& D / C & 0.87 / 0.64 & 0.63 / 0.62 & 0.75 / 0.86 & 0.61 / 0.44 & 0.64 / 0.09 & 0.77 / 0.83 & 0.69 / 0.61 & 0.57 / 0.37 & 0.51 / 0.80 \\
\bottomrule
\end{tabular}
}
\label{Table_NON_AT}
\end{table*}
\centering
 
\begin{table*}[tbh]

\caption{The detailed AUROC of the experiments for ND, OSR, and OOD settings under different training modes.}
\label{Table:Ablation_Last}

\begin{subtable}{1\textwidth}

\subcaption{ND}
\setlength{\tabcolsep}{10pt}
\resizebox{\linewidth}{!}{\begin{tabular}{c*{12}{c}} 
\hline\noalign{\smallskip}

  
   \multicolumn{1}{c}{Method} & \multicolumn{1}{c}{Training Mode} & \multicolumn{1}{c}{Attack} &\multicolumn{6}{c}{Dataset }  
     \\      \cmidrule(lr){1-1} \cmidrule(lr){2-2} \cmidrule(lr){3-3} 
    \cmidrule(lr){4-9}  
   
  & & &\multirow{2}{*}{CIFAR10}&\multirow{2}{*}{CIFAR100}&\multirow{2}{*}{MNIST}&\multirow{2}{*}{FashionMNIST}&\multirow{2}{*}{Head-CT} &\multirow{2}{*}{Covid19}  \\
  & & & & & &  & \\
  
\noalign{\smallskip}
 \hline

\noalign{\smallskip}
\multirow{2}{*}{Ours} &Non-Adversarial&Clean / PGD-1000& 93.1 / 0.0 & 86.6 / 0.0 & 98.4 / 0.0 & 94.8 / 0.0 & 96.1 / 0.0 & 89.2 / 0.0 \\
  & Adversarial& Clean / PGD-1000& 87.4 / 70.2 & 79.6 / 61.3 & 99.4 / 94.6 & 95.6 / 87.2 & 87.3 / 68.6 & 79.6 / 58.3
  \\

\noalign{\smallskip}

\hline
\end{tabular}}
\end{subtable}

\medskip

\begin{subtable}{1\textwidth}
\centering

\subcaption{OSR}
\setlength{\tabcolsep}{20pt}
\resizebox{\linewidth}{!}{\begin{tabular}{c*{12}{c}} 
\hline\noalign{\smallskip}

  
   \multicolumn{1}{c}{Method} & \multicolumn{1}{c}{Training Mode} & \multicolumn{1}{c}{Attack} &\multicolumn{4}{c}{Dataset}  
     \\      \cmidrule(lr){1-1} \cmidrule(lr){2-2} \cmidrule(lr){3-3} 
    \cmidrule(lr){4-9}  
   
  & & &\multirow{2}{*}{CIFAR10}&\multirow{2}{*}{CIFAR100}&\multirow{2}{*}{MNIST}&\multirow{2}{*}{FashionMNIST}  \\
  & & & & & & \\
  
\noalign{\smallskip}
 \hline

\noalign{\smallskip}
\multirow{2}{*}{Ours} &Non-Adversarial&Clean / PGD-1000& 84.3 / 0.0 & 69.0 / 0.0 & 99.1 / 0.0 & 91.9 / 0.0 \\
  & Adversarial& Clean / PGD-1000& 79.6 / 62.7 & 64.1 / 35.3 & 97.2 / 85.0 & 87.7 / 65.3
  \\

\noalign{\smallskip}

\hline
\end{tabular}}
\end{subtable}

\medskip

\begin{subtable}{1\textwidth}
\centering

\subcaption{OOD}
\setlength{\tabcolsep}{25pt}
\resizebox{\linewidth}{!}{\begin{tabular}{c*{12}{c}} 
\hline\noalign{\smallskip}

  
   \multicolumn{1}{c}{Method} & \multicolumn{1}{c}{Training Mode} & \multicolumn{1}{c}{Attack} &\multicolumn{4}{c}{Dataset }  
     \\      \cmidrule(lr){1-1} \cmidrule(lr){2-2} \cmidrule(lr){3-3} 
    \cmidrule(lr){4-9}  
   
  & & &\multirow{2}{*}{CIFAR10 vs CIFAR100}&\multirow{2}{*}{CIFAR100 vs CIFAR10}  \\
  & & & &\\
  
\noalign{\smallskip}
 \hline

\noalign{\smallskip}
\multirow{2}{*}{Ours} &Non-Adversarial&Clean / PGD-1000& 83.0 / 0.0 & 71.2 / 0.0 \\
  & Adversarial& Clean / PGD-1000& 75.6 / 37.8 & 61.5 / 29.0 
  \\

\noalign{\smallskip}

\hline
\end{tabular}}
\end{subtable}

\end{table*}

\begin{table*}[thb]
\centering
\caption{Adversarial training step hyper-parameters}
\label{tab:adv_training_hyperparams}

\begin{tabular}{|c|c|c|c|c|c|c|c|}
\hline
\textbf{Adv. Tr.} & \textbf{N} & \textbf{High Res.} $\epsilon$ & \textbf{Low Res.} $\epsilon$ & $\alpha$ & \textbf{Classifier} & \textbf{Optimizer} & \textbf{LR} \\
\hline
PGD & 10 & $\frac{2}{255}$ & $\frac{8}{255}$ & $2.5 \times \frac{\epsilon}{N}$ & ResNet & Adam & 0.001 \\
\hline
\end{tabular}
\end{table*}

\begin{table*}[thb]
\centering
\caption{Generation step hyper-parameters}
\label{tab:gen_step_hyperparams}
\begin{tabular}{|c|c|c|c|c|c|c|}
\hline
\textbf{Gen. Backbone} & \textbf{Pre. Dataset} & \textbf{T} & $t_0$ & $\tau_{\text{image}}$ & $\tau_{\text{text}}$ \\
\hline
 GLIDE\cite{nichol2021glide} & 67 Million created dataset & 1000 & (0, 600) & Eq.11 & Eq.17 \\
\hline
\end{tabular}
\end{table*}

\clearpage

\begin{figure}[t]
  \begin{center} 
    \includegraphics[width=1\linewidth]{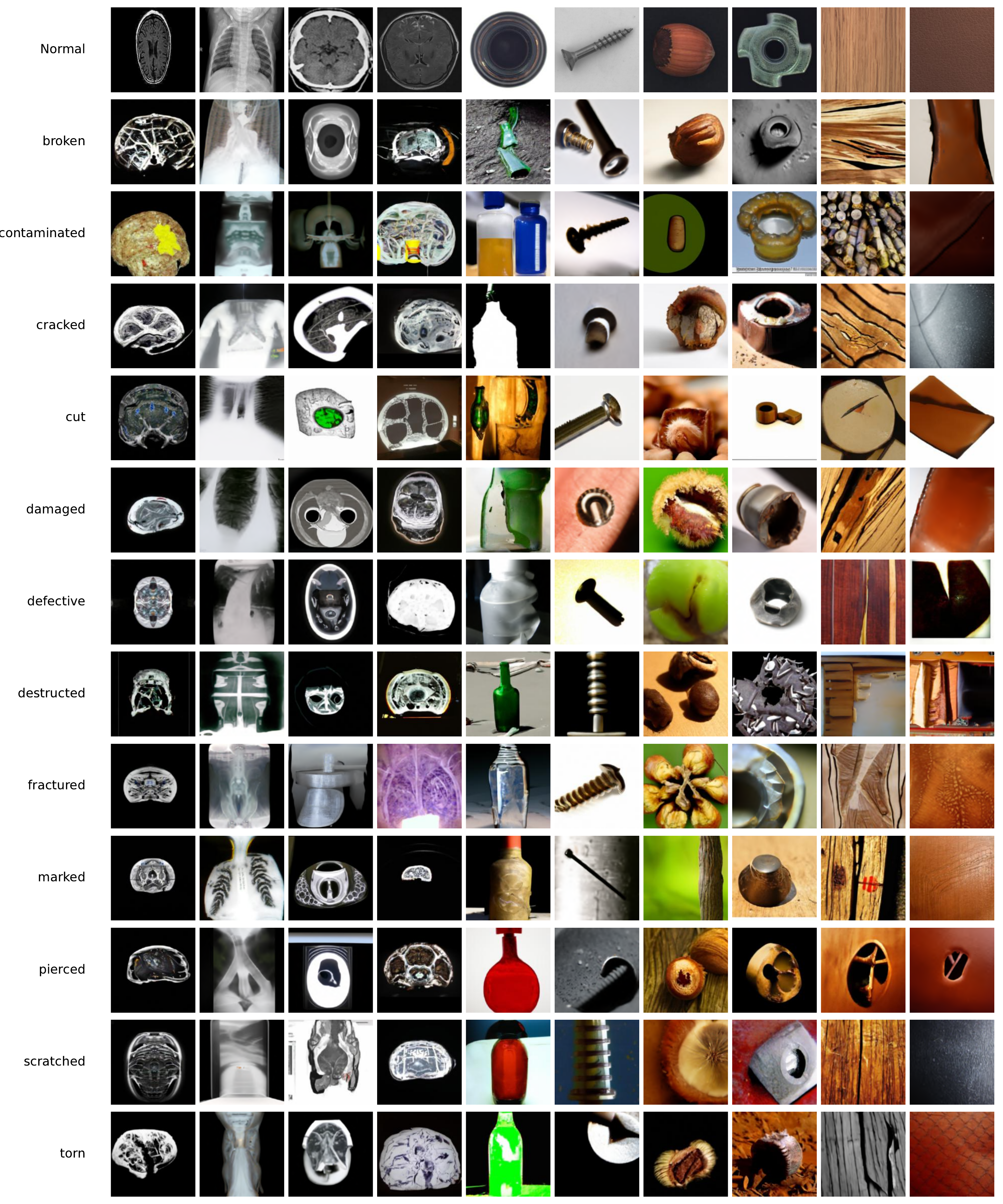}
    \caption{In this experiment, we skipped image conditioning and started the denoising process from pure Gaussian noise, directed by the extracted outlier labels. The resulting generated outlier samples are conditioned solely on text (excluding image conditioning), produced using our pipeline. Comparing the visuals of the generated data in this case with those where the image has also been conditioned showcases the importance of simultaneous image and text conditioning in generating near outlier data. }
    \label{fig:Samples_Plot1}
  \end{center}
\end{figure}

\clearpage

\begin{figure}[t]
  \begin{center}
    \includegraphics[width=1\linewidth]{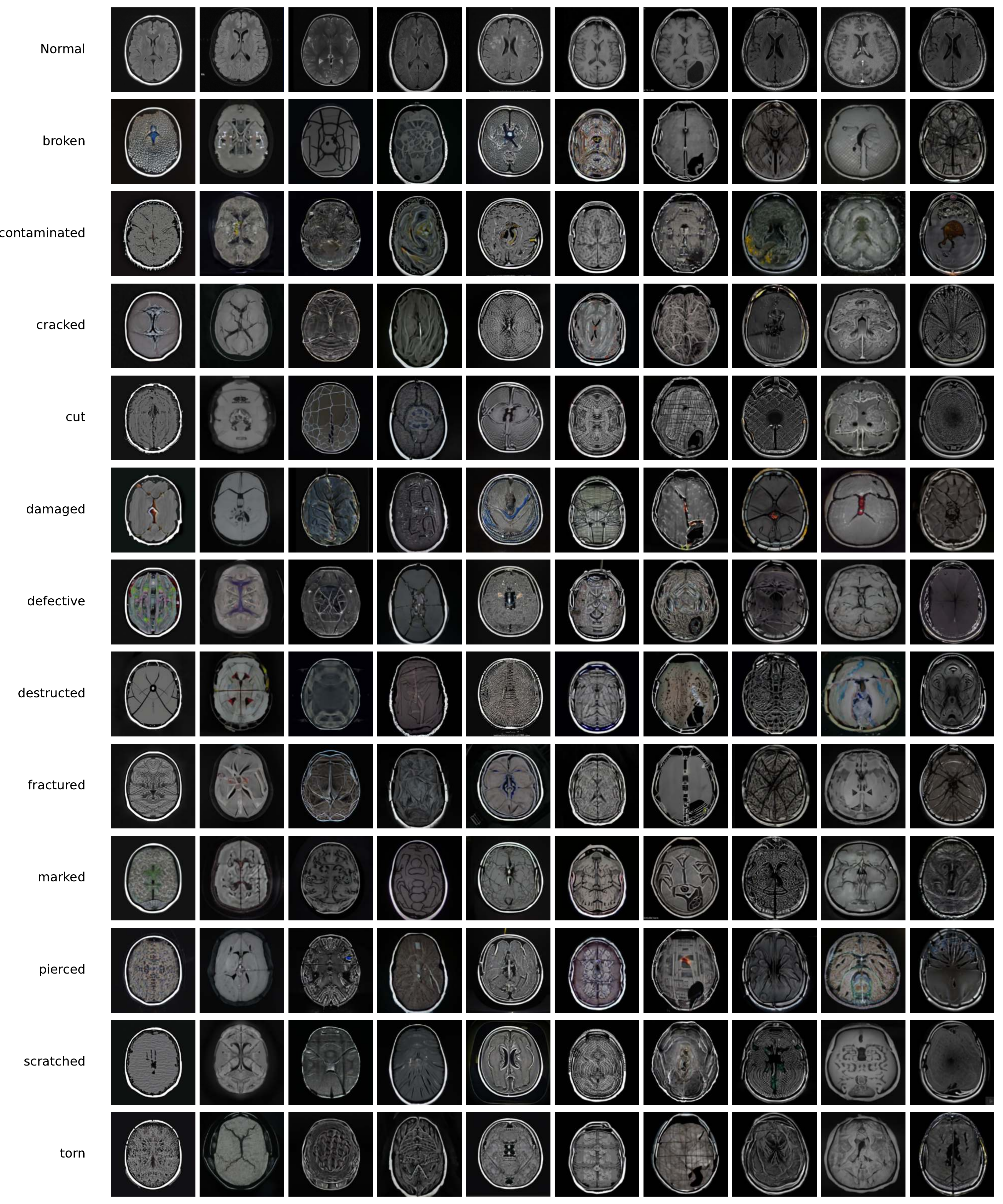}
    \caption{BrainMRI: Examples of generated auxiliary outliers from the BrainMRI dataset, conditioned on negative adjectives and inlier images. The first row depicts inlier images, while the subsequent rows demonstrate generated auxiliary outliers corresponding to the negative adjectives written to the left of each row.}
    \label{fig:Samples_Plot2}
  \end{center}
\end{figure}

\clearpage

\begin{figure}[t]
  \begin{center}
    \includegraphics[width=1\linewidth]{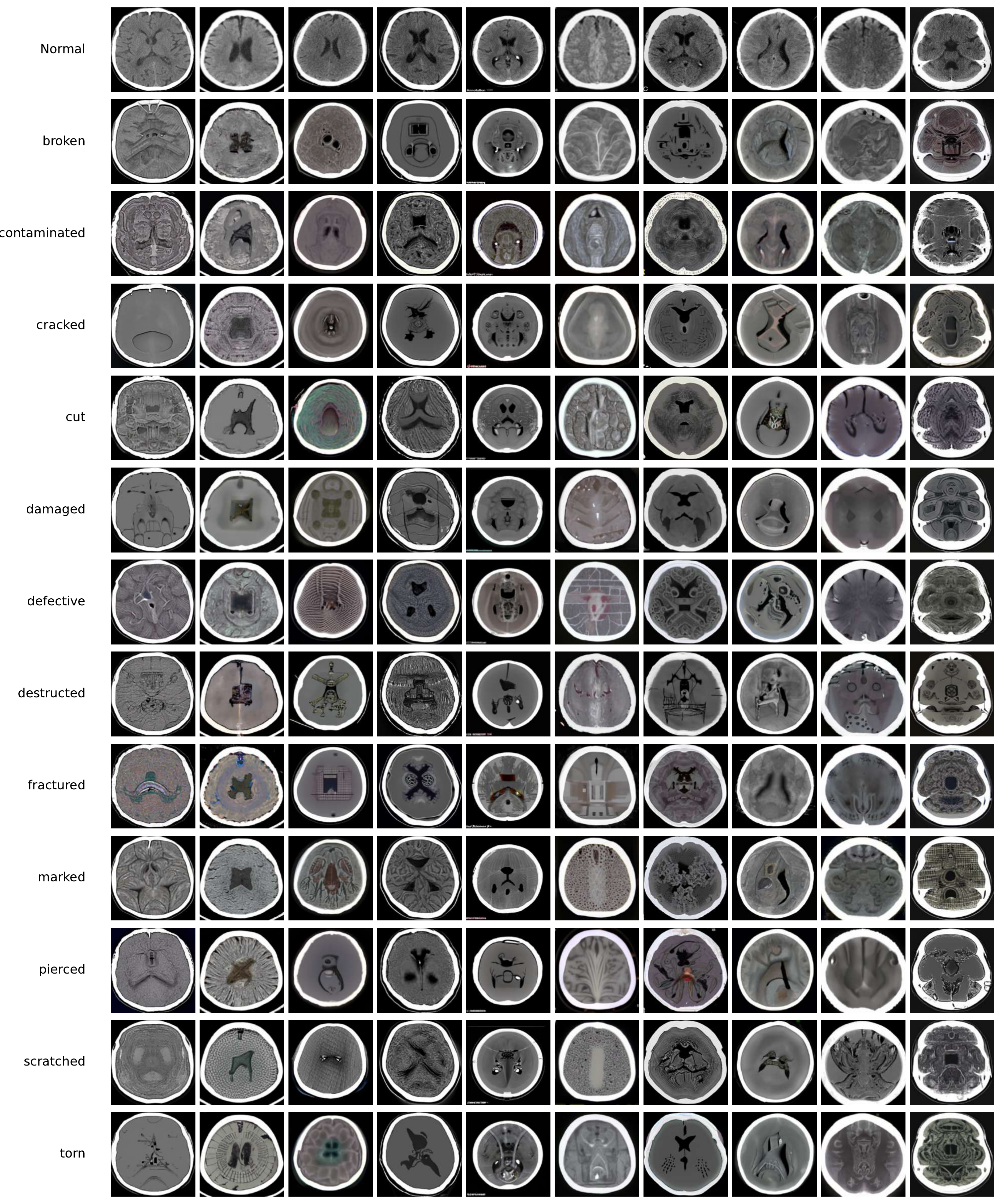}
    \caption{Head-CT: Examples of generated auxiliary outliers from the Head-CT dataset, conditioned on negative adjectives and inlier images. The first row depicts inlier images, while the subsequent rows demonstrate generated auxiliary outliers corresponding to the negative adjectives written to the left of each row.}
    \label{fig:Samples_Plot3}
  \end{center}
\end{figure}

\clearpage

\begin{figure}[t]
  \begin{center}
    \includegraphics[width=1\linewidth]{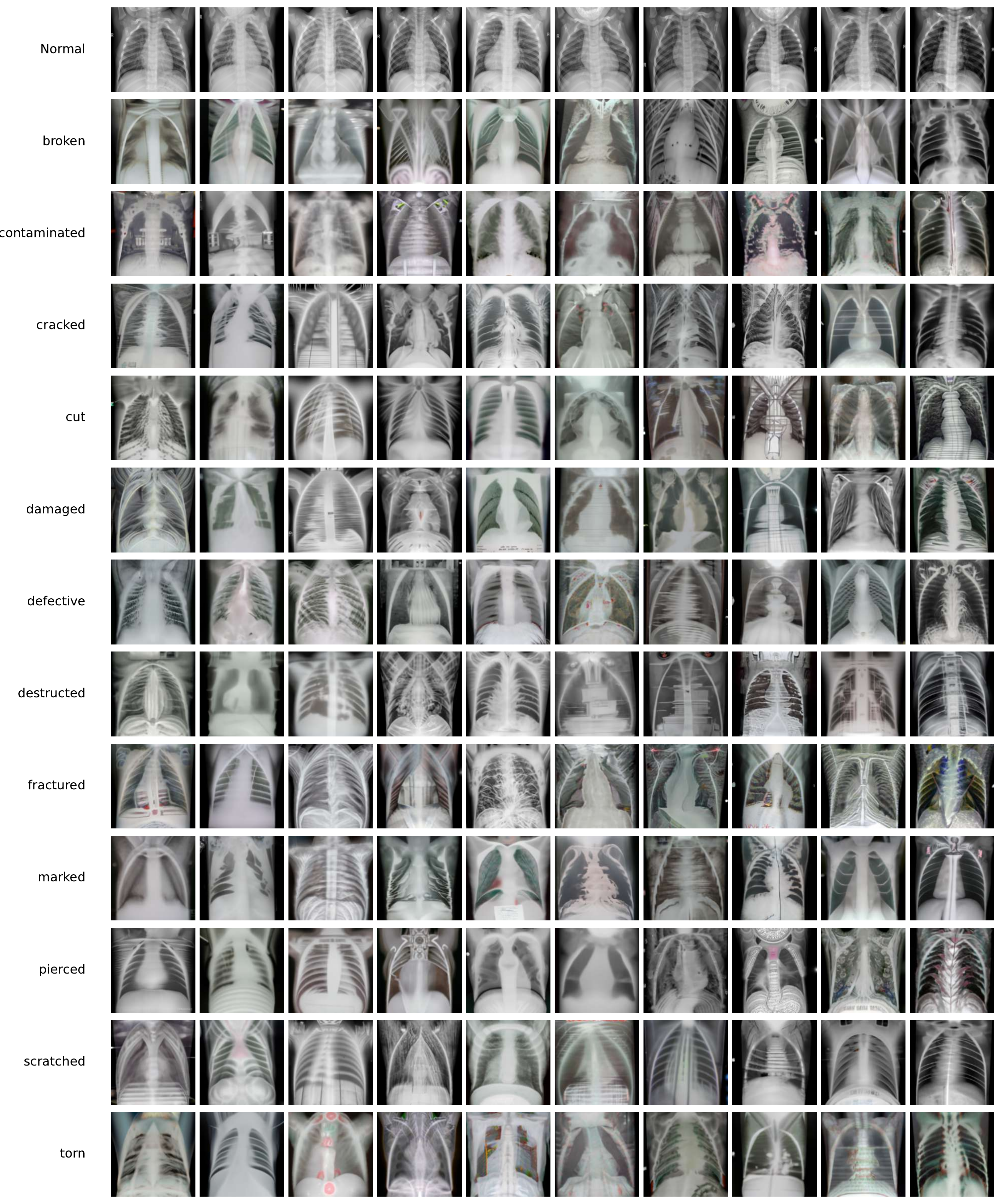}
    \caption{Covid19: Examples of generated auxiliary outliers from the Covid19 dataset, conditioned on negative adjectives and inlier images. The first row depicts inlier images, while the subsequent rows demonstrate generated auxiliary outliers corresponding to the negative adjectives written to the left of each row.}
    \label{fig:Samples_Plot4}
  \end{center}
\end{figure}

\clearpage

\begin{figure}[t]
  \begin{center}
    \includegraphics[width=1\linewidth]{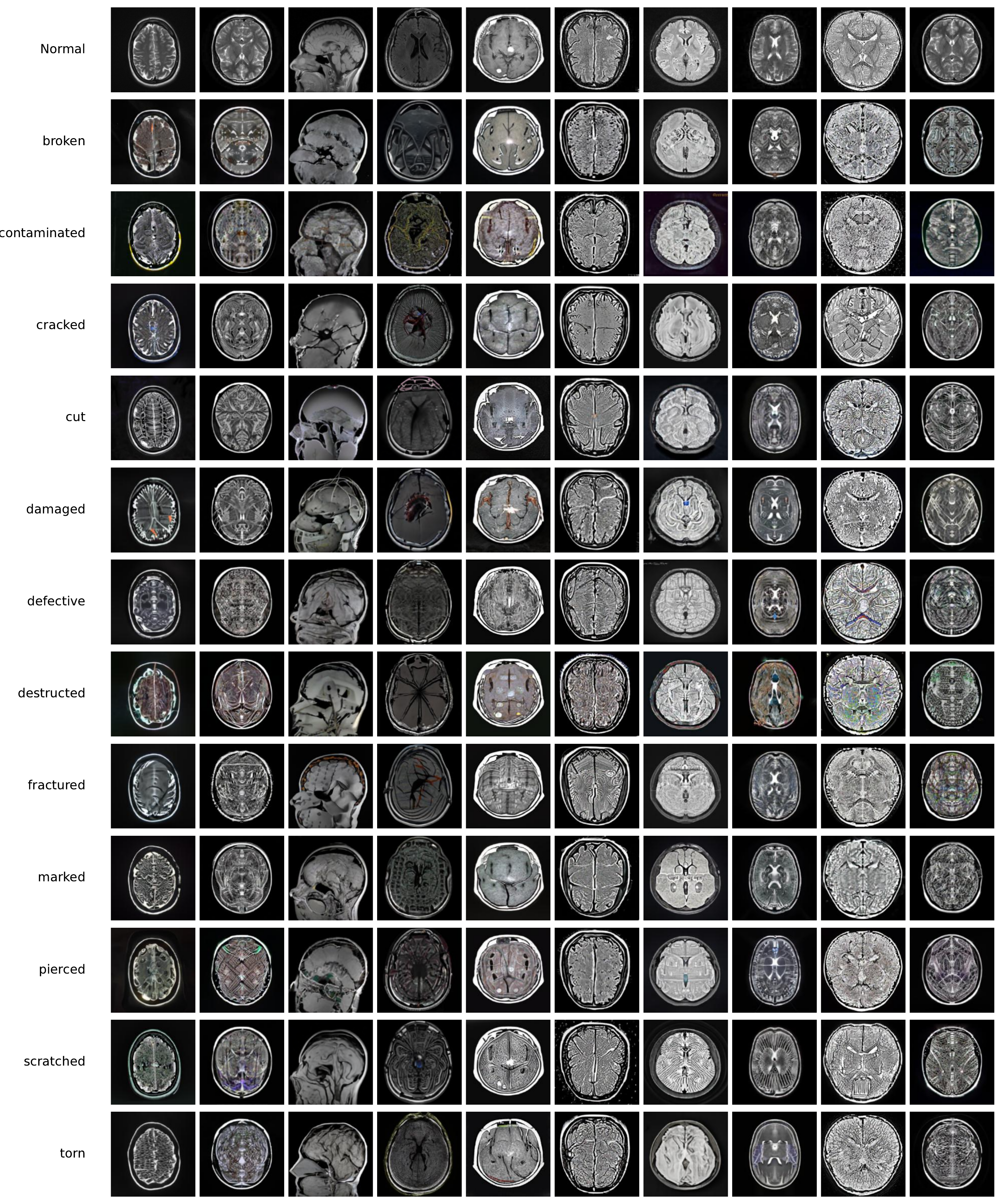}
    \caption{Tumor Detection: Examples of generated auxiliary outliers from the Tumor Detection dataset, conditioned on negative adjectives and inlier images. The first row depicts inlier images, while the subsequent rows demonstrate generated auxiliary outliers corresponding to the negative adjectives written to the left of each row.}
    \label{fig:Samples_Plot5}
  \end{center}
\end{figure}

\clearpage

\begin{figure}[t]
  \begin{center}
    \includegraphics[width=1\linewidth]{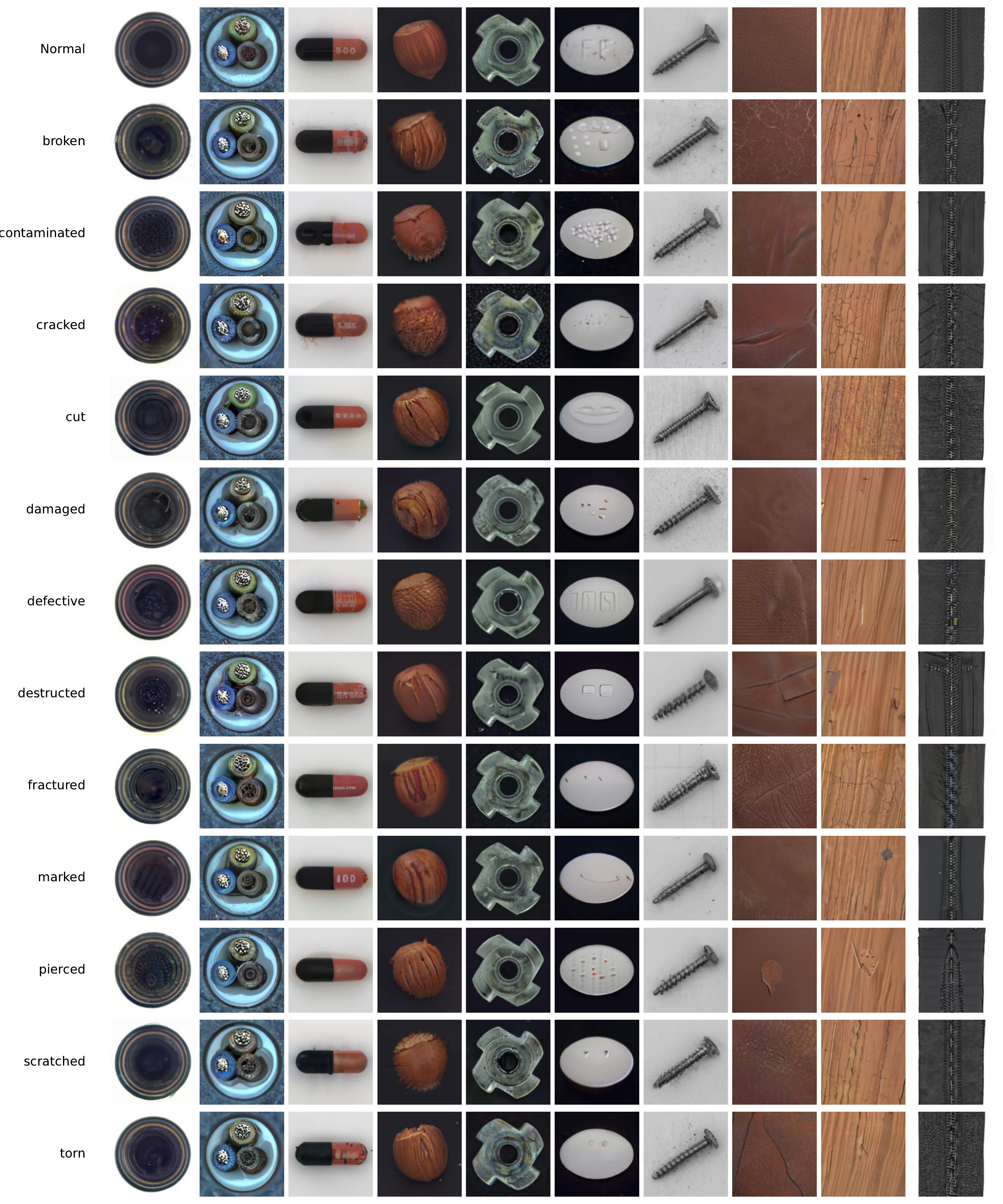}
    \caption{MVTec-AD: Examples of generated auxiliary outliers from the MVTec-AD dataset, conditioned on negative adjectives and inlier images. The first row depicts inlier images, while the subsequent rows demonstrate generated auxiliary outliers corresponding to the negative adjectives written to the left of each row.}
    \label{fig:Samples_Plot6}
  \end{center}
\end{figure}

\clearpage

\begin{figure}[t]
  \begin{center}
    \includegraphics[width=1\linewidth]{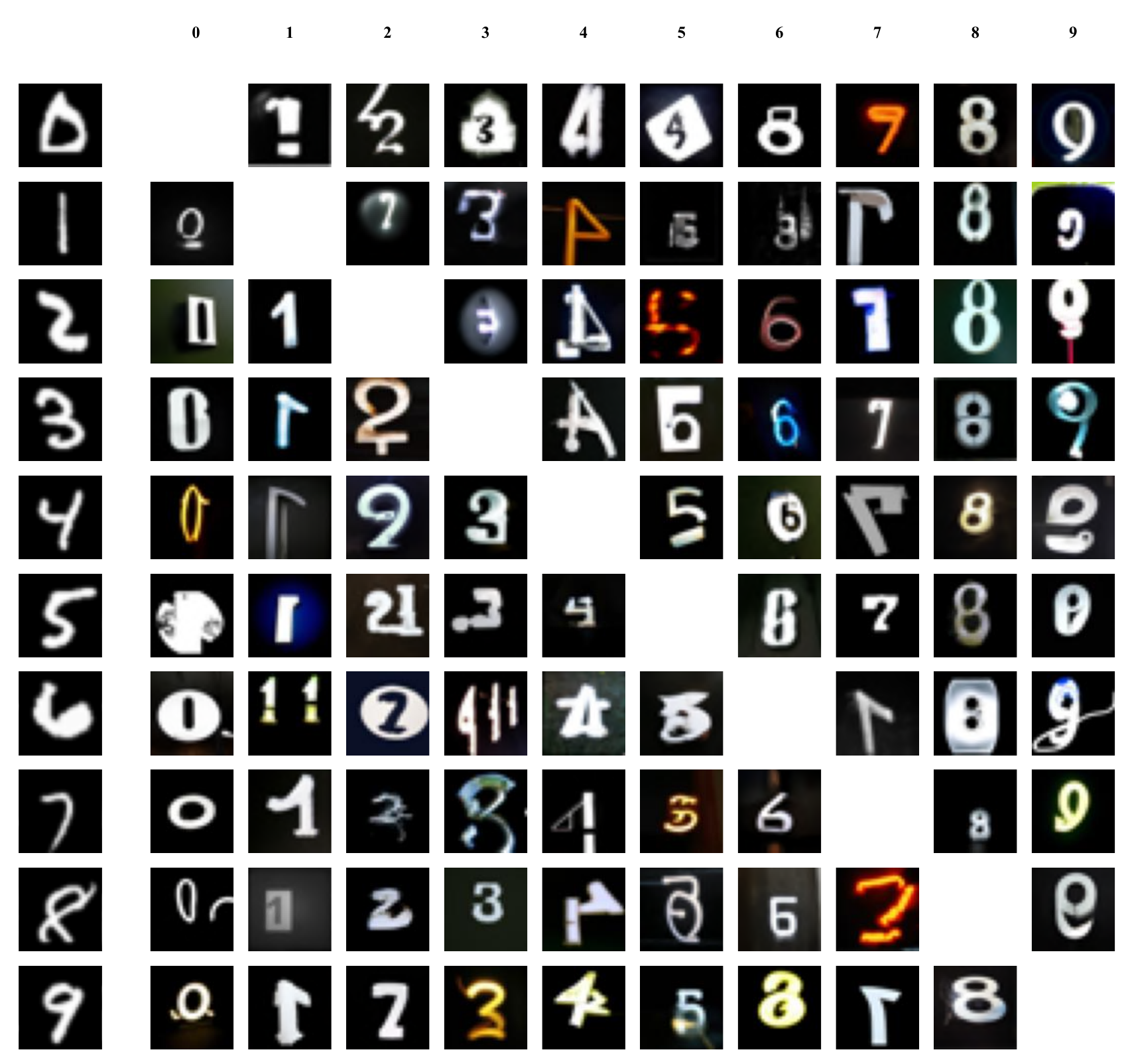}
    \caption{MNIST Adaptive Exposures Grid: This figure illustrates a grid of adaptive exposures of handwritten digits, created using our pipeline with the MNIST dataset, accompanied by their corresponding labels. By utilizing the original data and text prompts, our pipeline generates a variety of exposures that adaptively capture the dataset's distribution while incorporating outlier elements.}
    \label{fig:Samples_Plot7}
  \end{center}
\end{figure}

\begin{figure}[t]
  \begin{center}
    \includegraphics[width=1\linewidth]{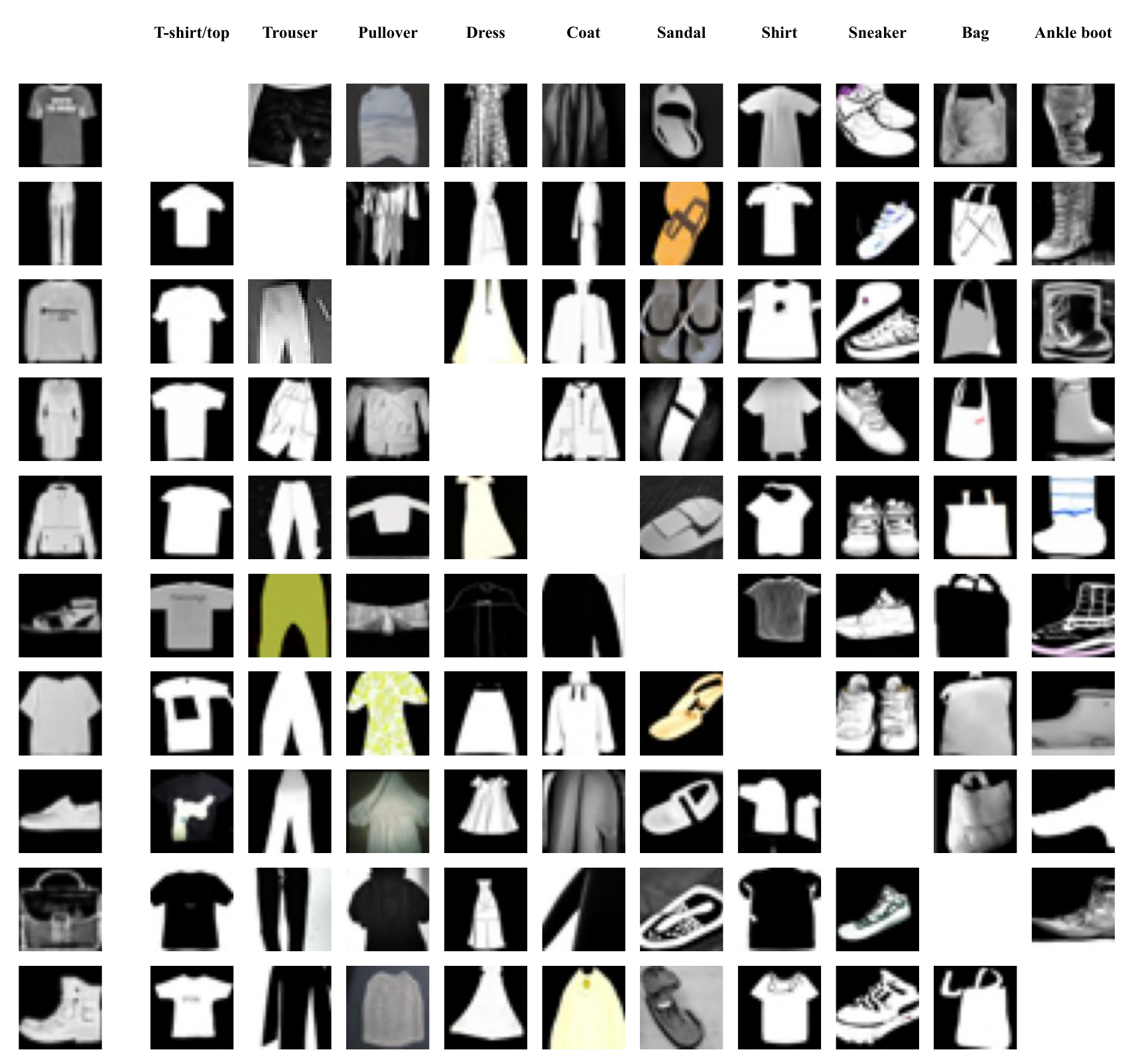}
    \caption{FashionMNIST Adaptive Exposures Grid: This figure illustrates a grid of adaptive exposures of handwritten digits, created using our pipeline with the FashionMNIST dataset, accompanied by their corresponding labels. By utilizing the original data and text prompts, our pipeline generates a variety of exposures that adaptively capture the dataset's distribution while incorporating outlier elements.}
    \label{fig:Samples_Plot8}
  \end{center}
\end{figure}


\end{document}